\documentclass{article}

     \usepackage[final, nonatbib]{neurips_2020}

\usepackage[utf8]{inputenc} 
\usepackage[T1]{fontenc}    
\usepackage{hyperref}       
\usepackage{url}            
\usepackage{booktabs}       
\usepackage{amsfonts}       
\usepackage{nicefrac}       
\usepackage{microtype}      

\usepackage{subfigure} 
\usepackage{placeins}
\usepackage{algorithm}
\usepackage{algorithmic}
\usepackage{array}
\usepackage{overpic}
\usepackage{wrapfig}
\usepackage{bm}
\usepackage{comment}
\usepackage{pgfplots}
\usepackage{pgf}
\usepackage{tikz}
\usepackage{soul}
\usepackage{tabstackengine}
\usepackage{mathtools}
\usepackage{amsthm}
\newtheorem{theorem}{Theorem}

\hypersetup{
    colorlinks=true
}

\newcommand{\X}{\mathcal{X}}
\newcommand{\Y}{\mathcal{Y}}

\DeclareMathOperator*{\argmin}{arg\,min}

\newcommand{\rick}[1]{\textcolor{black}{#1}}
\definecolor{mygray}{RGB}{120,120,120} 
\newcommand{\mygray}[1]{\textcolor{mygray}{#1}}

\pgfplotsset{compat=1.16}

\title{Correspondence Learning via Linearly-invariant Embedding}

\author{%
  Riccardo Marin\thanks{denotes equal contribution.}
  \\
  University of Verona\\
\texttt{riccardo.marin\_01@univr.it} \\
   \And
   Marie-Julie Rakotosaona\footnotemark[1]\\
   LIX, Ecole Polytechnique, IP Paris \\
   \texttt{mrakotos@lix.polytechnique.fr} \\
   \AND
   Simone Melzi \\
   LIX, Ecole Polytechnique, IP Paris \\
   Sapienza University of Rome\\
   \texttt{melzi@di.uniroma1.it} \\
   \And
   Maks Ovsjanikov \\
   LIX, Ecole Polytechnique, IP Paris \\
    \texttt{maks@lix.polytechnique.fr} \\
}

\begin{document}

\maketitle

\begin{abstract}
In this paper, we propose a fully differentiable pipeline for estimating accurate dense correspondences between 3D point clouds.  The proposed pipeline is an extension and a generalization of the \emph{functional maps framework}. However, instead of using the Laplace-Beltrami eigenfunctions as done in virtually all previous works in this domain, we demonstrate that learning the basis from data can both improve robustness and lead to better accuracy in challenging settings. We interpret the basis as a learned embedding into a higher dimensional space. Following the functional map paradigm the optimal transformation in this embedding space must be linear and we propose a separate architecture aimed at estimating the  transformation by learning optimal descriptor functions. This leads to the first end-to-end trainable functional map-based correspondence approach in which both the basis and the descriptors are learned from data. Interestingly, we also observe that learning a \emph{canonical} embedding leads to worse results, suggesting that leaving an extra linear degree of freedom to the embedding network gives it more robustness, thereby also shedding light onto the success of previous methods. Finally, we demonstrate that our approach achieves state-of-the-art results in challenging non-rigid 3D point cloud correspondence applications.

\end{abstract}

\section{Introduction}
\label{sec:introduction}
Computing correspondences between geometric objects is a widely investigated task. Its applications are countless: rigid and non-rigid registration methods are instrumental in engineering, medicine and biology \cite{jin2019fast, li2020toward, gainza2020deciphering} among other fields. Point cloud registration is important for range scan data, e.g., in robotics \cite{Gojcic_2019_CVPR, SANCHEZ2020}, but the problem can also be generalized  to abstract domains like graphs \cite{wang2019functional, fey2020deep}.

The \emph{non-rigid} correspondence problem is particularly challenging as a successful solution must deal with a large variability in shape deformations and be robust to noise in the input data. To address this problem, in recent years, several data-driven approaches have been proposed to learn the optimal transformation model from data rather than imposing it \emph{a priori}, including \cite{groueix20183d,wei2016dense,boscaini2016anisotropic} among others. In this domain, a prominent direction is based on the functional map representation \cite{ovsjanikov2012functional}, which has been adapted to the learning-based setting \cite{litany2017deep,Halimi_2019_CVPR,roufosse2019unsupervised,donati2020deep}. These  methods have shown that optimal feature or descriptor functions (also known as ``probe'' functions) can be learned from data and then used successfully within the functional map pipeline to obtain accurate dense correspondences. Unfortunately, the reduced functional basis, which forms the key ingredient in this approach, has so far been tied to the Laplace-Belrtami eigen-basis, specified and fixed \emph{a priori}. While this choice might be reasonable for near-isometric 3D shapes represented as triangle meshes, it does not allow to handle  more diverse deformations classes of  or significant noise in the data. 

Inspired by the success and robustness of these techniques, we propose  the first fully-differentiable functional maps pipeline, in which both the probe functions and the functional basis are learned from the data. Our key observation is that basis learning can be phrased as computing an embedding into a higher-dimensional space in which a non-rigid deformation becomes a \emph{linear transformation}. This follows the functional map paradigm in which functional maps arising from pointwise correspondences must always be linear \cite{ovsjanikov2012functional} and computing such a linear transformation is equivalent to solving the non-rigid correspondence problem. In the process, we also observe that training a network that aims to compute a \emph{canonical} embedding, in which optimal correspondences are simple nearest neighbors, leads to a drop in performance. As we discuss below, this suggests that the additional degree of freedom, by learning a linearly-invariant embedding, helps to regularize the learning process and avoid overfitting in challenging cases. 
Finally, we demonstrate that our simple (but effective) formulation leads to accurate dense maps. The code, datasets
and our pre-trained networks can be found online: \url{https://github.com/riccardomarin/Diff-FMaps}.

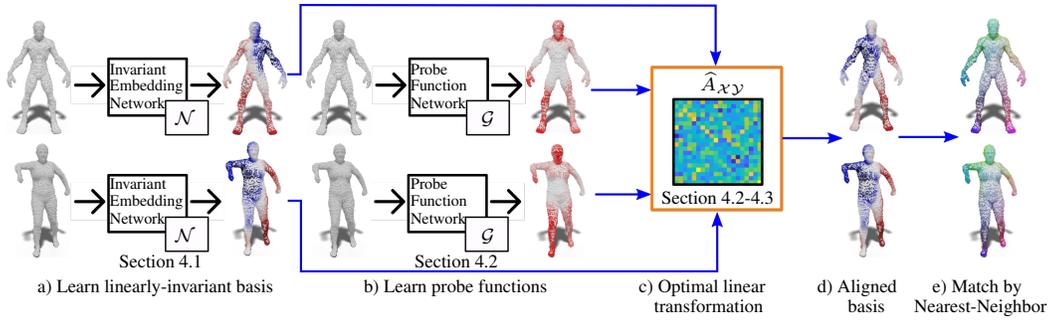
\begin{figure}
  \centering
  \scriptsize
  \begin{overpic}[width=0.97\textwidth]{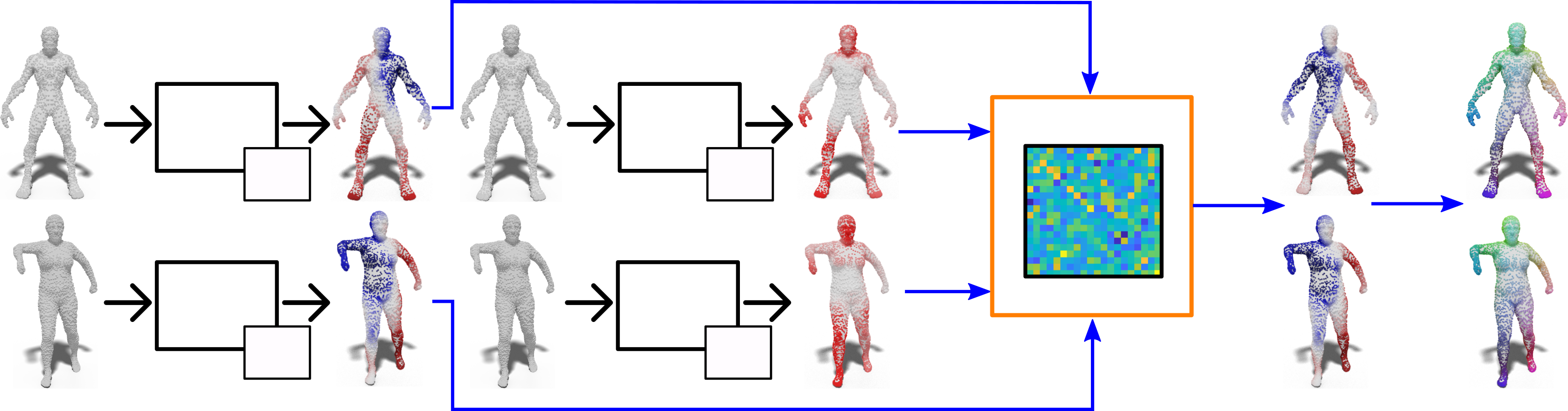}

    \put(10.1,19.4){\tiny Invariant}
    \put(10.1,17.8){\tiny Embedding}
    \put(10.1,16){\tiny Network}
    \put(16.5,14.5){$\mathcal{N}$}
    
    \put(10.1,8.1){\tiny Invariant}
    \put(10.1,6.5){\tiny Embedding}
    \put(10.1,4.7){\tiny Network}
    \put(16.5,3){$\mathcal{N}$}

    \put(39.6,19.4){\tiny Probe}
    \put(39.6,17.8){\tiny Function}
    \put(39.6,16){\tiny Network}
    \put(46.5,14.5){$\mathcal{G}$}
    
    \put(39.6,8.1){\tiny Probe}
    \put(39.6,6.5){\tiny Function}
    \put(39.6,4.7){\tiny Network}    
    \put(46.5,3){$\mathcal{G}$}
    
    \put(68,17.8){$\widehat{A}_{\mathcal{XY}}$}
    
    \put(11,0.4){Section 4.1}
    \put(40,0.4){Section 4.2}
    \put(64.2,6.7){Section 4.2-4.3}
    
    \put(3,-2){a) Learn linearly-invariant basis}
    
    \put(35,-2){b) Learn probe functions}    
    \put(62,-2){c) Optimal linear}
    \put(63.5,-4){transformation}
    
    \put(79.5,-2){d) Aligned}   
    \put(82.5, -4){basis}
    \put(90.8,-2){e) Match by}
    \put(89,-4){Nearest-Neighbor}

  \end{overpic}
  \vspace{0.3cm}
  
\caption{Pipeline overview:  starting from point cloud coordinates we obtain a set of linearly-invariant basis functions via the Invariant Embedding Network $\mathcal{N}$ (a), and descriptors using the Probe Function Network $\mathcal{G}$ (b). The learned basis and probe functions are used to compute the optimal linear transformation $\widehat{A}_{\X\Y}$ (c). This transformation is used to align the two sets of bases (d). The correspondence between point clouds is then estimated using nearest neighbors between the aligned basis sets (e). Note that the underlying meshes are depicted only for sake of clarity of visualization.}
  \label{fig:overview}
\end{figure}

   

  
  
   
  
  
   

\section{Related work}
\label{sec:related}
In addition to approaches mentioned above, here we briefly discuss previous works in the shape correspondence domain that are either closest to ours or most relevant for comparison and evaluation. 
We refer to the available surveys \cite{survey0,survey1} for a more complete overview.

\paragraph{Functional maps}
The core of our method is the functional maps framework originally proposed in \cite{ovsjanikov2012functional} which formulates the correspondence problem in the functional domain instead of the classical matching between points.
In the functional space, a correspondence can be represented by a small matrix encoded in a reduced basis and computed as the optimal transformation that aligns a given set of probe functions possibly with other regularization. 
This method inspired a large number of further extensions, including \cite{nogneng2017informative,ezuz2017deblurring,rodola2017partial,ren2018continuous} to name a few. A more general overview of this area can be found in \cite{ovsjanikov2017computing}.
In our paper we also exploit the link between the functional representation and the adjoint map that has been originally developed in \cite{huang2017adjoint}.

The most common basis used in the functional map framework is given by the eigenfunctions of the Laplace-Beltrami operator, which can be seen as a natural extension of the Fourier basis to non-Euclidean domains \cite{levy2010,taubin}. 
These basis functions are appropriate for shapes  represented as triangle meshes, undergoing near-isometric deformations. Unfortunately, however, they can be highly unstable and difficult to estimate reliably for more general deformations and on point cloud data.
Possible alternatives to this choice have been proposed in the literature such as 
\cite{LMH,nogneng18,CMH}. These works try to recover the information lost by the low-pass representation of the truncated Fourier basis but still suffer from the same underlying limitations. Several efforts have been made to apply the functional map framework to point cloud data, including \cite{rodola2017partial,melzi2019zoomout}. These works exploit existing discretizations of the Laplace-Beltrami operator on point clouds \cite{belkin2009constructing,liang2013solving}, and present acceptable results under clean dense sampling but quickly deteriorate in more challenging scenarios.

Other works have been devoted to the selection of appropriate probe functions used to guide the computation of functional maps.
Axiomatic descriptors such as HKS, WKS or SHOT \cite{sun2009concise,aubry2011wave,SHOT} are widely used as probe functions together with supervised information such as segments and landmarks \cite{StableRegion,Denitto_2017_ICCV}. More recently, an optimisation-based strategy has been proposed to compute optimal relative weights of probe functions \cite{desclearn},
while a set of five automatically estimated stable landmarks has been used as probe functions for functional maps on human shapes in \cite{FARM}. 

\paragraph{Learning based methods for functional maps}
While early works in functional maps are purely axiomatic \cite{kovnatsky2013coupled,pokrass2013sparse,nogneng18,rodola2017partial}, this framework has also recently been adapted to the learning setting. Specifically, starting with the seminal work of Deep Functional Maps \cite{litany2017deep}, several methods have been proposed to \emph{learn} optimal descriptors that can be used within the functional maps framework \cite{Halimi_2019_CVPR,roufosse2019unsupervised,donati2020deep}. Most recently, it was demonstrated in \cite{donati2020deep} that the optimal descriptor (or probe) functions can be learned directly from the 3D coordinates of the shapes. This work has also shown that a functional map layer can help to regularize shape correspondence learning, leading to better results with less training data compared to state-of-the-art purely point-based methods \cite{groueix20183d}. Nevertheless, the approach of \cite{donati2020deep} is still tied to the choice of the Laplace-Beltrami eigenbasis and therefore lacks robustness in challenging non-isometric settings. Instead our fully learnable pipeline allows to benefit from the functional map regularization while being both robust and applicable to point cloud data.


\paragraph{Other approaches}
A different line of work has also aimed to learn correspondences between 3D shapes by coordinate transfer \cite{groueix20183d,groueix2019unsupervised}.
Other recent techniques also use geometric information through diverse convolution operations \cite{wei2016dense,GCNN,ginzburg2019cyclic,NIPS2019_8962,Wang2020MGCN} and have demonstrated their effectiveness in 3D shape matching, typically by phrasing it as a dense segmentation problem. 
Learning for partial \emph{rigid} alignment has been also proposed \cite{wang2019prnet}.
Correspondences can also be computed through finding a  canonical embedding of the input. This idea has been developed for 2D images \cite{choy2016universal,thewlis2019unsupervised} as well as 3D data \cite{Zhou2019SiamesePointNet}.
The latter works, as many others in this domain, take advantage of point-based architectures such as PointNet \cite{qi2017pointnet} and its extensions \cite{Ponitnetplus,Atzmon,Thomas_2019_ICCV} that provide a powerful way to learn signatures for point clouds, and that have been mainly exploited for shape classification but not yet for \emph{smooth and consistent} dense non-rigid shape correspondence.

\section{Background, motivation and notation}
\label{sec:background}
In this section, we give a brief overview of the functional map representation and correspondence pipeline. We then provide a general motivation behind our work and introduce the main notation that we adopt in the rest of the paper. 

\paragraph{Functional Maps}
We start by summarizing the functional maps framework. This formalism was initially developed for smooth surfaces, and most of the constructions have immediate analogues in the discrete setting when shapes are represented as triangle meshes. Note that we describe our learning-based pipeline adapted to point clouds in the following sections. Given a pair of shapes $\X$ and $\Y$, let $\mathcal{F}(\X)$ and $\mathcal{F}(\Y)$ denote the spaces of real-valued functions on $\X$ and $\Y$,  respectively. A point-to-point map $T_{\X\Y}:\X\rightarrow\Y$ induces a functional correspondence $T^{\mathcal{F}}_{\Y\X}:\mathcal{F}(\Y)\rightarrow\mathcal{F}(\X)$ via pull-back (notice that $T^{\mathcal{F}}_{\Y\X}$ goes in the opposite direction). 
If we approximate the space of functions in a given basis $\Phi_{\X}$ and $\Phi_{\Y}$ of size $k$, then $T^{\mathcal{F}}_{\Y\X}$ can be compactly represented by a matrix $C_{\Y\X}$ of size $k\times k$ that maps the coefficients of a function in the basis $\Phi_{\Y}$ to the coefficients of its image via $T^{\mathcal{F}}_{\Y\X}$ in the basis $\Phi_{\X}$. Specifically, if the basis is orthonormal, then the entries of this matrix have an explicit expression: $C_{\Y\X}(i,j) = <T^{\mathcal{F}}_{\Y\X}(\phi^{\Y}_j), \phi^{\X}_i>$, where $<,>$ denotes the functional inner product and $\phi^{\X}_i, \phi^{\Y}_j$ are the individual basis functions on $\X$ and $\Y$ respectively.

The most common choice for $\Phi_{\X}$ and $\Phi_{\Y}$ is the set of the eigenfunctions of the Laplace-Beltrami operator $\Delta$ associated with the $k$ eigenvalues with smallest absolute value. This choice was advocated in the original functional maps work \cite{ovsjanikov2012functional} and then later used in virtually all follow-up approaches, including learning-based ones, e.g.,  \cite{rodola2017partial,nogneng2017informative,huang2014functional,litany2017deep,Halimi_2019_CVPR} among others (see also \cite{ovsjanikov2017computing} for an overview). The Laplacian eigenfunctions naturally generalize the Fourier basis to non-Euclidean domains \cite{taubin,levy2006laplace,levy2010} and enjoy many similar properties such being ordered from low to higher frequencies, and spanning the space of $L^2$ functions. In practice, this basis can be computed efficiently on 3D triangle meshes via an eigen-decomposition of the standard cotangent Laplacian matrix \cite{pinkall1993computing}.

A typical pipeline for solving the shape correspondence problem based on the functional map representation consists of the following steps \cite{ovsjanikov2017computing,ren2018continuous,melzi2019zoomout,rodola2017partial}: 1) Establish the basis functions, by computing the first $k$ Laplace-Beltrami eigenfunctions on $\X,\Y$ and store them as columns of matrices $\Phi_{\X}, \Phi_{\Y}$. 2) Compute probe functions $G_{\X}, G_{\Y}$ which are expected to be preserved by the underlying unknown map. 3) Compute the optimal functional map by solving:
\begin{equation}
    C_{\X\Y} = \argmin_{C \in \mathbb{R}^{k \times k}} \| C \Phi_{\X}^{\dagger} G_{\X} - \Phi_{\Y}^{\dagger} G_{\Y} \|_{2} + E_{reg}(C).
    \label{eq:fmaps}
\end{equation}
Here $^{\dagger}$ denotes the Moore Penrose pseudo-inverse, so that, e.g.,  $\Phi_{\X}^{\dagger} G_{\X} $ represents the coefficients of the probe functions in the given basis. The second term in the sum is the regularization on the functional map which promotes some structural properties of the correspondence. For example a popular choice is $E_{reg}(C) = \| C \Delta_{\X}  - \Delta_{\Y} C\|$, which enforces the commutativity between the functional map and Laplace-Beltrami operators (expressed in the respective basis), thereby promoting near-isometric point-to-point correspondences \cite{ovsjanikov2012functional}. Finally, 4) refine the functional map computed in the previous step and convert it to a dense correspondence, e.g., via nearest neighbor search \cite{ovsjanikov2012functional}.

One of the advantages of this approach is that the optimization in step 3) can be performed efficiently since $C_{\X\Y}$ is a matrix of size $k \times k$, which is typically much smaller than the number of points.

\paragraph{Limitations} The main limitations of this pipeline are two-fold: first, the quality of the map is strongly tied to the choice of probe functions, and second, the choice of the basis plays a fundamental role both for the expressive power and the accuracy of the final results. Several approaches have been proposed to learn the probe functions from data \cite{litany2017deep,Halimi_2019_CVPR,roufosse2019unsupervised,donati2020deep}. However, as mentioned above, no existing methods have attempted to learn the basis. This is particularly problematic since, as we show below, as the Laplacian eigen-basis is not only tied to near-isometric deformations, even more fundamentally, it can only be reliably computed on shapes represented as triangle meshes. While some attempts (e.g., in \cite{melzi2019zoomout,rodola2017partial}) have been made to compute eigenfunctions using existing discretizations of Laplace-Beltrami operators on point clouds, e.g., \cite{belkin2009constructing,liang2013solving}. Nevertheless, in part due to the \textit{differential nature} of the Laplacian, such discretizations cannot handle even mild levels of noise, in practice.

\subsection{Motivation and Overview}
Our main goal is to learn an optimal basis that can be used within the functional map pipeline on point cloud data. One possibility would be to use triangle meshes and learn a discretization of the Laplacian that would approximate the low frequency basis functions on point clouds. However, this requires differentiating through a sparse eigen-decomposition, which can be expensive and unstable. 

Instead, we propose an end-to-end learnable pipeline that uses a dual point of view. We summarize our overall pipeline in Figure \ref{fig:overview}. Our first remark is that entries of the basis functions can be interpreted as an embedding of the original 3D shape into a higher $k$-dimensional space. Namely, each point $x \in \X$ gets associated with a $k$-dimensional vector $[\phi^{\X}_1(x), \phi^{\X}_2(x), \ldots , \phi^{\X}_k(x)]$. This is called the ``spectral'' embedding and it is well-known (see e.g., \cite{rustamov2007laplace}) that when using the Laplacian basis on smooth surfaces, as $k\rightarrow \infty$ this embedding becomes injective, so that no two points can have the same associated vectors. 

The spectral embedding plays a role in the conversion between functional and pointwise maps. The standard approach for this conversion  \cite{ovsjanikov2012functional} is by mapping Dirac $\delta_x$ functions associated with each point $x$ on the source shape and finding the nearest Dirac $\delta$ function on the target. Interestingly, $\delta_x$ is \emph{not} a real-valued function but is rather a \emph{distribution}, which acts on real-valued functions through inner products: $<\delta_{x}, f> = f(x)$. As functional maps are operators that map real-valued functions, in principle they \textit{cannot} be used to transport Dirac $\delta$'s.  To transport such distributions, a more sound approach is to use the \textit{adjoint} operator of a functional map \cite{huang2017adjoint}. Surprisingly, although the notion of the adjoint has been studied , both its role and the limitations of functional maps in transferring $\delta$ functions seems to have been ignored in the functional maps literature so far.
The adjoint operator is defined implicitly as follows: given a functional map $C_{\Y\X}$, its adjoint  $A_{\X\Y}$ is defined so  for any pair of real-valued functions $f\in \mathcal{F}(\X)$ and $g\in \mathcal{F}(\Y)$ :  $<C_{\Y\X}g, f> = <g, A_{\X\Y}f>$.  
 We refer to the supplementary for a more complete treatment of the adjoint operator. Note that the adjoint operator: 1) associates functions in the opposite direction to that of the functional map, and 2) is defined using the $L_2$ inner products, and can thus be used to transport distributions. It is easy to see that the adjoint of the pull-back of a point-to-point map $T_{\X\Y}$ (see proof in the supplementary material) has the following nice property: $A_{\X\Y}\delta_x = \delta_{T_{\X\Y}(x)}$. 

Finally, we note that the coefficients of Dirac $\delta$ function $\delta_x$ are precisely the vector of values $[\phi^{\X}_1(x), \phi^{\X}_2(x), \ldots , \phi^{\X}_k(x)]$. Moreover, the adjoint is a linear operator that associates $\delta$ functions with $\delta$ functions. As such, the adjoint can be seen as a linear transformation that aligns the spectral embeddings of $\X$ and $\Y$. We emphasize that the same \textit{does not hold} for a functional map, in general.

This discussion implies that in the functional map framework, the basis can be interpreted as an embedding, and moreover the corresponding embeddings are related by a linear transformation, which is precisely the adjoint of the functional map.

\paragraph{Strategy} Our overall strategy is to  mimic this construction using a learning-based approach. We propose to train a network that computes for each shape an embedding into some $k$ dimensional space, such that the embeddings of two shapes are related by a linear transformation. We then train a separate network that computes probe functions that can be used for establishing the optimal linear transformation at test time. Remarkably, this decomposition of the problem 
consistently outperforms a baseline approach that aims to compute a canonical embedding, in which correspondences can be obtained through nearest neighbor search directly. As described below, we attribute this primarily to the fact that learning a canonical embedding is a difficult problem, and splitting it into two parts (invariant embedding + transformation) helps to regularize the problem in challenging practical settings. Note that we use the term ``basis'' only by analogy with the Laplace-Beltrami eigenfunctions, and do not formally impose a basis structure on our learned set of functions.




 \section{Linearly-invariant embedding}
\label{sec:method}
In this section, we propose a novel learning strategy to generalize the functional maps framework to noisy and incomplete data.

We discretize a shape $\X$ as a collection of 3D points $x_{i}\in \mathbb{R}^{3}$ where $ i \in \{1, \ldots , n_{\X}\}$. We collect these $n_{\X}$ points in a matrix $P_{\X}\in \mathbb{R}^{n_{\X}\times3}$ such that the $i$-th row of $P_{\X}$ captures the 3D coordinates of $x_{i}$.
We refer to the matrix $P_{\X}$ as the \emph{natural} embedding of $\X$.


Given a pair of shapes $\X$ and $\Y$  our goal is to find a correspondence between them. This correspondence is a mapping between the points of $\X$ and the points of $\Y$. We denote a correspondence as a map $T_{\X\Y}:\X \rightarrow \Y$ such that $T_{\X\Y}(x_i) = y_j$, $\forall i \in \{1, \ldots , n_{\X}\}$ and some $j  \in \{1, \ldots , n_{\Y}\}$.  This map has a natural matrix representation $\Pi_{\X\Y} \in \mathbb{R}^{n_{\X}\times n_{\Y}}$ such that $\Pi_{\X\Y}(i,j) = 1$ if  $T_{\X\Y}(x_i) = y_j$ and $0$ otherwise. 

Let $\Phi_{\X}$ and $\Phi_{\Y}$ denote the matrices, whose rows can be interpreted as embeddings of the points of $\X$ and $\Y$ as described in Section \ref{sec:background}. Below we do not assume that $\Phi_{\X}$ and $\Phi_{\Y}$ represent the Laplacian eigenbasis, but consider general embeddings into some fixed $k$ dimensional space. Recall that in the formalism of functional maps, there must exist a linear transformation $A_{\X\Y}$ that aligns the corresponding embeddings. This can be written as: $A_{\X\Y} \Phi^T_{\X}  = (\Pi_{\X\Y} \Phi_{\Y})^T$, where $\Pi_{\X\Y}$ is the binary matrix that encodes the correspondence between $\X$ and $\Y$. 
In the functional map framework, the linear transformation $A_{\X\Y}$ is precisely the adjoint operator, since $A_{\X\Y} = (\Phi_{\X}^{+}\Pi_{\X\Y} \Phi_{\Y})^T = C_{\Y\X}^T$ using the standard definition of a functional map $C_{\Y\X}$ \cite{ovsjanikov2017computing}.

Given $A_{\X\Y}$, we can estimate $\Pi_{\X\Y}$ by solving the following optimization problem:
\begin{equation}
\Pi_{\X\Y} = \argmin_{\Pi} \|\Phi_{\X}  A_{\X\Y}^T  - \Pi \Phi_{\Y} \|_2.
    \label{eq:nnassignment}
\end{equation}
Note that Eq. \eqref{eq:nnassignment} can be solved in closed form by finding, for every row of $\Phi_{\X}  A_{\X\Y}^T$, the closest row in $\Phi_{\Y}$ in the standard $L_2$ sense.

Based on Equation \eqref{eq:nnassignment}, our general goal is to train a network $\mathcal{N}$ that can produce for any shape $\X$ an embedding  $\Phi_{\X}$ into a $k$-dimensional space, such that embeddings of every pair of shapes $\Phi_{\X}, \Phi_{\Y}$ are related by a linear transformation. In other words the network $\mathcal{N}$ must be able to transform a shape from the original 3D space, in which complex non-rigid deformations could occur, to another space, in which transformations across shapes must always be linear. Interestingly, as we show below, the additional linear degree of freedom helps to regualize the learning procedure, achieving better results than simply learning a canonical embedding in which corresponding points are nearest neighbors.

\setlength{\columnsep}{13pt}
\setlength{\intextsep}{1pt}
\begin{wrapfigure}[5]{r}{0.45\linewidth}
\vspace{-0.5cm}
\begin{center}
\begin{overpic}
[trim=0cm 0cm 0cm 0cm,clip,width=1.0\linewidth]{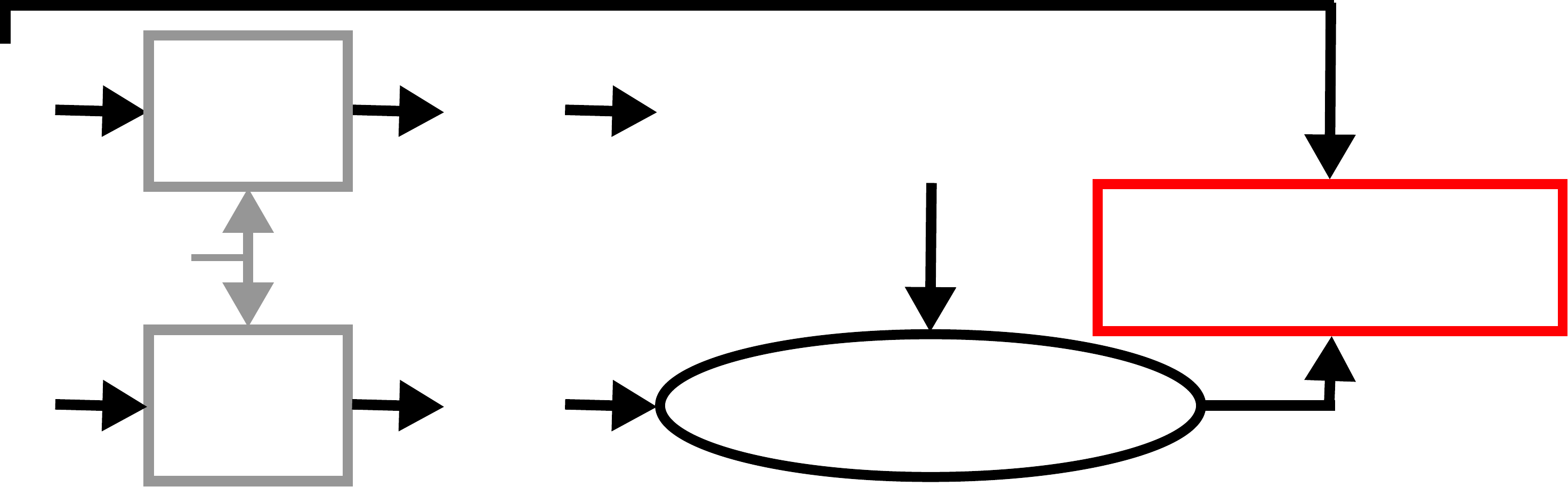}
\put(-4,22){\footnotesize $P_{\X}$}
\put(-4,3.4){\footnotesize $P_{\Y}$}
\put(12.8,22){\footnotesize \mygray{$\mathcal{N}$}}
\put(12.8,3.4){\footnotesize \mygray{$\mathcal{N}$}}
\put(7.5,12.5){\footnotesize \mygray{$\Theta$}}
\put(28.5,22){\footnotesize $\Phi_{\X}$}
\put(28.5,3.4){\footnotesize $\Phi_{\Y}$}
\put(42.2,22){\footnotesize $\widehat{\Phi}_{\X} = \Phi_{\X} A_{\X\Y}^T$}
\put(46.1,3.4){\footnotesize \emph{softmap} $S_{\X\Y}$}
\put(72.7,12.5){\footnotesize \rick{$L(\Phi_{\X}, \Phi_{\Y})$}}
\end{overpic}
\end{center}
\end{wrapfigure}
\subsection{Learning a linearly-invariant embedding}
\label{subsec:learnembedding}

To learn a linearly-invariant embedding we first observe that for fixed matrices $\Phi_{\X},\Phi_{\Y}$ the expression $\|\Phi_{\X}  A_{\X\Y}^T  - \Pi_{\X\Y} \Phi_{\Y} \|_2$ depends both on $A_{\X\Y}^T$ and $\Pi_{\X\Y}$, which can make training difficult. However, for a fixed correspondence matrix $\Pi_{\X\Y}$ the optimal matrix $A_{\X\Y}$ can be obtained in closed form simply as: $A_{\X\Y} = (\Phi_{\X}^{+} \Pi_{\X\Y} \Phi_{\Y})^T$, which can be computed by solving a linear system of equations. Importantly, this procedure can be differentiated using the closed-form expression of derivatives of matrix inverses, which we exploit in our approach.

\paragraph{Embedding network training}
Given a set of training pairs of shapes $\X,\Y$ for which ground truth correspondences $\Pi^{gt}_{\X\Y}$ are known, our embedding network $\mathcal{N}$ computes an embedding  $\Phi_{\X}, \Phi_{\Y}$ for each shape using a Siamese architecture with shared parameters. I.e., $\mathcal{N}_{\Theta}(P_{\X}) = \Phi_{\X}$ and $\mathcal{N}_{\Theta}(P_{\Y}) = \Phi_{\Y}$. We use the notation $\mathcal{N}_{\Theta}$ to highlight that this network has trainable parameters $\Theta$ which are shared across shapes. In the following we refer to this network as simply $\mathcal{N}$. The exact details of the architecture that we use are provided in the supplementary.

In order to define our loss we first compute the optimal linear transformation $A_{\X\Y} = (\Phi_{\X}^{+} \Pi^{gt}_{\X\Y} \Phi_{\Y})^T$ and use it to obtain a \emph{transformed} embedding $\widehat{\Phi}_{\X} = \Phi_{\X} A_{\X\Y}^T$. We then compare the rows of $\widehat{\Phi}_{\X}$ to those of $\Phi_{\Y}$ to obtain the \textit{soft} permutation matrix $S_{\X\Y}$ that approximates the discrete mapping between the shapes in a differentiable way using the \emph{softmax} operation (for completeness see details in the supplementary). Finally, we use the following loss to train the embedding network:
\begin{equation}
\begin{aligned}
 L(\Phi_{\X}, \Phi_{\Y}) =  \frac{1}{n_{\beta}} \sum \|  S_{\X\Y} P_{\X} - \Pi^{gt}_{\X\Y} P_{\X} \|_2^2.
    \end{aligned}
    \label{eq:loss_spatial}
\end{equation}
Recall that $P_{\X}$ is the matrix encoding the 3D coordinates of the shape $\X$. 

Intuitively, the main goal of the loss in Eq. \eqref{eq:loss_spatial} is to compare the ground truth correspondence $\Pi^{gt}_{\X\Y}$ to the computed softmap matrix $S_{\X\Y}$. An alternative would to use the geodesic distances as weights as done in \cite{litany2017deep}, but the computation of the geodesic distances is expensive and unreliable in the context of point clouds. Other possible solutions are the direct Frobenius loss on the permutation matrix or a \emph{multinomial
regression loss} as done in e.g. \cite{GCNN,poulenard2018multi}. However, these losses do not involve the geometry and penalize incorrect correspondences independently of their proximity to correct ones.  Instead, our loss penalizes incorrect correspondences based on the Euclidean distances of associated points.  Moreover, Eq. \eqref{eq:loss_spatial} can be seen as the comparison between the action of the ground-truth functional map in the full basis and the action of the estimated functional map on a specific set of functions that completely describe the geometry of the data. As such, our loss is efficient, takes the geometry into account, and is directly related to the functional map formalism.




\setlength{\columnsep}{16pt}
\setlength{\intextsep}{1pt}
\begin{wrapfigure}[5]{r}{0.42\linewidth}
\vspace{-0.7cm}
\begin{center}
\begin{overpic}
[trim=0cm 0cm 0cm 0cm,clip,width=1.0\linewidth]{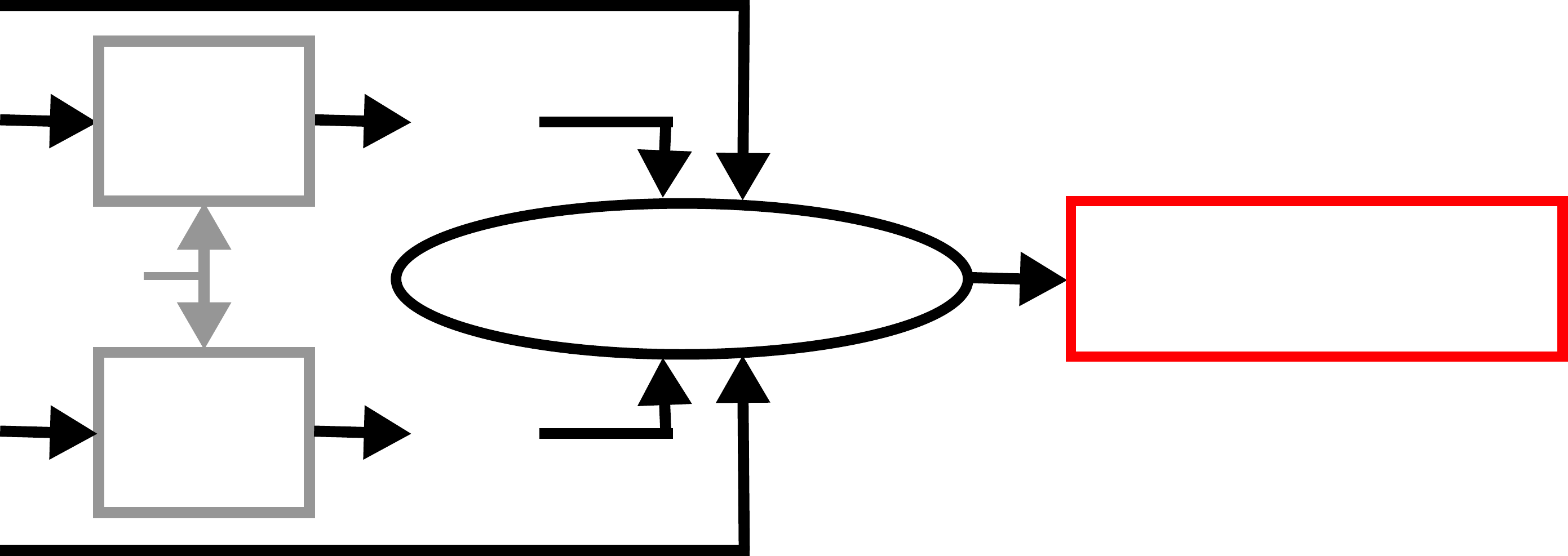}
\put(-6.3,25){\footnotesize $P_{\X}$}
\put(-6.3,5.4){\footnotesize $P_{\Y}$}
\put(-6.3,32.3){\footnotesize $\Phi_{\X}$}
\put(-6.3,-2.3){\footnotesize $\Phi_{\Y}$}
\put(10.8,25){\footnotesize \mygray{$\mathcal{G}$}}
\put(10.8,5.4){\footnotesize \mygray{$\mathcal{G}$}}
\put(4.3,16.2){\footnotesize \mygray{$\Theta$}}
\put(26.5,25){\footnotesize $G_{\X}$}
\put(26.5,5.4){\footnotesize $G_{\Y}$}
\put(29.9,16.2){\footnotesize \emph{adjoint} $\widehat{A}_{\X\Y}$}
\put(70.6,16.2){\footnotesize \rick{$L(G_{\X}, G_{\Y})$}}
\end{overpic}
\end{center}
\end{wrapfigure}
\subsection{Learning the optimal transformation}
As mentioned above, we train our approach in two stages: first we train an embedding network using the loss described in Section \ref{subsec:learnembedding}. We then train a separate network that aims to compute an optimal linear transformation between the embeddings, which can be used to compute correspondences at test time. Our observation is that this linear transformation can be obtained given enough constraints, by solving a linear system. Therefore, following the ideas in Deep Functional Maps \cite{litany2017deep} our second network $\mathcal{G}$ takes as input the natural embedding of a shape and outputs a set of $p$ ``probe'' functions via $\mathcal{G}_{\Theta}(P_{\X}) = G_{\X}$ and $\mathcal{G}_{\Theta}(P_{\Y}) = G_{\Y}$ using shared trainable parameters $\Theta$. We then minimize the following loss:
  \begin{equation}
 \begin{aligned}
L(G_{\X}, G_{\Y}) & =  \| A_{\X\Y}^{gt} - \widehat{A}_{\X\Y}\|_{2}.
    \end{aligned}
  \label{eq:loss_descriptor}
 \end{equation}
 Here $A_{\X\Y}^{gt}$ is the ground truth linear transformation between the learned embeddings $A_{\X\Y}^{gt} = (\Phi_{\X}^{+} \Pi^{gt}_{\X\Y} \Phi_{\Y})^T$ whereas $\widehat{A}_{\X\Y} = \left((\Phi_{\Y}^{\dagger} G_{\Y})^{T}\right)^{\dagger}(\Phi_{\X}^{\dagger} G_{\X})^{T}$. This equation arises from the fact that if $A_{\X\Y}$ is the adjoint that aligns the embeddings then $A_{\X\Y}^{T}$ is a functional map from $\Y$ to $\X$ which implies that $A_{\X\Y}^{T} \Phi_{\Y}^{\dagger} G_{\Y} = \Phi_{\X}^{\dagger} G_{\X}$ whenever $G_{\X}, G_{\Y}$ are corresponding functions. Please see the supplementary material for a more detailed discussion.

\subsection{Test phase}
Once we train these two networks, we can estimate the correspondence between an arbitrary pair of point clouds $\X$ and $\Y$ in four steps: (1) compute the embeddings $\Phi_{\X}$ and $\Phi_{\Y}$ using the embedding network $\mathcal{N}$; (2) compute the set of probe functions, $G_{\X}$ and $G_{\Y}$ using the network $\mathcal{G}$; (3) solve for the linear transformation $A_{\X\Y}$ using the expression given for $\widehat{A}_{\X\Y}$ above; (4) estimate for the correspondence $\Pi_{\X\Y}$ via nearest neighbor search as described in Eq. \eqref{eq:nnassignment}. 

\paragraph{Discussion}
While the basis and probe function networks appear similar as they both output a matrix, they are different in their losses and, as consequence, in the task that they solve. Our first linearly-invariant embedding (basis) network aims to output a representation in $k$ dimensions so that different shapes share the same structure up to rotation and non-uniform scaling. Further, our loss in Eq \eqref{eq:loss_spatial} promotes continuity of the embedding with respect to the original shape coordinates. In contrast, the descriptor network aims to find a small set of reliable descriptors that can establish the linear transformation in the $k$ dimensional space. Our strategy is different from a network which would aim to find an embedding where correspondences are directly obtained as nearest neighbors (we call this option a ``universal embedding''), as such a network would have to  disambiguate each point directly. Instead, by first obtaining a smooth embedding and then using a small number of salient feature descriptors (probe functions in our case) our approach allows to find a dense correspondence even in challenging cases, in which individual points may not be easy to distinguish.


\section{Experiments}
\label{sec:results}
We evaluate our pipeline on the correspondence problem between non-rigid 3D point clouds in the challenging class of human models. We use this class because of the availability of data and baselines for comparison but stress that our method is general and can be applied to any shape category.

\paragraph{Architecture and parameters}
Both of our networks $\mathcal{N}$ and $\mathcal{G}$ are built upon the PointNet architecture \cite{qi2017pointnet}. 
For our experiments we train over $10$K shapes from the SURREAL dataset \cite{varol17_surreal}, resampled at $1$K vertices. We learn a $k=20$ dimensional embedding (basis) and $p=40$ probe functions for each point cloud. 
We report in Supplementary Materials the complete description of the architectures and the training data.

\subsection{Non-isometric pointclouds}
\begin{figure}[!t]
\hspace{-0.7cm}
\begin{tabular}{cccccccc}
    \centering
    \begin{minipage}{0.27\linewidth}
    \input{./figures/PC_000_new_wide.tikz} 
    \end{minipage}
    &
    \hspace{0.14cm}
    \begin{minipage}{0.27\linewidth}
     
    \input{./figures/PC_001_new_wide.tikz} 
    \end{minipage}    
    &
    \hspace{-0.3cm}
    
\begin{minipage}{0.47\linewidth}
     \vspace{0cm}
 \begin{overpic}[trim=0cm 0cm 60cm 0cm,clip,width=\linewidth]{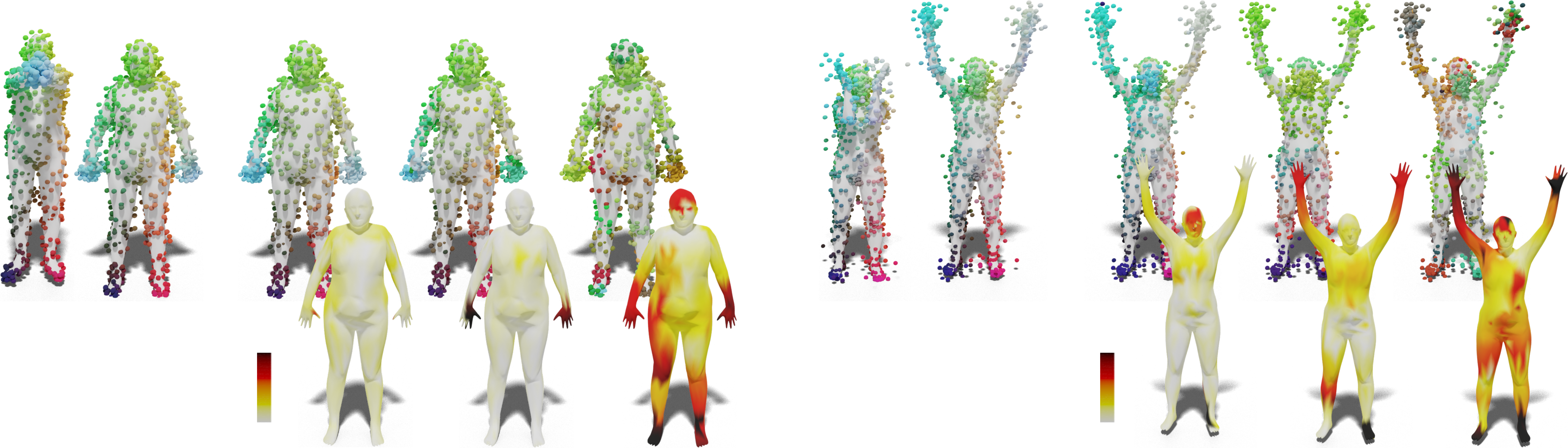}
 \put(3,53){N}
 \put(14,53){GT}
 \put(33,53){\textbf{our}}
 \put(53,53){3DC} 
 \put(70,53){GFM} 
 
  \put(27,10){\tiny 0.6}
   \put(30,2){\tiny 0} 
   \end{overpic} 
\end{minipage}
\end{tabular}
    \caption{The evaluation of the correspondence for point clouds generated from the FAUST dataset without or with additional noise. On the left, cumulative curves with mean error in the legends. On the right, a qualitative example in Noise setup, with the related hotmap error.}
    \label{fig:pointclouds_comp}
\end{figure}

We consider a first test set composed by the $100$ shapes from the FAUST dataset \cite{FAUST} ($10$ different subjects in $10$ different poses). We treat each shape as an unorganized point cloud selecting only $1K$ of its vertices and discarding mesh  connectivity. We generate a second test set perturbing the first one with Gaussian noise. 
In both test sets, we deal with non-isometric pairs (different subjects) and strong non rigid deformations (different poses).
The second one is particularly challenging because it ruins the underlying shape structure.
\rick{
As competitive baselines we consider \emph{universal embeddings} (Uni20 and Uni60) obtained with the same architecture we used for $\mathcal{N}$ by learning $20$ and $60$ basis respectively,} but enforcing the optimal linear transformation to be identity.
We also compare our method with the standard functional maps, with $5$ ground-truth landmarks (FMAP), the recent state-of-the-art methods (GFM) \cite{donati2020deep}, and finally against 3D-CODED \cite{groueix20183d} (3DC). For the GFM and FMAP methods we also compare to a version refined with ZoomOut \cite{melzi2019zoomout} (FMAP+ZOO, GFM+ZOO). For the methods that require the LBO basis, we adopt the estimation of LBO for point clouds proposed in \cite{clarenz2004finite}.
As can be seen in Figure \ref{fig:pointclouds_comp}, we outperform the baseline and all the competitors including the the state-of-the-art methods GFM and 3DC in both the considered scenarios. We stress that both \cite{groueix20183d} and \cite{donati2020deep} are very recent highly complex state-of-the-art methods, with e.g. \cite{groueix20183d} being directly adapted to point clouds with an expensive test-time post-processing. Our method achieves state-of-the-art results without any additional post-processing. \rick{Further robustness of our method is illustrated in Figure \ref{fig:Outliers}, where we evaluate our networks trained on clean data, on the FAUST test set augmented with outliers points. Our method shows significant resilience and outperforms competing methods in this challenging setting, despite not being presented with outlier data at training time.}

\begin{figure}[!t]
\hspace{-1.3cm}
\centering
  \setlength{\tabcolsep}{0pt}
  
  \begin{tabular}{c c c c c c c}

\begin{tabular}{|l|c|}
    
    \hline
     & Outlier \\
    \hline
    \textbf{our}             & \textbf{2.7e-1} \\
    Uni60             & 3.5e-1 \\
    GFM          & 3.5e-1\\
    \hline
\end{tabular}

&

\hspace{1cm}

      \begin{minipage}{0.11\linewidth} 
      \vspace{0.5cm}
     \begin{overpic}[trim=0cm 0cm 0cm 0cm,clip,width=0.95\linewidth]{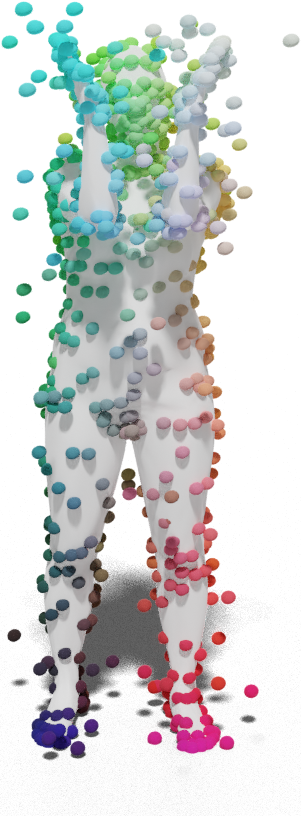}
    \put(12,105){{N}}

  \end{overpic}
   
    \end{minipage}%
    
     &
      \begin{minipage}{0.15\linewidth} 
     \begin{overpic}[trim=0cm 0cm 0cm 0cm,clip,width=0.95\linewidth]{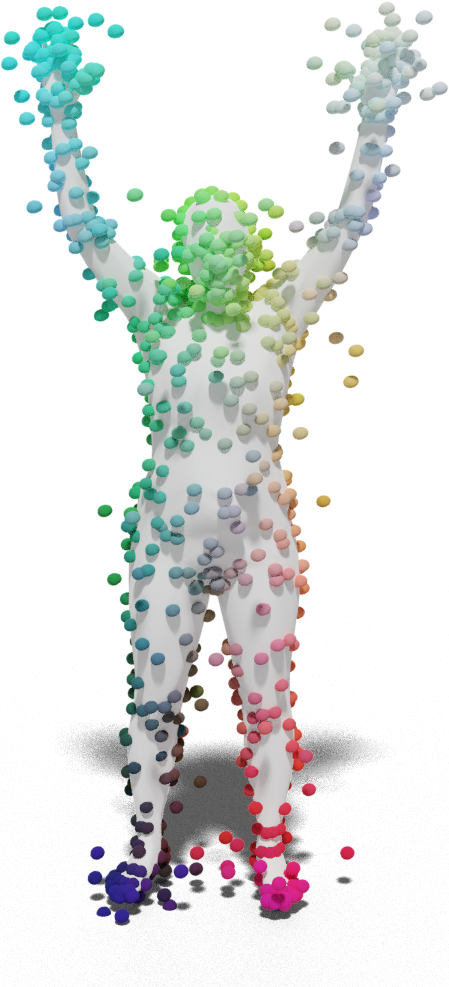}
    \put(16,96){{GT}}
  \end{overpic}

    \end{minipage}%
    
        &

      \begin{minipage}{0.15\linewidth} 
      
     \begin{overpic}[trim=0cm 0cm 0cm 0cm,clip,width=0.95\linewidth]{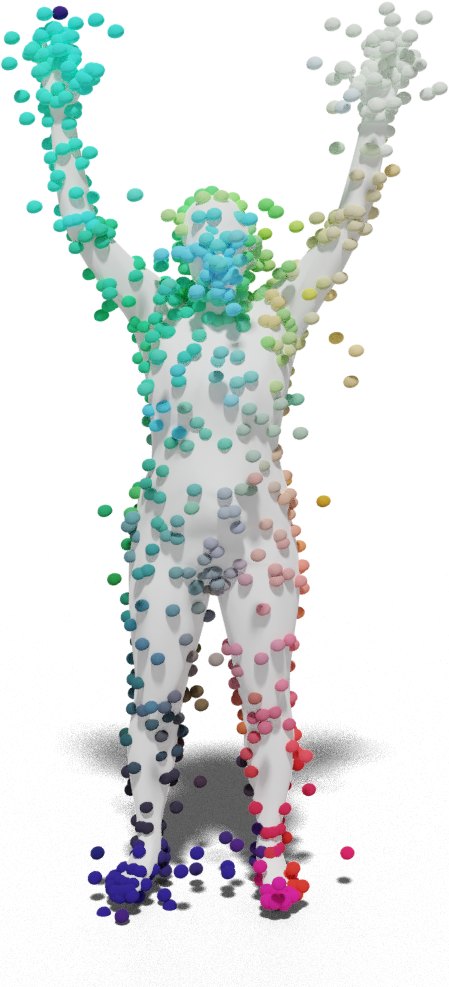}
    \put(16,96){{\textbf{our}}}

  \end{overpic}
     
    \end{minipage}%
          &
     
      \begin{minipage}{0.15\linewidth} 
     \begin{overpic}[trim=0cm 0cm 0cm 0cm,clip,width=0.95\linewidth]{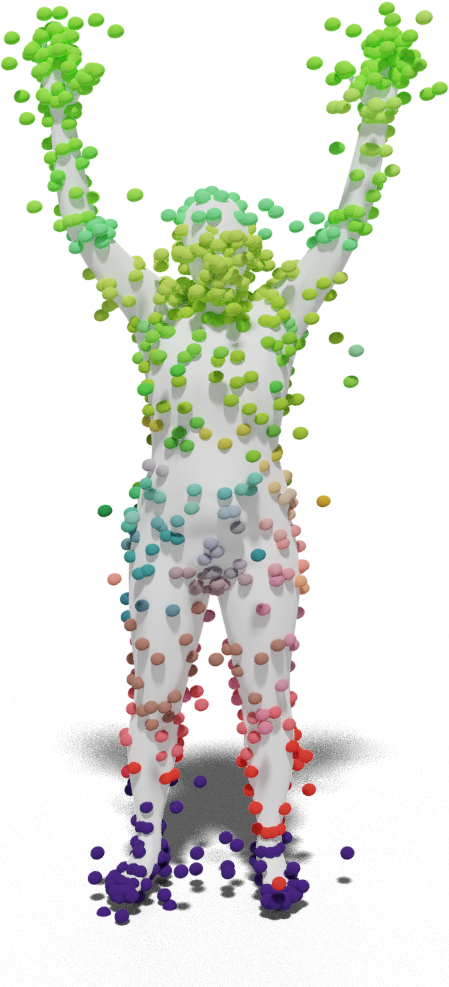}
    \put(14,96){{Uni60}}

  \end{overpic}
     
    \end{minipage}%
    
         &
        
      \begin{minipage}{0.15\linewidth} 
     \begin{overpic}[trim=0cm 0cm 0cm 0cm,clip,width=0.95\linewidth]{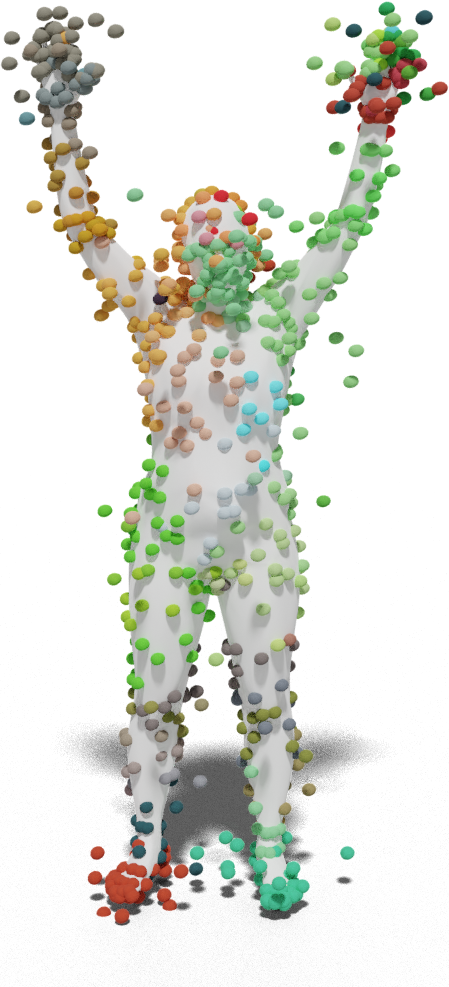}
\put(15.3,96){{GFM}}
  \end{overpic}

    \end{minipage} 

  \end{tabular}
 \caption{\label{fig:Outliers}
Qantitative results on 100 pairs of the test set with 30\% outlier points, compared to the baselines, with a qualitative example.
\vspace{-3mm}} 
\end{figure}

\begin{wrapfigure}[18]{r}{0.4\linewidth}

\vspace{-0.4cm}

    \begin{tabular}{cc}
    \hspace{-0.1cm}
    
    \begin{minipage}{0.45\linewidth}
    \hspace{0.6cm}
     \begin{overpic}[trim=0cm 0cm 0cm 0cm,clip,width=0.93\linewidth]{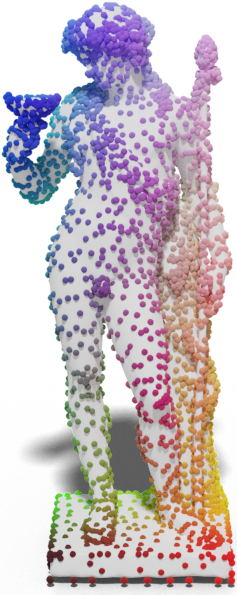}
       \end{overpic}
\end{minipage}
&
    \begin{minipage}{0.55\linewidth}
    \vspace{0.5cm}
    \hspace{-0.5cm}
     \begin{overpic}[trim=0cm 0cm 0cm 0cm,clip,width=0.93\linewidth]{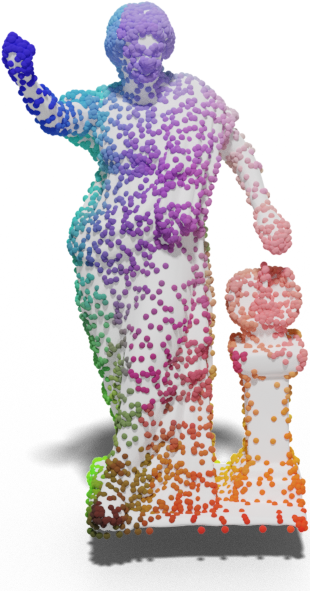}
       \end{overpic}
\end{minipage}
    \end{tabular}

     \vspace{-0.1cm}
     
     \caption{\label{fig:statues} A qualitative example of matching between two statues. Despite the presence of clutter, partiality and non-isometry our point cloud-based approach shows resilience.}
     
\end{wrapfigure}
\rick{
In addition, in Figure \ref{fig:statues} we also visualize a correspondence, computed using our network, between a pair of real-world scans taken from the \emph{Scan the world} project collection \cite{scantheworld}. The presence of significant topological changes, partiality, clutter, non-isometry and self-intersections represent a significant challenge. Despite this, our method, shows remarkable resilience and provides a reliable result even without retraining or post-processing. We provide more illustrations on real scans in the Supplementary Materials.
}

\subsection{Fragmented partiality}
Finally, we compare our approach, the universal embedding and the LBO basis (LBO) in an extreme scenario. We compute a correspondence between each of the $100$ full shapes from FAUST and a fragmented version that consists of several small disconnected components. This experiment tests
how each basis is affected by heavy loss of geometry.
Fixing a basis, we evaluate 1) the matching using a ground-truth transformation to retrieve the optimal linear transformation, on the left of Figure \ref{fig:Partial}; 2) the correspondence estimated with the best pipeline for the given basis, on the right. 
The average geodesic errors are reported in the legends. 
In 2) for LBO  we consider partial functional maps  (PFM) \cite{rodola2017partial}, which extends the functional maps framework to partial cases.
In the middle we visualize a qualitative comparison on one of the 100 pairs tested, where the correspondence in encoded by the color transfer.
We  highlight that it is not always possible to have a transformation that produces a perfect matching.
LBO+opt and PFM suffer from the significant sensitivity of the LBO to partiality and topological noise. The universal embedding shows also a significant loss of information. With the linear invariant embedding it is still possible to retrieve good information and to generalize to corrupted data that are completely unseen during training.

\begin{figure}[!t]
\centering

  \setlength{\tabcolsep}{0pt}
  \begin{tabular}{l | c c c c c | r}
\hspace{-0.3cm}
  
\begin{minipage}{0.19\linewidth}
    \input{./figures/partials.tikz}
\end{minipage}

\hspace{0.95cm}

&
\hspace{0.05cm}

\begin{minipage}{0.1\linewidth}
 \begin{overpic}[trim=0cm 0cm 0cm 0cm,clip,width=0.93\linewidth]{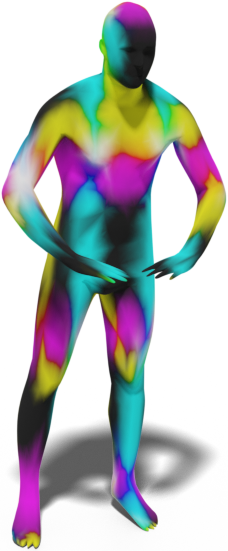}
  \put(12,-9){\footnotesize{{N}}}
   \end{overpic}
\end{minipage}
&
\begin{minipage}{0.1\linewidth}
 \begin{overpic}[trim=0cm 0cm 0cm 0cm,clip,width=0.93\linewidth]{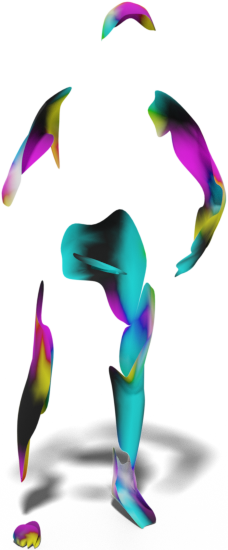}
  \put(10,-9){\footnotesize{{GT}}}
   \end{overpic}
\end{minipage}
&
\begin{minipage}{0.1\linewidth}
 \begin{overpic}[trim=0cm 0cm 0cm 0cm,clip,width=0.93\linewidth]{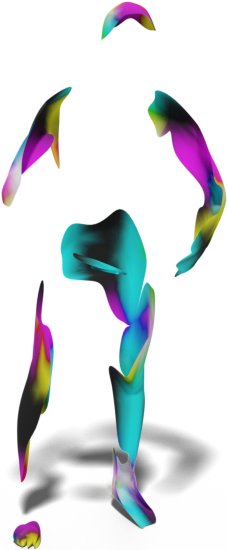}
  \put(8,-9){\footnotesize{\textbf{OUR}}}
  \put(7.8,-18){\footnotesize{\textbf{+Opt}}}
   \end{overpic}
\end{minipage}
&
\begin{minipage}{0.1\linewidth}
 \begin{overpic}[trim=0cm 0cm 0cm 0cm,clip,width=0.93\linewidth]{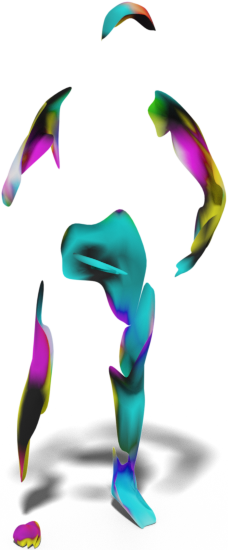}
  \put(8.5,-9){\footnotesize{{Uni20}}}
  \put(6.8,-18){\footnotesize{{+Opt}}}
   \end{overpic}
\end{minipage}
&
\begin{minipage}{0.1\linewidth}
 \begin{overpic}[trim=0cm 0cm 0cm 0cm,clip,width=0.93\linewidth]{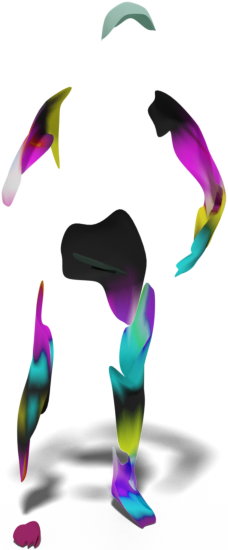}
  \put(8,-9){\footnotesize{{LBO}}}
  \put(7.8,-18){\footnotesize{{+Opt}}}
   \end{overpic}
\end{minipage}
&
\begin{minipage}{0.19\linewidth}
    \input{./figures/rigid_partials.tikz}
\end{minipage}

\hspace{1cm}
\end{tabular}
\caption{\label{fig:Partial}
Partial setup. The shape is matched with a fragmented version of itself. We show the amount of information lost by the basis due to surface destruction and compare our method to baseline and partial functional map (pfm) \cite{rodola2017partial}. More details in the main text.
\vspace{-3mm}}
\end{figure}

\section{Conclusion}
\label{sec:conclusion}
In this paper, we presented an extension to the functional maps framework by replacing the standard Laplace-Beltrami eigenfunctions with learned functions. We achieve this by learning an optimal linearly-invariant embedding and a separate network that aligns embeddings of different shapes.

While general, our approach still assumes that the input data poses a ``natural'' embedding in 3D making it yet not applicable to data such as graphs. Moreover, we do not exploit the mesh structure that \emph{might} be available in certain cases. Combining our method with a mesh-aware approach is an interesting direction for future work. 
Our preliminary investigation outperforms the competitors in challenging scenarios. We believe that these results only scratch the surface and can pave the way to future work on invariant embeddings for shape correspondence and other related problems.

\section{Broader Impact}
\label{sec:impact}
Computing reliable correspondences is a problem that arises in many scientific disciplines and practical scenarios including medical imaging and industrial quality control (for detecting anomalies and performing repair and analysis), as well as 3D animation and texture transfer, statistical shape analysis, and even personalized medicine (with accurate detection of measurements often performed by template matching). Our novel fully trainable pipeline paves the way to more accurate results with direct practical applications in all of these fields, especially as it is applicable to arbitrary 3D shapes and deformations, unlike many existing methods which are specific e.g. to humans or near-isometries. This has the potential to replace highly specific axiomatic methods and tedious manual intervention. Finally, our insights can shed light on the structure of functional maps and shape analysis more broadly. We do not see any ethical issue with the proposed method, at least no ethical issues may be caused by our method as it is. 
\section{Acknowledgements}
\label{sec:Acknowledgements}
The authors would like to thank the anonymous reviewers for their detailed feedback and suggestions. Parts of this work were supported by the KAUST OSR Award No. CRG-2017-3426, the ERC Starting Grant No. 758800 (EXPROTEA) and the ANR AI Chair AIGRETTE. 

%

\bibliographystyle{splncs04}
\bibliography{biblio}

\begin{thebibliography}{10}
\providecommand{\url}[1]{\texttt{#1}}
\providecommand{\urlprefix}{URL }
\providecommand{\doi}[1]{https://doi.org/#1}

\bibitem{scantheworld}
Scan the world project. \url{https://www.myminifactory.com/scantheworld},
  [Online; accessed 22-October-2020]

\bibitem{Atzmon}
Atzmon, M., Maron, H., Lipman, Y.: Point convolutional neural networks by
  extension operators. ACM Trans. Graph.  \textbf{37}(4) (Jul 2018)

\bibitem{aubry2011wave}
Aubry, M., Schlickewei, U., Cremers, D.: The wave kernel signature: A quantum
  mechanical approach to shape analysis. In: Computer Vision Workshops (ICCV
  Workshops), 2011 IEEE International Conference on. pp. 1626--1633. IEEE
  (2011)

\bibitem{belkin2009constructing}
Belkin, M., Sun, J., Wang, Y.: Constructing laplace operator from point clouds
  in rd. In: Proceedings of the twentieth annual ACM-SIAM symposium on Discrete
  algorithms. pp. 1031--1040. Society for Industrial and Applied Mathematics
  (2009)

\bibitem{survey0}
Biasotti, S., Cerri, A., Bronstein, A., Bronstein, M.: Recent trends,
  applications, and perspectives in 3d shape similarity assessment. Computer
  Graphics Forum  \textbf{35}(6),  87--119 (2016)

\bibitem{FAUST}
Bogo, F., Romero, J., Loper, M., Black, M.J.: {FAUST}: Dataset and evaluation
  for {3D} mesh registration. In: Proceedings IEEE Conf. on Computer Vision and
  Pattern Recognition (CVPR). IEEE, Piscataway, NJ, USA (Jun 2014)

\bibitem{boscaini2016anisotropic}
Boscaini, D., Masci, J., Rodol{\`a}, E., Bronstein, M.M., Cremers, D.:
  Anisotropic diffusion descriptors. Computer Graphics Forum  \textbf{35}(2),
  431--441 (2016)

\bibitem{choy2016universal}
Choy, C.B., Gwak, J., Savarese, S., Chandraker, M.: Universal correspondence
  network. In: Advances in Neural Information Processing Systems. pp.
  2414--2422 (2016)

\bibitem{clarenz2004finite}
Clarenz, U., Rumpf, M., Telea, A.: Finite elements on point based surfaces. In:
  Proceedings of the First Eurographics conference on Point-Based Graphics. pp.
  201--211. SPBG’04, Eurographics Association, Goslar, DEU (2004)

\bibitem{desclearn}
Corman, {\'E}., Ovsjanikov, M., Chambolle, A.: Supervised descriptor learning
  for non-rigid shape matching. In: Agapito, L., Bronstein, M.M., Rother, C.
  (eds.) Computer Vision - ECCV 2014 Workshops. pp. 283--298. Springer
  International Publishing, Cham (2015)

\bibitem{Denitto_2017_ICCV}
Denitto, M., Melzi, S., Bicego, M., Castellani, U., Farinelli, A., Figueiredo,
  M.A.T., Kleiman, Y., Ovsjanikov, M.: Region-based correspondence between 3d
  shapes via spatially smooth biclustering. In: The IEEE International
  Conference on Computer Vision (ICCV) (Oct 2017)

\bibitem{NIPS2019_8962}
Deprelle, T., Groueix, T., Fisher, M., Kim, V., Russell, B., Aubry, M.:
  Learning elementary structures for 3d shape generation and matching. In:
  Wallach, H., Larochelle, H., Beygelzimer, A., d'Alch\'{e} Buc, F., Fox, E.,
  Garnett, R. (eds.) Advances in Neural Information Processing Systems 32, pp.
  7435--7445. Curran Associates, Inc. (2019)

\bibitem{donati2020deep}
Donati, N., Sharma, A., Ovsjanikov, M.: Deep geometric functional maps: Robust
  feature learning for shape correspondence. In: IEEE Conference on Computer
  Vision and Pattern Recognition (CVPR) (June 2020)

\bibitem{ezuz2017deblurring}
Ezuz, D., Ben-Chen, M.: Deblurring and denoising of maps between shapes.
  Computer Graphics Forum  \textbf{36}(5),  165--174 (2017)

\bibitem{fey2020deep}
Fey, M., Lenssen, J.E., Morris, C., Masci, J., Kriege, N.M.: Deep graph
  matching consensus. arXiv preprint arXiv:2001.09621  (2020)

\bibitem{gainza2020deciphering}
Gainza, P., Sverrisson, F., Monti, F., Rodola, E., Boscaini, D., Bronstein, M.,
  Correia, B.: Deciphering interaction fingerprints from protein molecular
  surfaces using geometric deep learning. Nature Methods  \textbf{17}(2),
  184--192 (2020)

\bibitem{StableRegion}
Ganapathi-Subramanian, V., Thibert, B., Ovsjanikov, M., Guibas, L.: Stable
  region correspondences between non-isometric shapes. In: Proceedings of the
  Symposium on Geometry Processing. p. 121–133. SGP ’16, Eurographics
  Association, Goslar, DEU (2016)

\bibitem{ginzburg2019cyclic}
Ginzburg, D., Raviv, D.: Cyclic functional mapping: Self-supervised
  correspondence between non-isometric deformable shapes. arXiv preprint
  arXiv:1912.01249  (2019)

\bibitem{Gojcic_2019_CVPR}
Gojcic, Z., Zhou, C., Wegner, J.D., Wieser, A.: The perfect match: 3d point
  cloud matching with smoothed densities. In: The IEEE Conference on Computer
  Vision and Pattern Recognition (CVPR) (June 2019)

\bibitem{groueix20183d}
Groueix, T., Fisher, M., Kim, V.G., Russell, B.C., Aubry, M.: 3d-coded: 3d
  correspondences by deep deformation. In: Proceedings of the European
  Conference on Computer Vision (ECCV). pp. 230--246 (2018)

\bibitem{groueix2019unsupervised}
Groueix, T., Fisher, M., Kim, V.G., Russell, B.C., Aubry, M.: Unsupervised
  cycle-consistent deformation for shape matching. In: Computer Graphics Forum.
  vol.~38, pp. 123--133. Wiley Online Library (2019)

\bibitem{Halimi_2019_CVPR}
Halimi, O., Litany, O., Rodola, E., Bronstein, A.M., Kimmel, R.: Unsupervised
  learning of dense shape correspondence. In: The IEEE Conference on Computer
  Vision and Pattern Recognition (CVPR) (June 2019)

\bibitem{huang2014functional}
Huang, Q., Wang, F., Guibas, L.: Functional map networks for analyzing and
  exploring large shape collections. ACM Transactions on Graphics (TOG)
  \textbf{33}(4),  1--11 (2014)

\bibitem{huang2017adjoint}
Huang, R., Ovsjanikov, M.: Adjoint map representation for shape analysis and
  matching. In: Computer Graphics Forum. vol.~36, pp. 151--163. Wiley Online
  Library (2017)

\bibitem{jin2019fast}
Jin, Y.H., Lee, W.H.: Fast cylinder shape matching using random sample
  consensus in large scale point cloud. Applied Sciences  \textbf{9}(5), ~974
  (2019)

\bibitem{kovnatsky2013coupled}
Kovnatsky, A., Bronstein, M.M., Bronstein, A.M., Glashoff, K., Kimmel, R.:
  Coupled quasi-harmonic bases. In: Computer Graphics Forum. vol.~32, pp.
  439--448. Wiley Online Library (2013)

\bibitem{levy2006laplace}
Levy, B.: Laplace-beltrami eigenfunctions towards an algorithm that"
  understands" geometry. In: IEEE International Conference on Shape Modeling
  and Applications 2006 (SMI'06). pp. 13--13. IEEE (2006)

\bibitem{levy2010}
Levy, B., Zhang, R.H.: Spectral geometry processing. In: ACM SIGGRAPH Course
  Notes (2010)

\bibitem{li2020toward}
Li, K., Chong, M.J., Liu, J., Forsyth, D.: Toward accurate and realistic
  virtual try-on through shape matching and multiple warps. arXiv preprint
  arXiv:2003.10817  (2020)

\bibitem{liang2013solving}
Liang, J., Zhao, H.: Solving partial differential equations on point clouds.
  SIAM Journal on Scientific Computing  \textbf{35}(3),  A1461--A1486 (2013)

\bibitem{litany2017deep}
Litany, O., Remez, T., Rodol{\`a}, E., Bronstein, A., Bronstein, M.: Deep
  functional maps: Structured prediction for dense shape correspondence. In:
  Proceedings of the IEEE International Conference on Computer Vision. pp.
  5659--5667 (2017)

\bibitem{FARM}
Marin, R., Melzi, S., Rodolà, E., Castellani, U.: Farm: Functional automatic
  registration method for 3d human bodies. Computer Graphics Forum
  \textbf{39}(1),  160--173 (2020)

\bibitem{GCNN}
Masci, J., Boscaini, D., Bronstein, M.M., Vandergheynst, P.: Geodesic
  convolutional neural networks on riemannian manifolds. In: Proceedings of the
  2015 IEEE International Conference on Computer Vision Workshop (ICCV). p.
  832–840. ICCV’15, IEEE Computer Society, USA (2015)

\bibitem{CMH}
Melzi, S., Marin, R., Musoni, P., Bardon, F., Tarini, M., Castellani, U.:
  Intrinsic/extrinsic embedding for functional remeshing of 3d shapes.
  Computers \& Graphics  \textbf{88},  1 -- 12 (2020)

\bibitem{melzi2019zoomout}
Melzi, S., Ren, J., Rodol{\`a}, E., Sharma, A., Wonka, P., Ovsjanikov, M.:
  Zoomout: Spectral upsampling for efficient shape correspondence. ACM
  Transactions on Graphics (TOG)  \textbf{38}(6), ~155 (2019)

\bibitem{LMH}
Melzi, S., Rodol{\`a}, E., Castellani, U., Bronstein, M.: Localized manifold
  harmonics for spectral shape analysis. Computer Graphics Forum
  \textbf{37}(6),  20--34 (2018)

\bibitem{nogneng18}
Nogneng, D., Melzi, S., Rodol{\`a}, E., Castellani, U., Bronstein, M.,
  Ovsjanikov, M.: Improved functional mappings via product preservation.
  Computer Graphics Forum  \textbf{37}(2),  179--190 (2018)

\bibitem{nogneng2017informative}
Nogneng, D., Ovsjanikov, M.: Informative descriptor preservation via
  commutativity for shape matching. In: Computer Graphics Forum. vol.~36, pp.
  259--267. Wiley Online Library (2017)

\bibitem{ovsjanikov2012functional}
Ovsjanikov, M., Ben-Chen, M., Solomon, J., Butscher, A., Guibas, L.: Functional
  maps: a flexible representation of maps between shapes. ACM Transactions on
  Graphics (TOG)  \textbf{31}(4),  30:1--30:11 (2012)

\bibitem{ovsjanikov2017computing}
Ovsjanikov, M., Corman, E., Bronstein, M., Rodol{\`a}, E., Ben-Chen, M.,
  Guibas, L., Chazal, F., Bronstein, A.: Computing and processing
  correspondences with functional maps. In: SIGGRAPH 2017 Courses (2017)

\bibitem{pinkall1993computing}
Pinkall, U., Polthier, K.: Computing discrete minimal surfaces and their
  conjugates. Experimental mathematics  \textbf{2}(1),  15--36 (1993)

\bibitem{pokrass2013sparse}
Pokrass, J., Bronstein, A.M., Bronstein, M.M., Sprechmann, P., Sapiro, G.:
  Sparse modeling of intrinsic correspondences. In: Computer Graphics Forum.
  vol.~32, pp. 459--468. Wiley Online Library (2013)

\bibitem{poulenard2018multi}
Poulenard, A., Ovsjanikov, M.: Multi-directional geodesic neural networks via
  equivariant convolution. ACM Transactions on Graphics (TOG)  \textbf{37}(6),
  1--14 (2018)

\bibitem{qi2017pointnet}
Qi, C.R., Su, H., Mo, K., Guibas, L.J.: Pointnet: Deep learning on point sets
  for 3d classification and segmentation. In: Proceedings of the IEEE
  Conference on Computer Vision and Pattern Recognition. pp. 652--660 (2017)

\bibitem{Ponitnetplus}
Qi, C.R., Yi, L., Su, H., Guibas, L.J.: Pointnet++: Deep hierarchical feature
  learning on point sets in a metric space. In: Proceedings of the 31st
  International Conference on Neural Information Processing Systems. p.
  5105–5114. NIPS’17, Curran Associates Inc., Red Hook, NY, USA (2017)

\bibitem{ren2018continuous}
Ren, J., Poulenard, A., Wonka, P., Ovsjanikov, M.: Continuous and
  orientation-preserving correspondences via functional maps. ACM Transactions
  on Graphics (TOG)  \textbf{37}(6),  1--16 (2018)

\bibitem{rodola2017partial}
Rodol{\`a}, E., Cosmo, L., Bronstein, M.M., Torsello, A., Cremers, D.: Partial
  functional correspondence. In: Computer Graphics Forum. vol.~36, pp.
  222--236. Wiley Online Library (2017)

\bibitem{roufosse2019unsupervised}
Roufosse, J.M., Sharma, A., Ovsjanikov, M.: Unsupervised deep learning for
  structured shape matching. In: Proceedings of the IEEE International
  Conference on Computer Vision. pp. 1617--1627 (2019)

\bibitem{rustamov2007laplace}
Rustamov, R.M.: Laplace-beltrami eigenfunctions for deformation invariant shape
  representation. In: Proceedings of the fifth Eurographics symposium on
  Geometry processing. pp. 225--233. Eurographics Association (2007)

\bibitem{survey1}
Sahillio\v{g}lu, Y.: Recent advances in shape correspondence. The Visual
  Computer pp. 1--17 (09 2019)

\bibitem{sun2009concise}
Sun, J., Ovsjanikov, M., Guibas, L.: A concise and provably informative
  multi-scale signature based on heat diffusion. Computer Graphics Forum
  \textbf{28}(5),  1383--1392 (2009)

\bibitem{SANCHEZ2020}
Sánchez-Belenguer, C., Ceriani, S., Taddei, P., Wolfart, E., Sequeira, V.:
  Global matching of point clouds for scan registration and loop detection.
  Robotics and Autonomous Systems  \textbf{123},  103324 (2020)

\bibitem{taubin}
Taubin, G.: A signal processing approach to fair surface design. In:
  Proceedings of the 22nd Annual Conference on Computer Graphics and
  Interactive Techniques. p. 351–358. SIGGRAPH, Association for Computing
  Machinery, New York, NY, USA (1995)

\bibitem{thewlis2019unsupervised}
Thewlis, J., Albanie, S., Bilen, H., Vedaldi, A.: Unsupervised learning of
  landmarks by descriptor vector exchange. In: Proceedings of the IEEE
  International Conference on Computer Vision. pp. 6361--6371 (2019)

\bibitem{Thomas_2019_ICCV}
Thomas, H., Qi, C.R., Deschaud, J.E., Marcotegui, B., Goulette, F., Guibas,
  L.J.: Kpconv: Flexible and deformable convolution for point clouds. In: The
  IEEE International Conference on Computer Vision (ICCV) (October 2019)

\bibitem{SHOT}
Tombari, F., Salti, S., Di~Stefano, L.: Unique signatures of histograms for
  local surface description. In: Proc. ECCV. pp. 356--369. Springer (2010)

\bibitem{varol17_surreal}
Varol, G., Romero, J., Martin, X., Mahmood, N., Black, M.J., Laptev, I.,
  Schmid, C.: Learning from synthetic humans. In: CVPR (2017)

\bibitem{wang2019functional}
Wang, F.D., Xue, N., Zhang, Y., Xia, G.S., Pelillo, M.: A functional
  representation for graph matching. IEEE transactions on pattern analysis and
  machine intelligence  (2019)

\bibitem{Wang2020MGCN}
Wang, Y., Ren, J., Yan, D.M., Guo, J., Zhang, X., Wonka, P.: Mgcn: Descriptor
  learning using multiscale gcns. ACM Trans. Graph.  \textbf{39}(4) (2020).
  \doi{10.1145/3386569.3392443}

\bibitem{wang2019prnet}
Wang, Y., Solomon, J.M.: Prnet: Self-supervised learning for partial-to-partial
  registration. In: Advances in Neural Information Processing Systems. pp.
  8812--8824 (2019)

\bibitem{wei2016dense}
Wei, L., Huang, Q., Ceylan, D., Vouga, E., Li, H.: Dense human body
  correspondences using convolutional networks. In: Proceedings of the IEEE
  Conference on Computer Vision and Pattern Recognition. pp. 1544--1553 (2016)

\bibitem{Zhou2019SiamesePointNet}
Zhou, J., Wang, M.J., Mao, W.D., Gong, M.L., Liu, X.P.: Siamesepointnet: A
  siamese point network architecture for learning 3d shape descriptor. Computer
  Graphics Forum  \textbf{39}(1),  309--321 (2020)

\end{thebibliography}
\clearpage

\appendix
\renewcommand*{\thesection}{\arabic{section}}

\null
  \vbox{%
    \hsize\textwidth
    \linewidth\hsize
    \vskip 0.1in
  \hrule height 4px
  \vskip 0.25in
  \vskip -\parskip%
    \centering
    {\LARGE\bf Supplementary Materials\par}
  \vskip 0.29in
  \vskip -\parskip
  \hrule height 1px
  \vskip 0.09in%
    \vskip 0.2in
  }
  
\section{Adjoint operator definition and properties}
\label{sec:proof_delta_adjoint}

In this section, we provide a concise description of the adjoint operator and its relation to the transfer of Dirac delta functions and functional maps. Note that the adjoint operator of functional maps has been considered, e.g., in \cite{huang2017adjoint} although its role in delta function transfer was not explicitly addressed in that work.

\subsection{Formal definition of the Adjoint operator}
Suppose we have a pointwise map $T_{\X\Y}: \X \rightarrow \Y$ between two smooth surfaces $\X,\Y$. Then we will denote $T^{\mathcal{F}}_{\Y\X}$ the functional correspondence defined by the pull-back: $T^{\mathcal{F}}_{\Y\X}: f \rightarrow f \circ T_{\X\Y}$, where $f: \Y \rightarrow \mathbb{R}$ and $f \circ T_{\X\Y} : \X \rightarrow \mathbb{R}$ such that $f\circ T_{\X\Y} (x) = f(T_{\X\Y}(x))$ for any $x \in \X$.

The \emph{adjoint functional map operator} $A_{\X\Y}$ is defined implicitly through the following equation:
\begin{align}
\label{eq:distribution}
<A_{\X\Y} g, f>_{\Y} = <g , T^{\mathcal{F}}_{\Y\X} f>_{\X} ~~ \forall ~~ f: \Y\rightarrow \mathbb{R}, ~ g : \X\rightarrow \mathbb{R}.
\end{align}
Here we denote with $<,>_{\X}$ and  $<,>_{\Y}$ the $L^2$ inner product for functions respectively on shape $\X$ and $\Y$. The adjoint always exists and is unique by the Riesz representation theorem (see also Theorem 3.1 in \cite{huang2017adjoint}).

\subsection{Adjoint operator and delta functions}
As mentioned in the main manuscript, the adjoint can be used to map \emph{distributions} (or generalized functions), which is particularly important for mapping points represented as Dirac delta functions.

Recall that $\forall y \in \Y$, a Dirac delta function $\delta_y$ is a distribution such that, by definition, for any function $f$ we have $<\delta_y, f>_{\Y} = f(y)$. 

\begin{theorem}
If $A_{\X\Y}$ is the adjoint operator associated with a point-to-point mapping $T_{\X\Y}$ as in Eq. \eqref{eq:distribution}, then $A_{\X\Y} \delta_x = \delta_{T_{\X\Y}(x)}$.
\end{theorem}
\begin{proof}
Using Eq. \eqref{eq:distribution} we get:
\begin{align}
<A_{\X\Y} \delta_x, f>_{\Y} &= <\delta_x , T^{\mathcal{F}}_{\Y\X} f>_{\X} = <\delta_x , f \circ T_{\X\Y} >_{\X}
\\ &=  f(T_{\X\Y}(x)).
\end{align}
Therefore, $A_{\X\Y} \delta_x$ equals some distribution $d$ such that $<d,f>_{\Y} = f(T_{\X\Y}(x))$  for any function $f : \Y \rightarrow \mathbb{R}$. By uniqueness of distributions this means that: $A_{\X\Y} \delta_x = \delta_{T_{\X\Y}(x)}$.
\end{proof}

In other words, the previous derivation proves that, unlike a functional map, \textbf{the functional map adjoint always maps delta functions to delta functions}. 

\subsection{Relation between the functional maps and the adjoint operator in the discrete setting}

Here we assume that the two shapes are represented in the discrete setting, with two embeddings $\Phi_{\X}, \Phi_{\Y}$, and a  pointwise map $\Pi_{\X\Y}$, using the notation from the main paper. Our goal is to establish the relationship between the functional map matrix and the linear operator, which aligns the two embeddings.

Given two embeddings $\Phi_{\X}, \Phi_{\Y}$ and a pointwise map $\Pi_{\X\Y}$ we would like to find a linear transformation $A_{\X\Y}$ such that:
\begin{align}
    A_{\X\Y} \Phi_{\X}^T &= (\Pi_{\X\Y}\Phi_{\Y})^T, \text{ or equivalently} \\
    \Phi_{\X} A_{\X\Y}^T &= \Pi_{\X\Y}\Phi_{\Y}
\end{align}

Formulating this as a least squares problem we get:
\begin{align}
    \min_{A} \| \Phi_{\X} A_{\X\Y}^T - \Pi_{\X\Y}\Phi_{\Y} \|_2,
\end{align}
from which the solution is given by:
\begin{align}
    A = \left(\Phi_{\X}^{\dagger} \Pi_{\X\Y}\Phi_{\Y}\right)^T
\end{align}

Recall that a functional map induced by $\Pi_{\X\Y}$ is defined as $C_{\Y\X} = \Phi_{\X}^{\dagger} \Pi_{\X\Y} \Phi_{\Y}$. Therefore, we can write: $A_{\X\Y} = C_{\Y\X}^T$. In other words, in the discrete setting the adjoint is nothing but the transpose of the functional map in the opposite direction.

\subsection{Probe function constraints}
Below we derive the relation between the probe function constraints for functional maps and those for the adjoint operator used in our approach, as described in Section 4.2 of the main paper. Here we derive the formula used in the main manuscript directly below Eq. (4).

In the main paper (Eq. (1) of the main manuscript) we wrote the following basic optimization problem for estimating functional maps:
\begin{equation}
    C_{\X\Y} = \argmin_{C \in \mathbb{R}^{k \times k}} \| C \Phi_{\X}^{\dagger} G_{\X} - \Phi_{\Y}^{\dagger} G_{\Y} \|_{2} + E_{reg}(C).
    \label{eq:fmaps1}
\end{equation}
Inverting the role of $\X$ and $\Y$ and removing the regularization we obtain:
\begin{equation}
    C_{\Y\X} = \argmin_{C \in \mathbb{R}^{k \times k}} \| C \Phi_{\Y}^{\dagger} G_{\Y} - \Phi_{\X}^{\dagger} G_{\X} \|_{2}.
    \label{eq:fmaps2}
\end{equation}
This implies that the optimal $C_{\Y\X}$ can be found as the solution of $C_{\Y\X} \Phi_{\Y}^{\dagger} G_{\Y} = \Phi_{\X}^{\dagger} G_{\X}$. This is equivalent to $(\Phi_{\Y}^{\dagger} G_{\Y})^{T} C_{\Y\X}^T = (\Phi_{\X}^{\dagger} G_{\X})^{T}$ that can be solved as a least squares problem:
\begin{equation}
C_{\Y\X}^T = \left((\Phi_{\Y}^{\dagger} G_{\Y})^{T}\right)^{\dagger}(\Phi_{\X}^{\dagger} G_{\X})^{T}.
\end{equation}
From the equation $A_{\X\Y} = C_{\Y\X}^T$ we can conclude that:
\begin{equation}
    A_{\X\Y} = \left((\Phi_{\Y}^{\dagger} G_{\Y})^{T}\right)^{\dagger}(\Phi_{\X}^{\dagger} G_{\X})^{T}.
\end{equation}
This is precisely the equation used in the main manuscript directly below Eq. (4).


This provides an explicit connection between the functional map and the linear transformation that we are optimizing for.

To summarize, one advantage of the adjoint is that it can be used to map \emph{distributions} and not just functions. In particular, unlike a functional map, the functional map adjoint always maps delta functions to delta functions. At the same time, similarly to functional maps, it also allows estimation via probe functions and a solution of a linear system. For this reason, despite the strong relation with functional maps, the adjoint is better suited for estimating the correspondence.

\section{Implementation details}
\label{sec:implementations}

\subsection{Training set details}

\paragraph{Pre-processing}
In our analysis, we consider shapes that are centered in the origin of $\mathbb{R}^{3}$ and scaled with uniform unit area. These requirements are not strong and every input shape can be easily pre-processed to satisfy these properties.

\subsection{The softmax operation}
We compute the \textit{soft} permutation matrix as follows:
\begin{equation}
    \begin{aligned}
        (S_{\Y\X})_{ij} = \frac{e^{- \|\widehat{\Phi}_{\X}^{i} - \Phi_{\Y}^{j}\|_2 }}{\sum_{k=1}^{n_{\Y}} e^{- \|\widehat{\Phi}_{\X}^{i} - \Phi_{\Y}^{k}\|_2 }} 
    \end{aligned}
    \label{eq:softmax}
\end{equation}
where $n_{\Y}$ is the number of points of $\Y$ and $\Phi_{\X}$ and $\Phi_{\Y}$ are learned embeddings.

\subsection{Architecture description}

We describe our complete pipeline in figure \ref{fig:pipeline}. The invariant embedding network and probe function network are built with the semantic segmentation architecture of PointNet \cite{qi2017pointnet}. 

\begin{figure}
  \centering
  \includegraphics[width=1\textwidth]{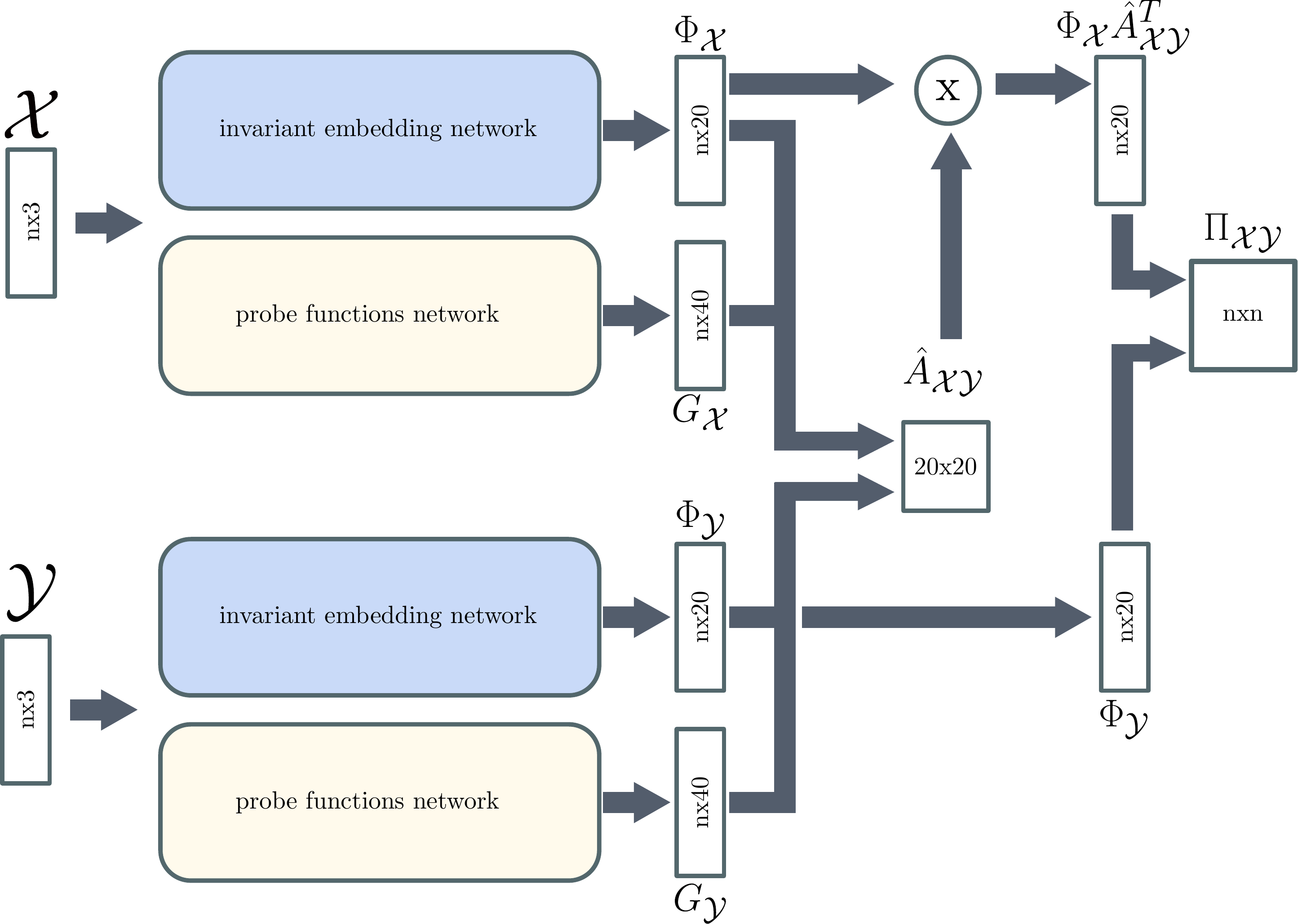}
  \vspace{-1mm}
  \caption{Method pipeline \vspace{-2mm}}
  \label{fig:pipeline}
\end{figure}
\begin{figure}[!t]
  \centering
  \includegraphics[width=0.8\textwidth]{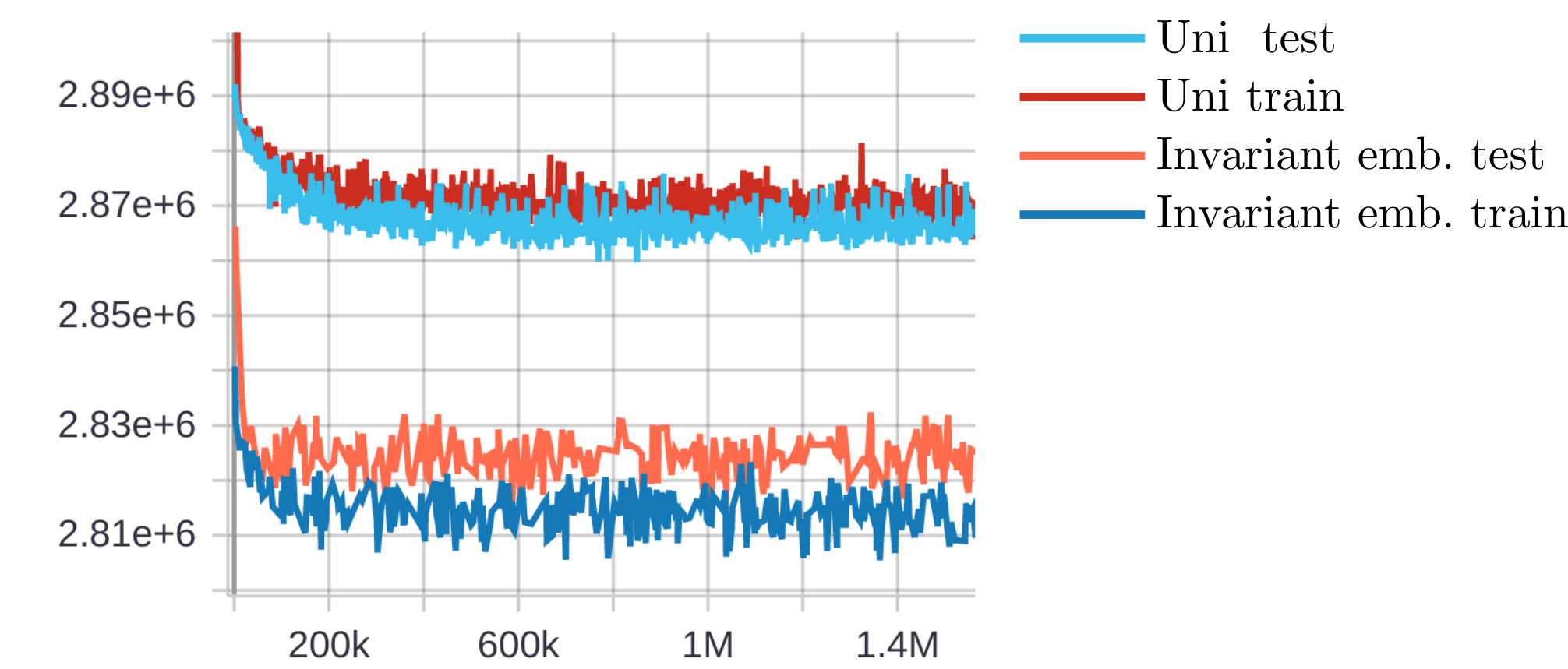}
  \vspace{-1mm}
  \caption{Comparison of the linearly-invariant embedding model and the universal embedding model training curves.  \vspace{-2mm}}
  \label{fig:training}
\end{figure}
\subsection{Relation to the universal embedding network}
In figure \ref{fig:training}, we show the training curves of the universal embedding model and the linearly-invariant embedding model. We observe that learning the linearly-invariant embedding leads to faster learning and a lower loss. It confirms that the linearly-invariant embedding simplifies and is more adapted to the correspondence task.


\section{Additional results and visualizations}
\label{sec:addresults}
We show some other results on noisy point clouds in Figure \ref{fig:results}.
We provide an example of our networks output over a couple of FAUST shapes at high-resolution ( 160K vertices). We show the $20$ basis in Figures \ref{fig:basis1} and \ref{fig:basis2}, and the $40$ descriptors in Figures \ref{fig:desc1}, \ref{fig:desc2} and \ref{fig:desc3}. Finally, we show further statues examples in Figure \ref{fig:smatstatues}. We would remark that our method consider only the points coordinates; we color the surface for better visualization.

\clearpage
\begin{figure}[!t]
\centering

  \setlength{\tabcolsep}{0pt}
  \begin{tabular}{c c c c c}
 
\begin{minipage}{0.15\linewidth}
 \begin{overpic}[trim=0cm 0cm 0cm 0cm,clip,width=0.93\linewidth]{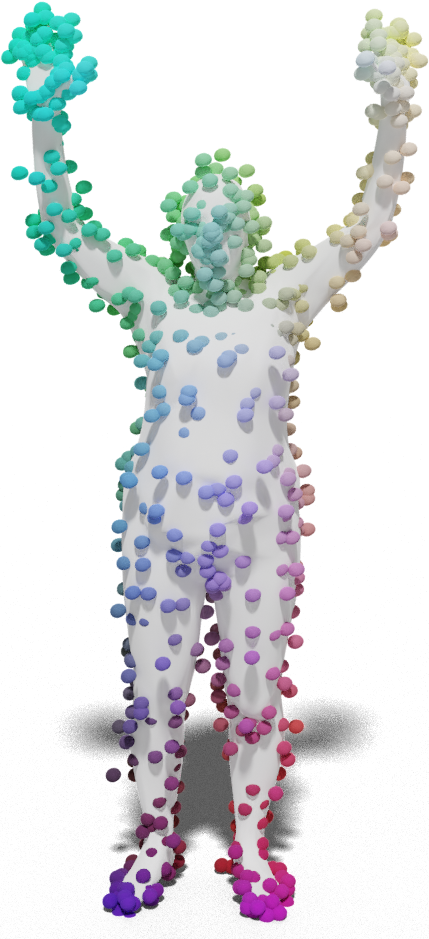}
  \put(20,105){\footnotesize{$X$}}
   \end{overpic}
\end{minipage}
&
\begin{minipage}{0.22\linewidth}
 \begin{overpic}[trim=0cm 0cm 0cm 0cm,clip,width=0.93\linewidth]{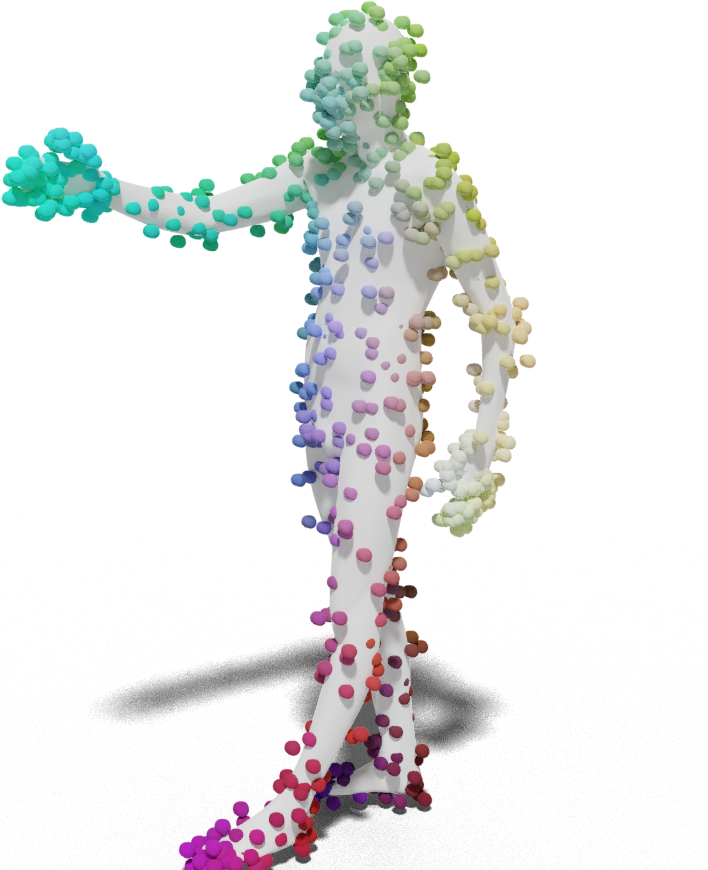}
  \put(30,105){\footnotesize{GT}}
   \end{overpic}
\end{minipage}
&

\begin{minipage}{0.17\linewidth}
 \begin{overpic}[trim=0cm 0cm 0cm 0cm,clip,width=0.93\linewidth]{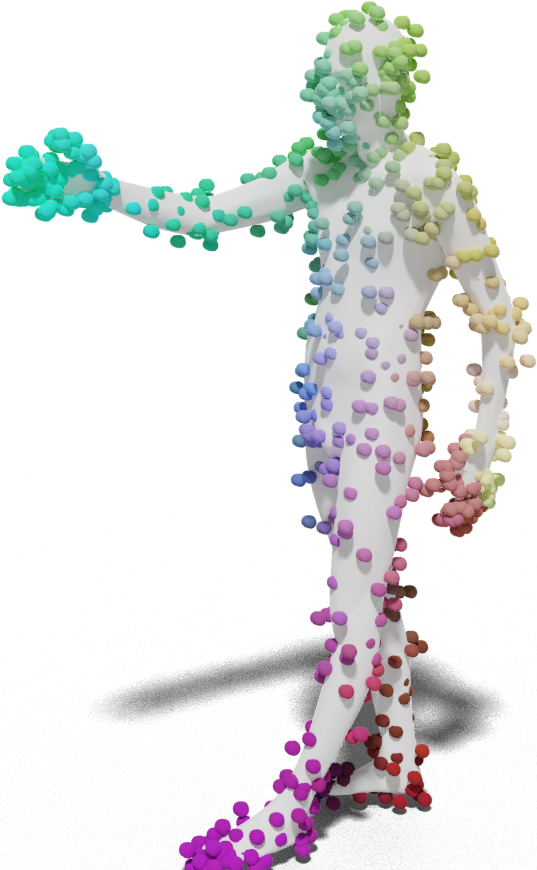}
  \put(25,105){\footnotesize{Our}}
   \end{overpic}
\end{minipage}

&
\begin{minipage}{0.17\linewidth}
 \begin{overpic}[trim=0cm 0cm 0cm 0cm,clip,width=0.93\linewidth]{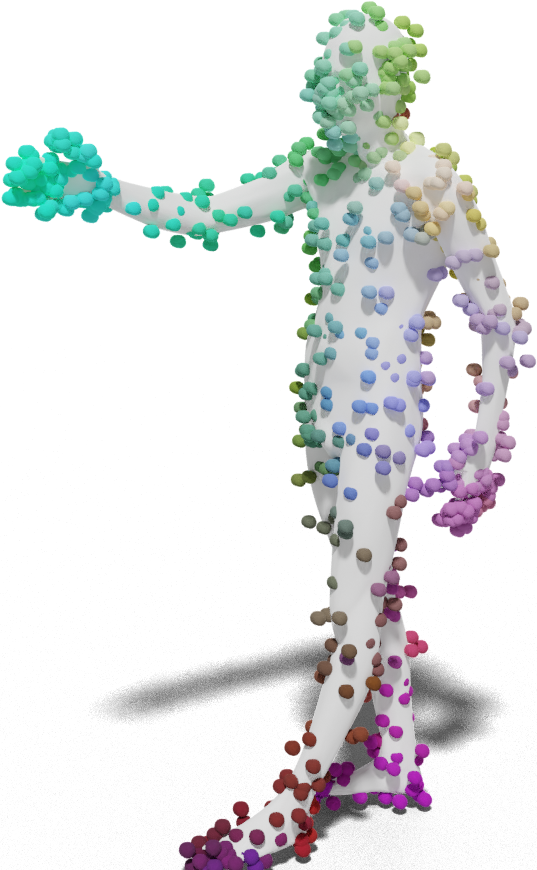}
  \put(25,105){\footnotesize{3DC}}
   \end{overpic}
\end{minipage}

&
\begin{minipage}{0.17\linewidth}
 \begin{overpic}[trim=0cm 0cm 0cm 0cm,clip,width=0.93\linewidth]{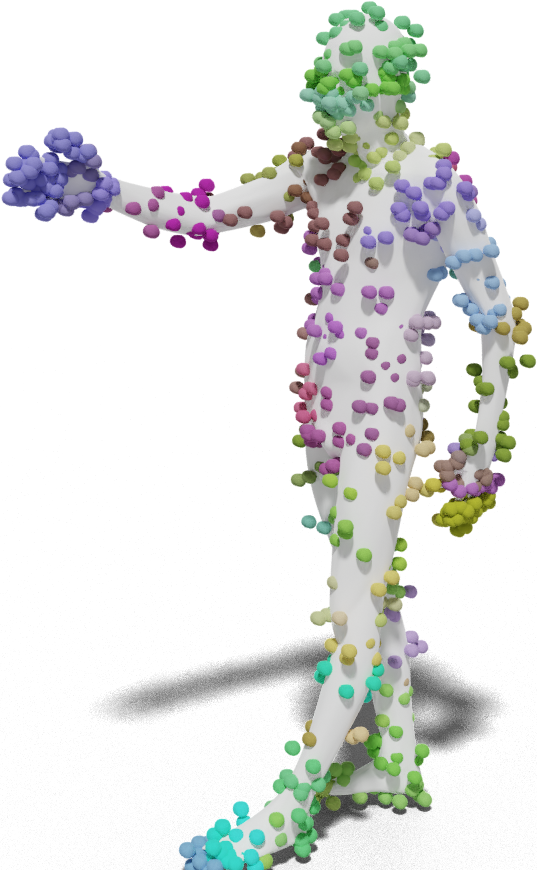}
  \put(25,105){\footnotesize{GFM}}
   \end{overpic}
\end{minipage}
\\

\begin{minipage}{0.28\linewidth}
 \begin{overpic}[trim=0cm 0cm 0cm 0cm,clip,width=0.93\linewidth]{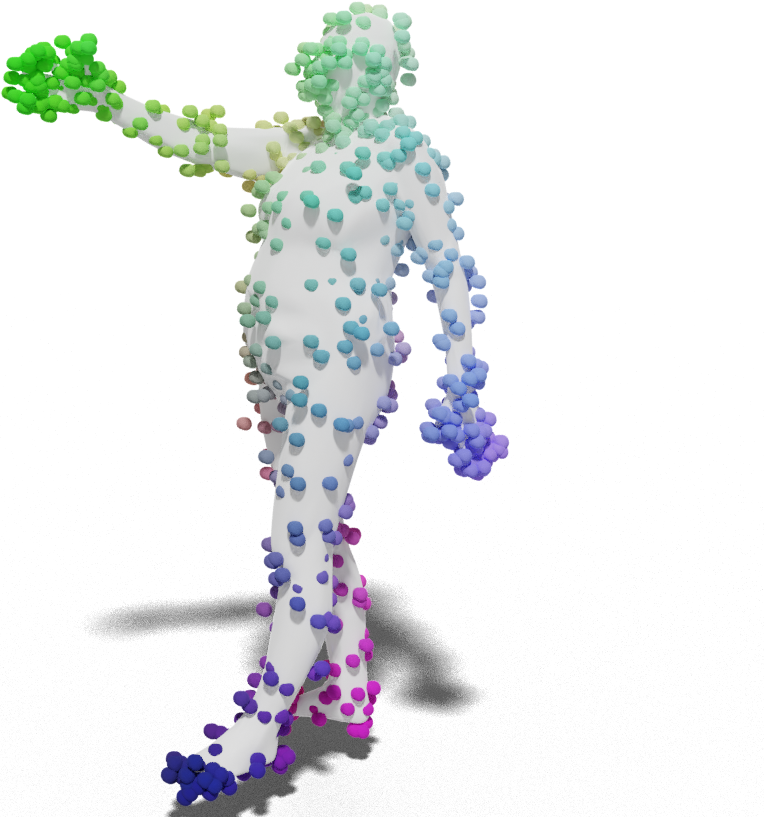}
   \end{overpic}
\end{minipage}
&
\begin{minipage}{0.15\linewidth}
 \begin{overpic}[trim=0cm 0cm 0cm 0cm,clip,width=0.93\linewidth]{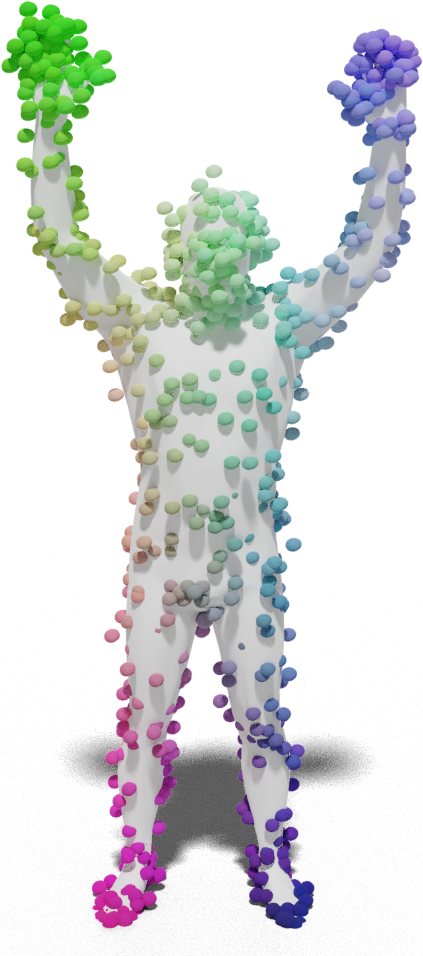}
   \end{overpic}
\end{minipage}
&

\begin{minipage}{0.15\linewidth}
 \begin{overpic}[trim=0cm 0cm 0cm 0cm,clip,width=0.93\linewidth]{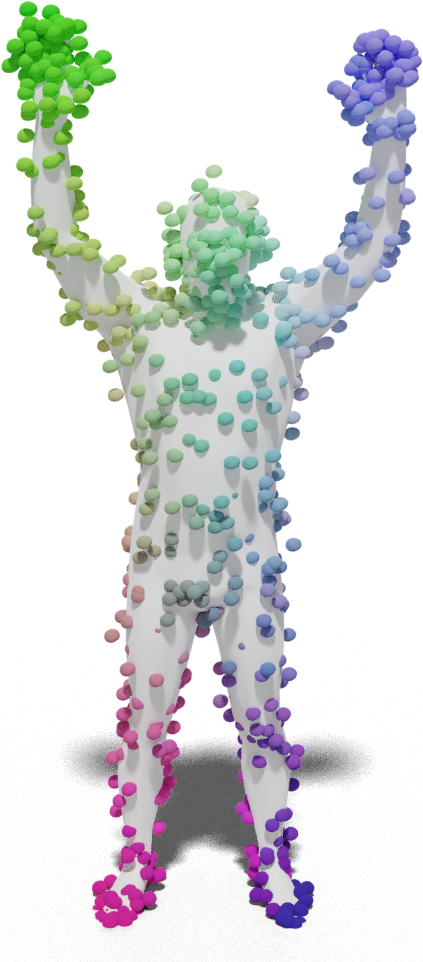}
   \end{overpic}
\end{minipage}

&
\begin{minipage}{0.15\linewidth}
 \begin{overpic}[trim=0cm 0cm 0cm 0cm,clip,width=0.93\linewidth]{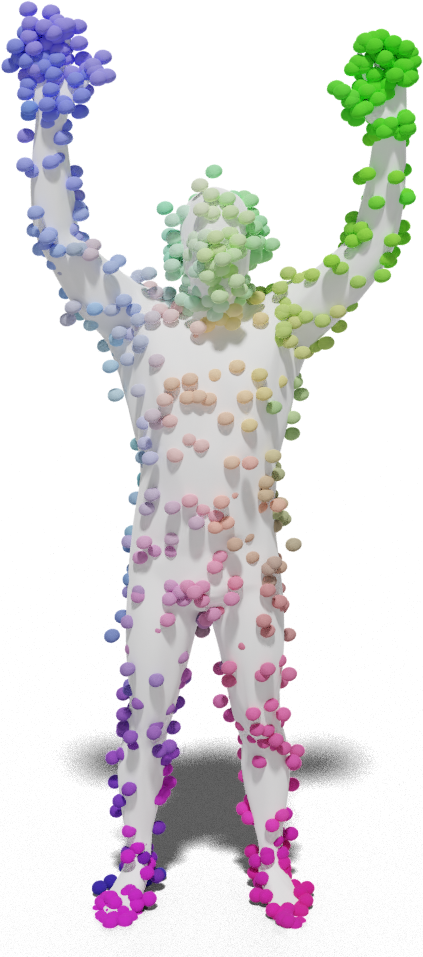}
   \end{overpic}
\end{minipage}

&
\begin{minipage}{0.15\linewidth}
 \begin{overpic}[trim=0cm 0cm 0cm 0cm,clip,width=0.93\linewidth]{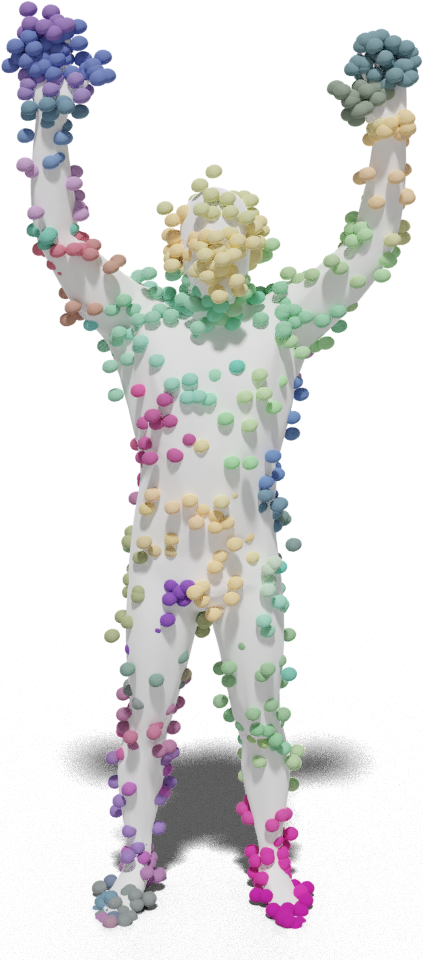}
   \end{overpic}
\end{minipage}

\\

\begin{minipage}{0.15\linewidth}
 \begin{overpic}[trim=0cm 0cm 0cm 0cm,clip,width=0.93\linewidth]{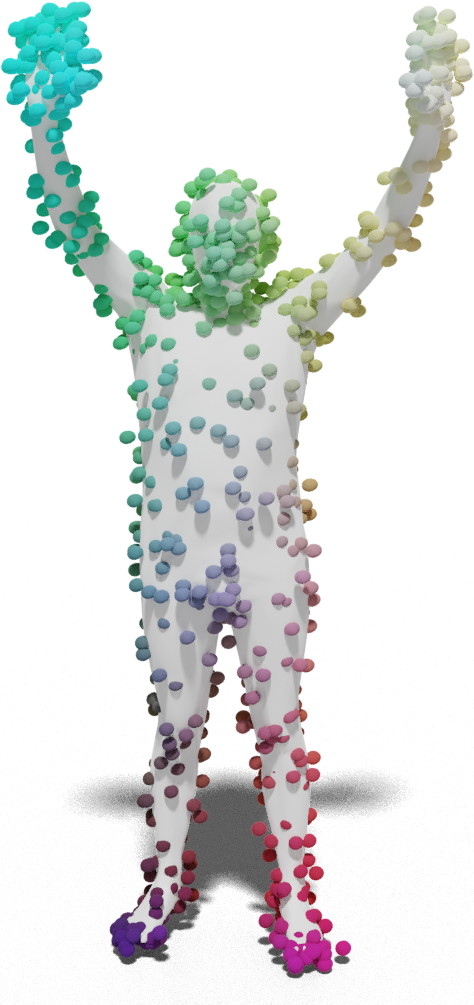}
   \end{overpic}
\end{minipage}
&
\begin{minipage}{0.15\linewidth}
 \begin{overpic}[trim=0cm 0cm 0cm 0cm,clip,width=0.93\linewidth]{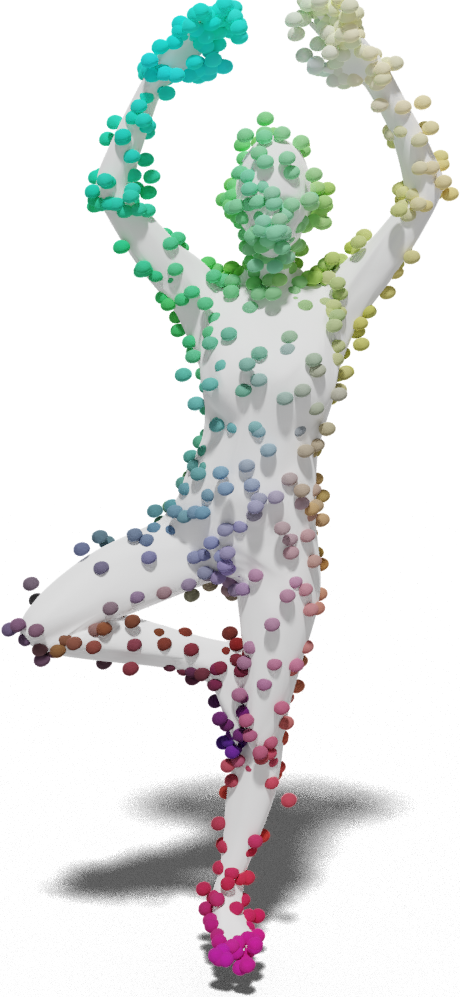}
   \end{overpic}
\end{minipage}
&

\begin{minipage}{0.15\linewidth}
 \begin{overpic}[trim=0cm 0cm 0cm 0cm,clip,width=0.93\linewidth]{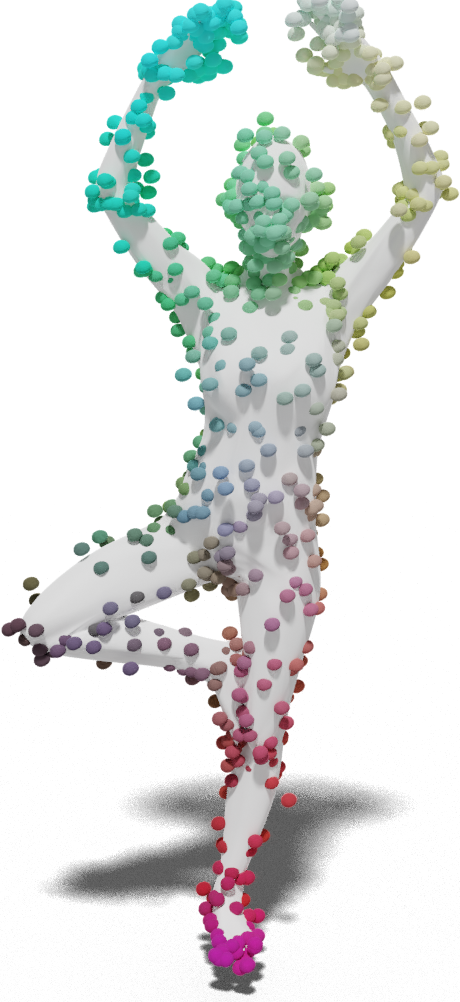}
   \end{overpic}
\end{minipage}

&
\begin{minipage}{0.15\linewidth}
 \begin{overpic}[trim=0cm 0cm 0cm 0cm,clip,width=0.93\linewidth]{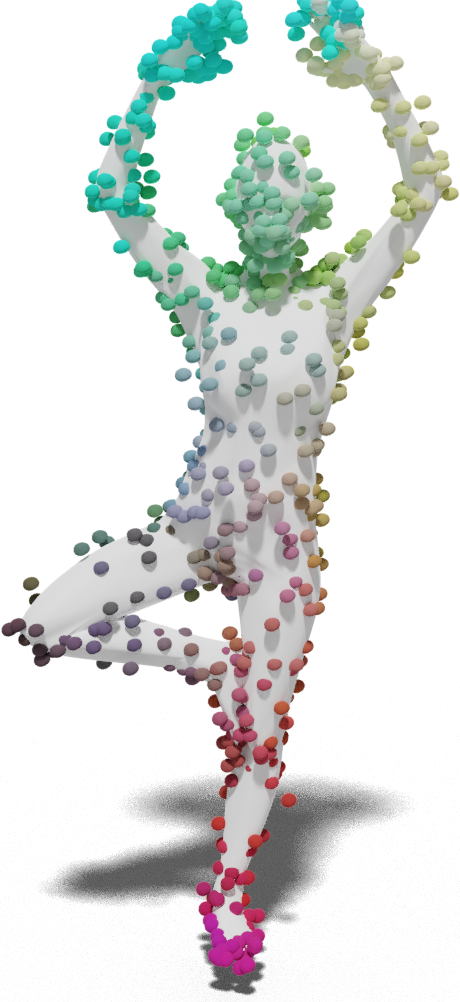}
   \end{overpic}
\end{minipage}

&
\begin{minipage}{0.15\linewidth}
 \begin{overpic}[trim=0cm 0cm 0cm 0cm,clip,width=0.93\linewidth]{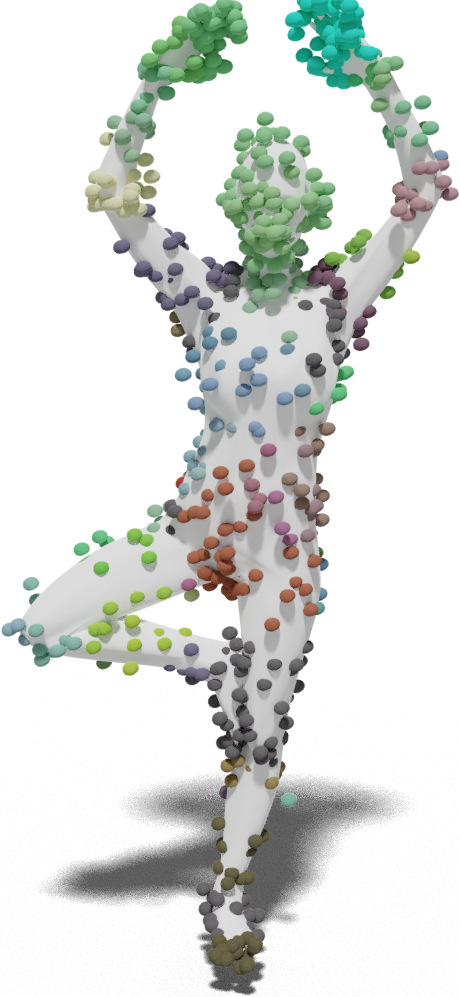}
   \end{overpic}
\end{minipage}

\\

\begin{minipage}{0.15\linewidth}
 \begin{overpic}[trim=0cm 0cm 0cm 0cm,clip,width=0.93\linewidth]{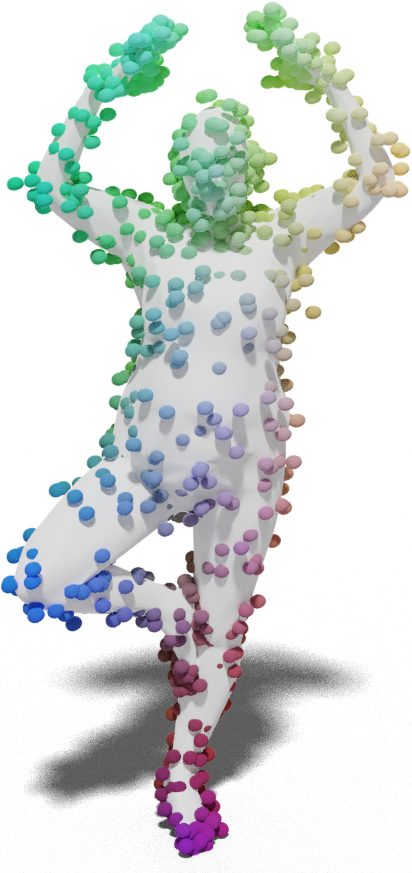}
   \end{overpic}
\end{minipage}
&
\begin{minipage}{0.15\linewidth}
 \begin{overpic}[trim=0cm 0cm 0cm 0cm,clip,width=0.93\linewidth]{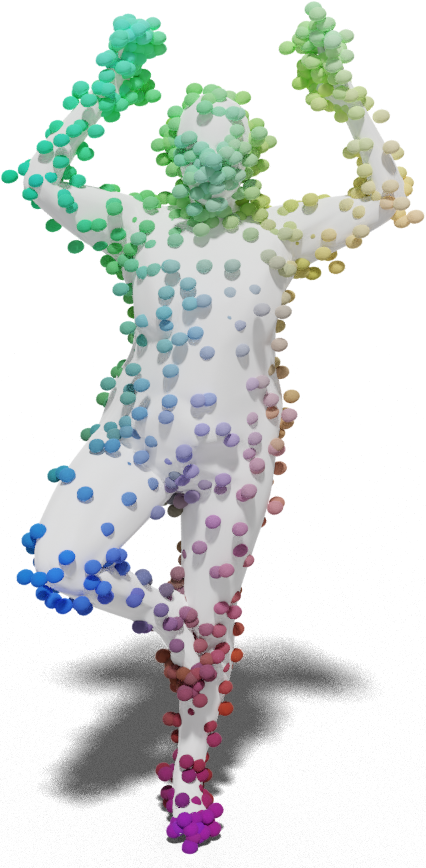}
   \end{overpic}
\end{minipage}
&

\begin{minipage}{0.15\linewidth}
 \begin{overpic}[trim=0cm 0cm 0cm 0cm,clip,width=0.93\linewidth]{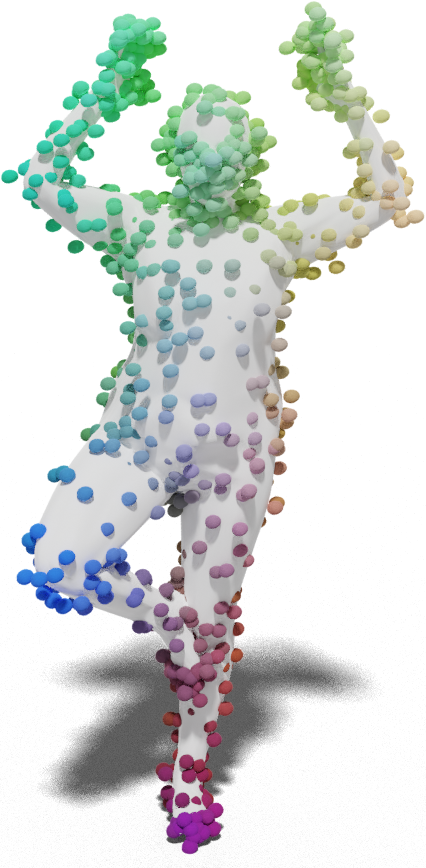}
   \end{overpic}
\end{minipage}

&
\begin{minipage}{0.15\linewidth}
 \begin{overpic}[trim=0cm 0cm 0cm 0cm,clip,width=0.93\linewidth]{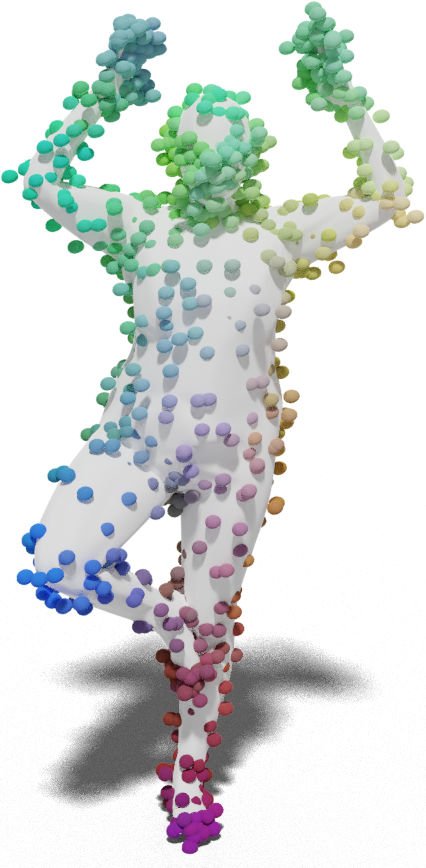}
   \end{overpic}
\end{minipage}

&
\begin{minipage}{0.15\linewidth}
 \begin{overpic}[trim=0cm 0cm 0cm 0cm,clip,width=0.93\linewidth]{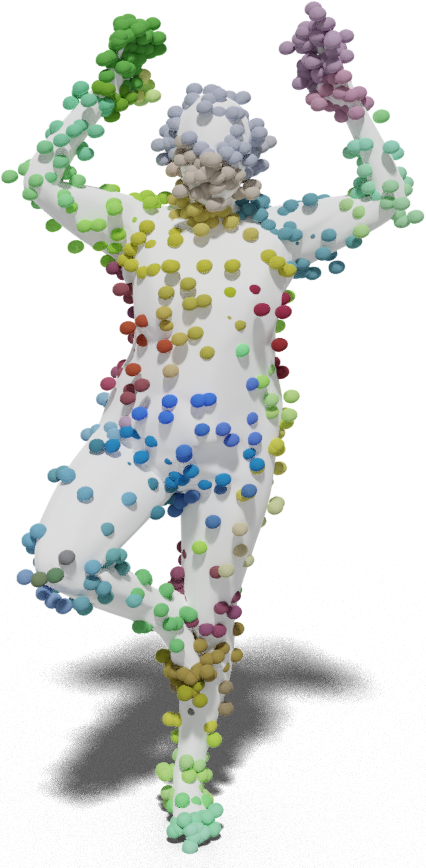}
   \end{overpic}
\end{minipage}

\\

\begin{minipage}{0.15\linewidth}
 \begin{overpic}[trim=0cm 0cm 0cm 0cm,clip,width=0.93\linewidth]{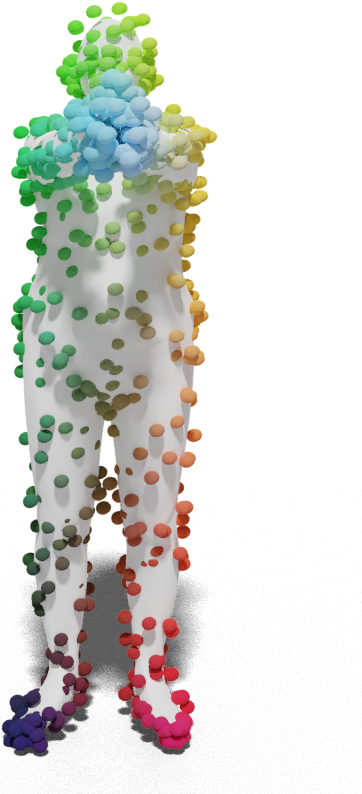}
   \end{overpic}
\end{minipage}
&
\begin{minipage}{0.15\linewidth}
 \begin{overpic}[trim=0cm 0cm 0cm 0cm,clip,width=0.93\linewidth]{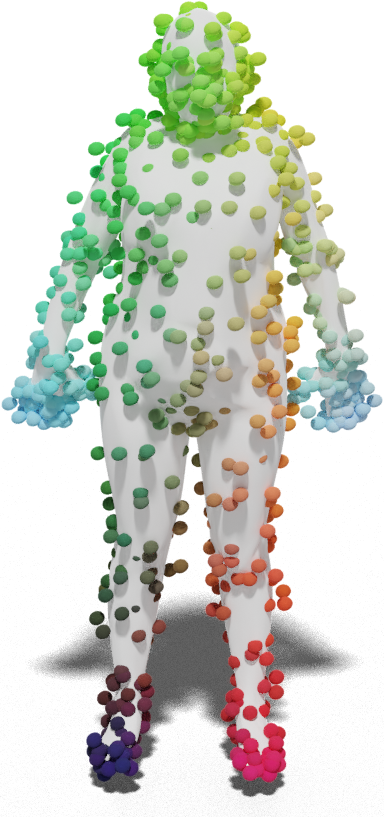}
   \end{overpic}
\end{minipage}
&

\begin{minipage}{0.15\linewidth}
 \begin{overpic}[trim=0cm 0cm 0cm 0cm,clip,width=0.93\linewidth]{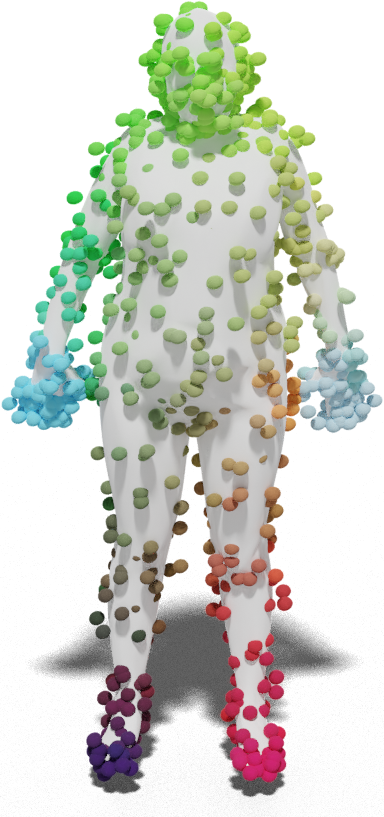}
   \end{overpic}
\end{minipage}

&
\begin{minipage}{0.15\linewidth}
 \begin{overpic}[trim=0cm 0cm 0cm 0cm,clip,width=0.93\linewidth]{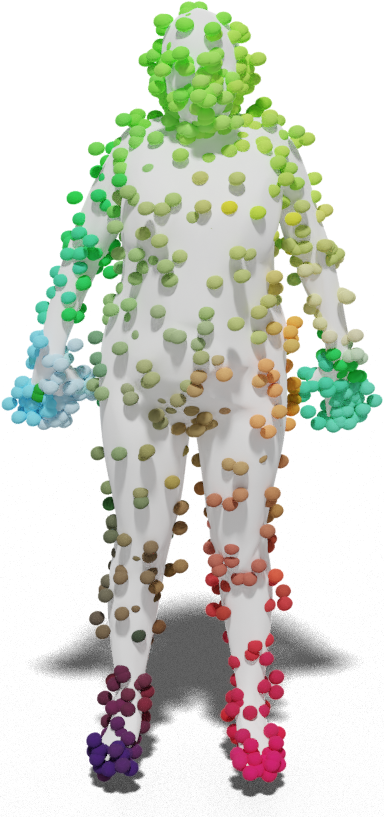}
   \end{overpic}
\end{minipage}

&
\begin{minipage}{0.15\linewidth}
 \begin{overpic}[trim=0cm 0cm 0cm 0cm,clip,width=0.93\linewidth]{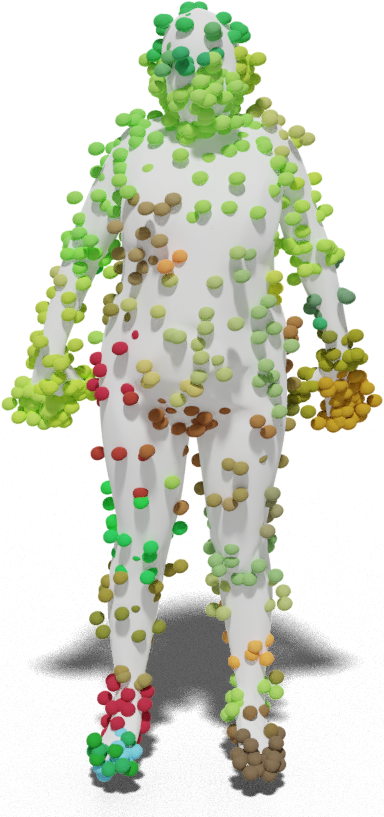}
   \end{overpic}
\end{minipage}

\end{tabular}
\caption{  \label{fig:results} Comparisons on noisy point clouds. 3DC major artifacts are over the hands and in some cases it confuses left and right. GFM suffers from the quality of the point clouds basis estimation.}
\end{figure}

\begin{figure}[!t]
\centering

  \setlength{\tabcolsep}{0pt}
  \begin{tabular}{c c c c c c}
 
\begin{minipage}{0.15\linewidth}
 \begin{overpic}[trim=0cm 0cm 0cm 0cm,clip,width=0.93\linewidth]{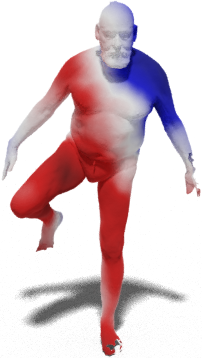}
  \put(20,105){\footnotesize{$\Phi_X$}}
   \end{overpic}
\end{minipage}
&
\begin{minipage}{0.15\linewidth}
 \begin{overpic}[trim=0cm 0cm 0cm 0cm,clip,width=0.93\linewidth]{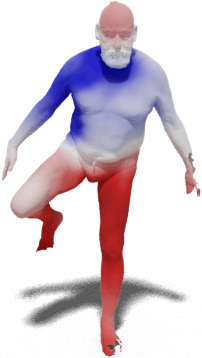}
  \put(20,105){\footnotesize{$\Phi_XA_{XY}$}}
   \end{overpic}
\end{minipage}
&
\begin{minipage}{0.1\linewidth}
 \begin{overpic}[trim=0cm 0cm 0cm 0cm,clip,width=0.93\linewidth]{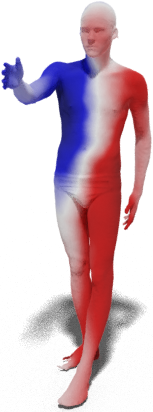}
  \put(18,105){\footnotesize{$\Phi_Y$}}
   \end{overpic}
\end{minipage}

\hspace{1cm}

&
\begin{minipage}{0.15\linewidth}
 \begin{overpic}[trim=0cm 0cm 0cm 0cm,clip,width=0.93\linewidth]{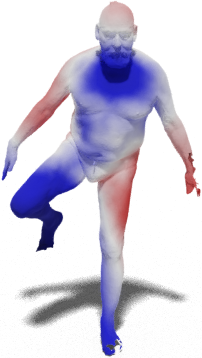}
  \put(20,105){\footnotesize{$\Phi_X$}}
   \end{overpic}
\end{minipage}
&
\begin{minipage}{0.15\linewidth}
 \begin{overpic}[trim=0cm 0cm 0cm 0cm,clip,width=0.93\linewidth]{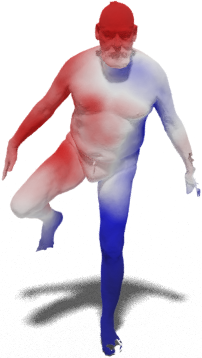}
  \put(20,105){\footnotesize{$\Phi_XA_{XY}$}}
   \end{overpic}
\end{minipage}
&
\begin{minipage}{0.1\linewidth}
 \begin{overpic}[trim=0cm 0cm 0cm 0cm,clip,width=0.93\linewidth]{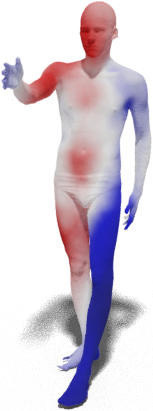}
  \put(18,105){\footnotesize{$\Phi_Y$}}
   \end{overpic}
\end{minipage}

\hspace{0.5cm} 

\\

\begin{minipage}{0.15\linewidth}
 \begin{overpic}[trim=0cm 0cm 0cm 0cm,clip,width=0.93\linewidth]{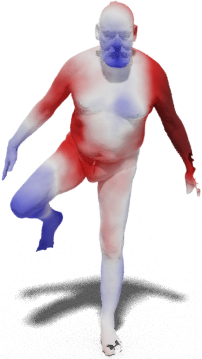}
  
   \end{overpic}
\end{minipage}
&
\begin{minipage}{0.15\linewidth}
 \begin{overpic}[trim=0cm 0cm 0cm 0cm,clip,width=0.93\linewidth]{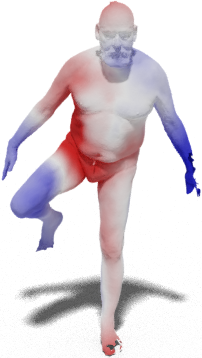}
  
   \end{overpic}
\end{minipage}
&
\begin{minipage}{0.1\linewidth}
 \begin{overpic}[trim=0cm 0cm 0cm 0cm,clip,width=0.93\linewidth]{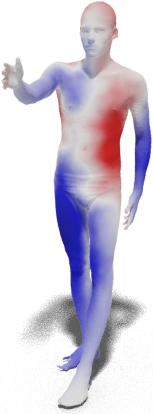}
  
   \end{overpic}
\end{minipage}

\hspace{1cm}

&
\begin{minipage}{0.15\linewidth}
 \begin{overpic}[trim=0cm 0cm 0cm 0cm,clip,width=0.93\linewidth]{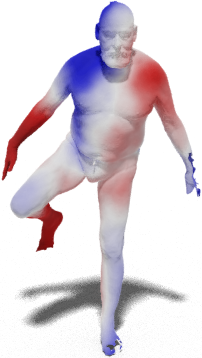}
  
   \end{overpic}
\end{minipage}
&
\begin{minipage}{0.15\linewidth}
 \begin{overpic}[trim=0cm 0cm 0cm 0cm,clip,width=0.93\linewidth]{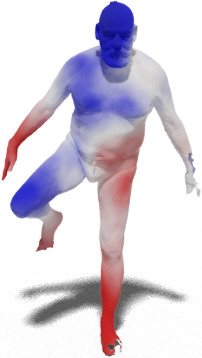}
   \end{overpic}
\end{minipage}
&
\begin{minipage}{0.1\linewidth}
 \begin{overpic}[trim=0cm 0cm 0cm 0cm,clip,width=0.93\linewidth]{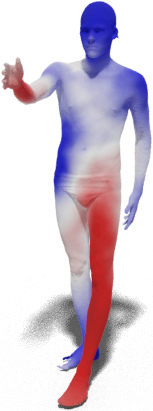}

   \end{overpic}
\end{minipage}

\hspace{0.5cm} 

\\

\begin{minipage}{0.15\linewidth}
 \begin{overpic}[trim=0cm 0cm 0cm 0cm,clip,width=0.93\linewidth]{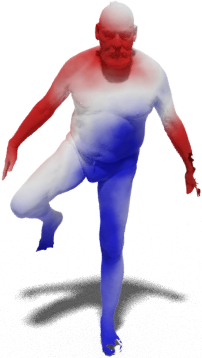}
  
   \end{overpic}
\end{minipage}
&
\begin{minipage}{0.15\linewidth}
 \begin{overpic}[trim=0cm 0cm 0cm 0cm,clip,width=0.93\linewidth]{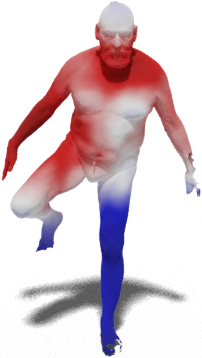}
  
   \end{overpic}
\end{minipage}
&
\begin{minipage}{0.1\linewidth}
 \begin{overpic}[trim=0cm 0cm 0cm 0cm,clip,width=0.93\linewidth]{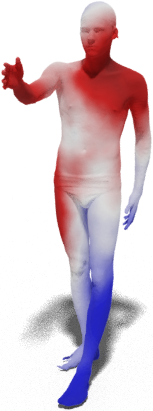}
  
   \end{overpic}
\end{minipage}

\hspace{1cm}

&
\begin{minipage}{0.15\linewidth}
 \begin{overpic}[trim=0cm 0cm 0cm 0cm,clip,width=0.93\linewidth]{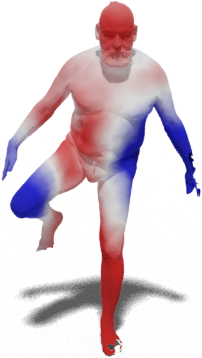}
  
   \end{overpic}
\end{minipage}
&
\begin{minipage}{0.15\linewidth}
 \begin{overpic}[trim=0cm 0cm 0cm 0cm,clip,width=0.93\linewidth]{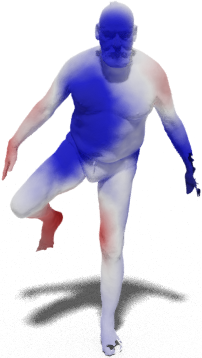}
   \end{overpic}
\end{minipage}
&
\begin{minipage}{0.1\linewidth}
 \begin{overpic}[trim=0cm 0cm 0cm 0cm,clip,width=0.93\linewidth]{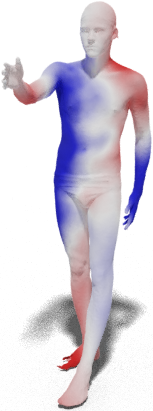}

   \end{overpic}
\end{minipage}
\\

\begin{minipage}{0.15\linewidth}
 \begin{overpic}[trim=0cm 0cm 0cm 0cm,clip,width=0.93\linewidth]{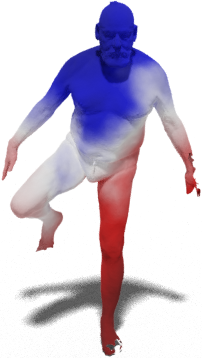}
  
   \end{overpic}
\end{minipage}
&
\begin{minipage}{0.15\linewidth}
 \begin{overpic}[trim=0cm 0cm 0cm 0cm,clip,width=0.93\linewidth]{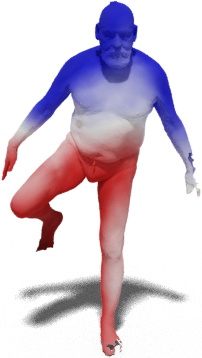}
  
   \end{overpic}
\end{minipage}
&
\begin{minipage}{0.1\linewidth}
 \begin{overpic}[trim=0cm 0cm 0cm 0cm,clip,width=0.93\linewidth]{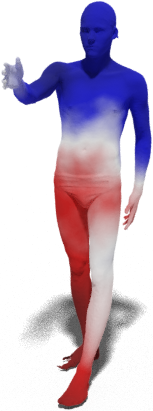}
  
   \end{overpic}
\end{minipage}

\hspace{1cm}

&
\begin{minipage}{0.15\linewidth}
 \begin{overpic}[trim=0cm 0cm 0cm 0cm,clip,width=0.93\linewidth]{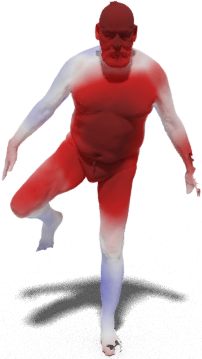}
  
   \end{overpic}
\end{minipage}
&
\begin{minipage}{0.15\linewidth}
 \begin{overpic}[trim=0cm 0cm 0cm 0cm,clip,width=0.93\linewidth]{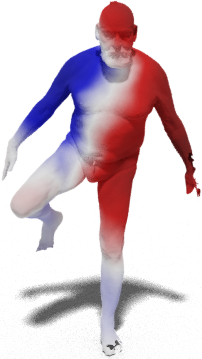}
  
   \end{overpic}
\end{minipage}
&
\begin{minipage}{0.1\linewidth}
 \begin{overpic}[trim=0cm 0cm 0cm 0cm,clip,width=0.93\linewidth]{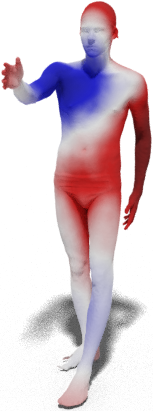}
  
   \end{overpic}
\end{minipage}

\hspace{0.5cm} 

\\

\begin{minipage}{0.15\linewidth}
 \begin{overpic}[trim=0cm 0cm 0cm 0cm,clip,width=0.93\linewidth]{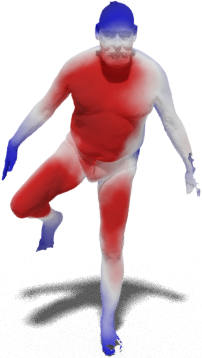}
  
   \end{overpic}
\end{minipage}
&
\begin{minipage}{0.15\linewidth}
 \begin{overpic}[trim=0cm 0cm 0cm 0cm,clip,width=0.93\linewidth]{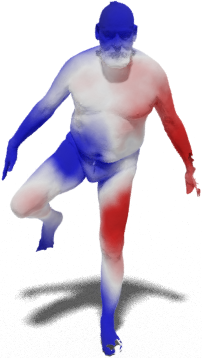}
  
   \end{overpic}
\end{minipage}
&
\begin{minipage}{0.1\linewidth}
 \begin{overpic}[trim=0cm 0cm 0cm 0cm,clip,width=0.93\linewidth]{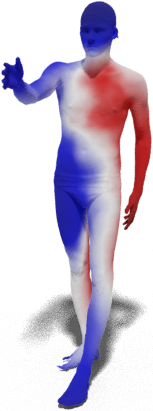}
  
   \end{overpic}
\end{minipage}

\hspace{1cm}

&
\begin{minipage}{0.15\linewidth}
 \begin{overpic}[trim=0cm 0cm 0cm 0cm,clip,width=0.93\linewidth]{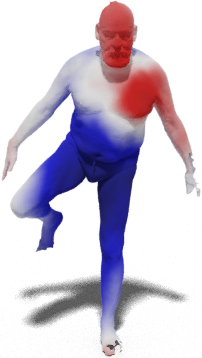}
  
   \end{overpic}
\end{minipage}
&
\begin{minipage}{0.15\linewidth}
 \begin{overpic}[trim=0cm 0cm 0cm 0cm,clip,width=0.93\linewidth]{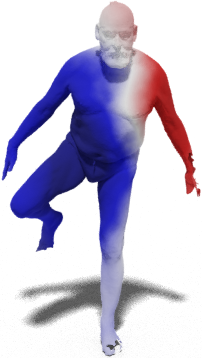}
   \end{overpic}
\end{minipage}
&
\begin{minipage}{0.1\linewidth}
 \begin{overpic}[trim=0cm 0cm 0cm 0cm,clip,width=0.93\linewidth]{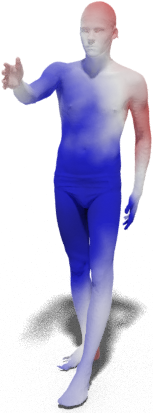}

   \end{overpic}
\end{minipage}

\hspace{0.5cm} 

\end{tabular}
\caption{  \label{fig:basis1} Basis from $1$ to $10$}
\end{figure}

\begin{figure}[!t]
\centering

  \setlength{\tabcolsep}{0pt}
  \begin{tabular}{c c c c c c}

\begin{minipage}{0.15\linewidth}
 \begin{overpic}[trim=0cm 0cm 0cm 0cm,clip,width=0.93\linewidth]{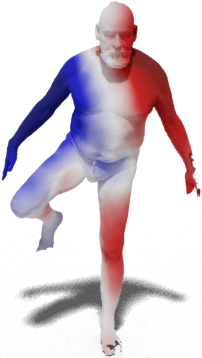}
  \put(20,105){\footnotesize{$\Phi_X$}}
   \end{overpic}
\end{minipage}
&
\begin{minipage}{0.15\linewidth}
 \begin{overpic}[trim=0cm 0cm 0cm 0cm,clip,width=0.93\linewidth]{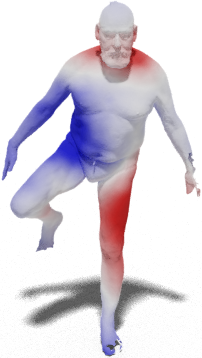}
  \put(20,105){\footnotesize{$\Phi_XA_{XY}$}}
   \end{overpic}
\end{minipage}
&
\begin{minipage}{0.1\linewidth}
 \begin{overpic}[trim=0cm 0cm 0cm 0cm,clip,width=0.93\linewidth]{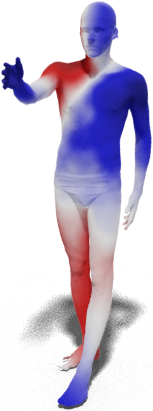}
  \put(18,105){\footnotesize{$\Phi_Y$}}
   \end{overpic}
\end{minipage}

\hspace{1cm}

&
\begin{minipage}{0.15\linewidth}
 \begin{overpic}[trim=0cm 0cm 0cm 0cm,clip,width=0.93\linewidth]{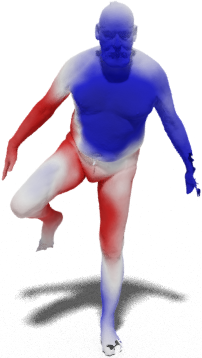}
  \put(20,105){\footnotesize{$\Phi_X$}}
   \end{overpic}
\end{minipage}
&
\begin{minipage}{0.15\linewidth}
 \begin{overpic}[trim=0cm 0cm 0cm 0cm,clip,width=0.93\linewidth]{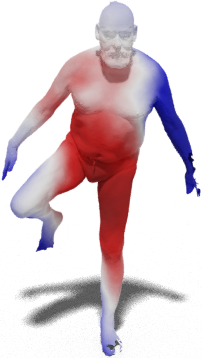}
   \put(20,105){\footnotesize{$\Phi_XA_{XY}$}}
   \end{overpic}
\end{minipage}
&
\begin{minipage}{0.1\linewidth}
 \begin{overpic}[trim=0cm 0cm 0cm 0cm,clip,width=0.93\linewidth]{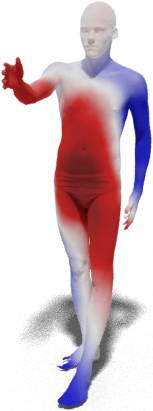}
  \put(18,105){\footnotesize{$\Phi_Y$}}
   \end{overpic}
\end{minipage}
\\

\begin{minipage}{0.15\linewidth}
 \begin{overpic}[trim=0cm 0cm 0cm 0cm,clip,width=0.93\linewidth]{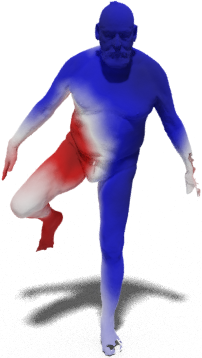}
  
   \end{overpic}
\end{minipage}
&
\begin{minipage}{0.15\linewidth}
 \begin{overpic}[trim=0cm 0cm 0cm 0cm,clip,width=0.93\linewidth]{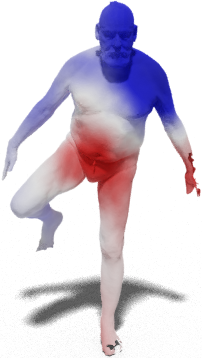}
  
   \end{overpic}
\end{minipage}
&
\begin{minipage}{0.1\linewidth}
 \begin{overpic}[trim=0cm 0cm 0cm 0cm,clip,width=0.93\linewidth]{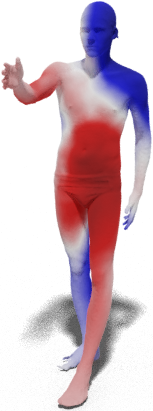}
  
   \end{overpic}
\end{minipage}

\hspace{1cm}

&
\begin{minipage}{0.15\linewidth}
 \begin{overpic}[trim=0cm 0cm 0cm 0cm,clip,width=0.93\linewidth]{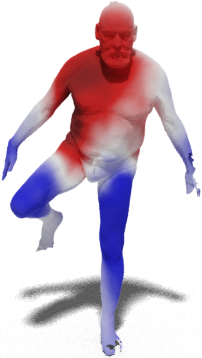}
  
   \end{overpic}
\end{minipage}
&
\begin{minipage}{0.15\linewidth}
 \begin{overpic}[trim=0cm 0cm 0cm 0cm,clip,width=0.93\linewidth]{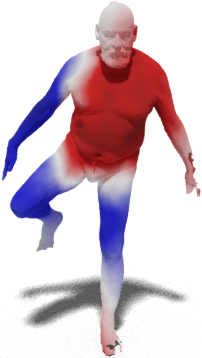}
  
   \end{overpic}
\end{minipage}
&
\begin{minipage}{0.1\linewidth}
 \begin{overpic}[trim=0cm 0cm 0cm 0cm,clip,width=0.93\linewidth]{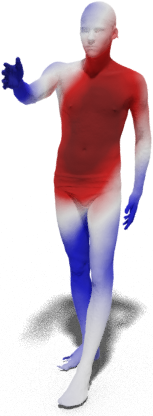}
  
   \end{overpic}
\end{minipage}

\hspace{0.5cm} 

\\

\begin{minipage}{0.15\linewidth}
 \begin{overpic}[trim=0cm 0cm 0cm 0cm,clip,width=0.93\linewidth]{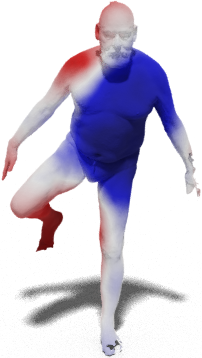}
  
   \end{overpic}
\end{minipage}
&
\begin{minipage}{0.15\linewidth}
 \begin{overpic}[trim=0cm 0cm 0cm 0cm,clip,width=0.93\linewidth]{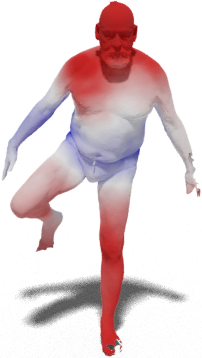}
  
   \end{overpic}
\end{minipage}
&
\begin{minipage}{0.1\linewidth}
 \begin{overpic}[trim=0cm 0cm 0cm 0cm,clip,width=0.93\linewidth]{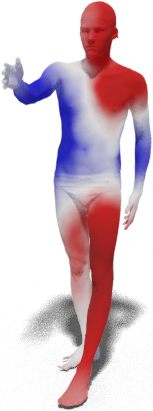}
  
   \end{overpic}
\end{minipage}

\hspace{1cm}

&
\begin{minipage}{0.15\linewidth}
 \begin{overpic}[trim=0cm 0cm 0cm 0cm,clip,width=0.93\linewidth]{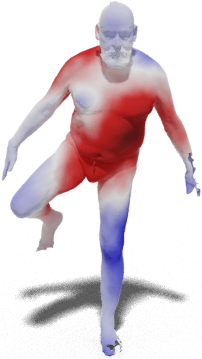}
  
   \end{overpic}
\end{minipage}
&
\begin{minipage}{0.15\linewidth}
 \begin{overpic}[trim=0cm 0cm 0cm 0cm,clip,width=0.93\linewidth]{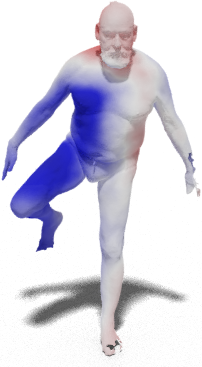}
   \end{overpic}
\end{minipage}
&
\begin{minipage}{0.1\linewidth}
 \begin{overpic}[trim=0cm 0cm 0cm 0cm,clip,width=0.93\linewidth]{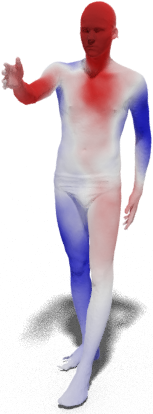}

   \end{overpic}
\end{minipage}

\hspace{0.5cm} 

\\

\begin{minipage}{0.15\linewidth}
 \begin{overpic}[trim=0cm 0cm 0cm 0cm,clip,width=0.93\linewidth]{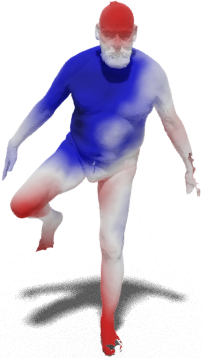}
  
   \end{overpic}
\end{minipage}
&
\begin{minipage}{0.15\linewidth}
 \begin{overpic}[trim=0cm 0cm 0cm 0cm,clip,width=0.93\linewidth]{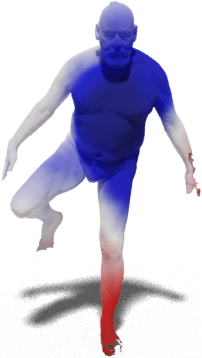}
  
   \end{overpic}
\end{minipage}
&
\begin{minipage}{0.1\linewidth}
 \begin{overpic}[trim=0cm 0cm 0cm 0cm,clip,width=0.93\linewidth]{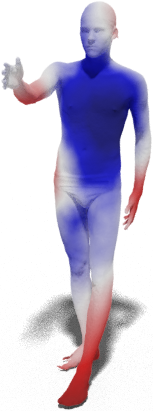}
  
   \end{overpic}
\end{minipage}

\hspace{1cm}

&
\begin{minipage}{0.15\linewidth}
 \begin{overpic}[trim=0cm 0cm 0cm 0cm,clip,width=0.93\linewidth]{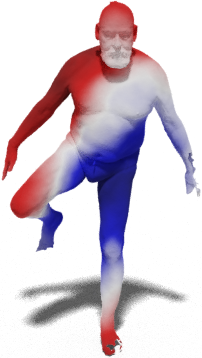}
  
   \end{overpic}
\end{minipage}
&
\begin{minipage}{0.15\linewidth}
 \begin{overpic}[trim=0cm 0cm 0cm 0cm,clip,width=0.93\linewidth]{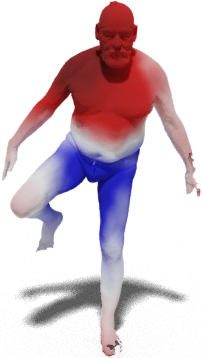}
   \end{overpic}
\end{minipage}
&
\begin{minipage}{0.1\linewidth}
 \begin{overpic}[trim=0cm 0cm 0cm 0cm,clip,width=0.93\linewidth]{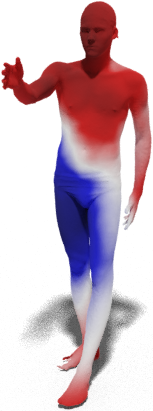}

   \end{overpic}
\end{minipage}

\\

\begin{minipage}{0.15\linewidth}
 \begin{overpic}[trim=0cm 0cm 0cm 0cm,clip,width=0.93\linewidth]{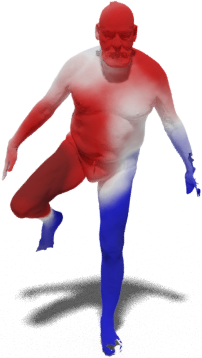}
  
   \end{overpic}
\end{minipage}
&
\begin{minipage}{0.15\linewidth}
 \begin{overpic}[trim=0cm 0cm 0cm 0cm,clip,width=0.93\linewidth]{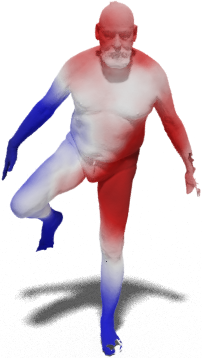}
  
   \end{overpic}
\end{minipage}
&
\begin{minipage}{0.1\linewidth}
 \begin{overpic}[trim=0cm 0cm 0cm 0cm,clip,width=0.93\linewidth]{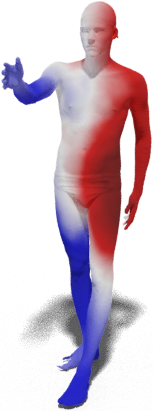}
  
   \end{overpic}
\end{minipage}

\hspace{1cm}

&
\begin{minipage}{0.15\linewidth}
 \begin{overpic}[trim=0cm 0cm 0cm 0cm,clip,width=0.93\linewidth]{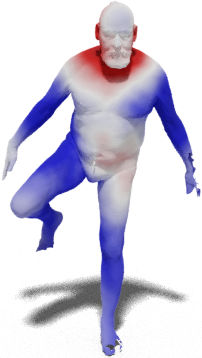}
  
   \end{overpic}
\end{minipage}
&
\begin{minipage}{0.15\linewidth}
 \begin{overpic}[trim=0cm 0cm 0cm 0cm,clip,width=0.93\linewidth]{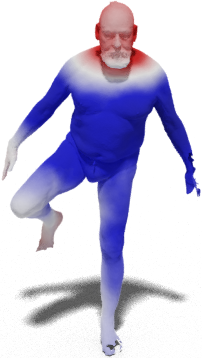}
  
   \end{overpic}
\end{minipage}
&
\begin{minipage}{0.1\linewidth}
 \begin{overpic}[trim=0cm 0cm 0cm 0cm,clip,width=0.93\linewidth]{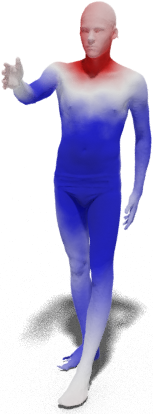}
  
   \end{overpic}
\end{minipage}

\hspace{0.5cm} 

\end{tabular}
\caption{  \label{fig:basis2} Basis from $11$ to $20$}
\end{figure}

\begin{figure}[!t]
\centering

  \setlength{\tabcolsep}{0pt}
  \begin{tabular}{c c c c c c}
 
\begin{minipage}{0.15\linewidth}
 \begin{overpic}[trim=0cm 0cm 0cm 0cm,clip,width=0.93\linewidth]{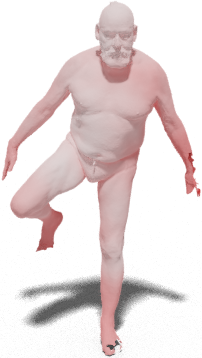}
  \put(20,105){\footnotesize{$G_X$}}
  \put(55,20){\footnotesize{1}}
   \end{overpic}
\end{minipage}
&
\begin{minipage}{0.1\linewidth}
 \begin{overpic}[trim=0cm 0cm 0cm 0cm,clip,width=0.93\linewidth]{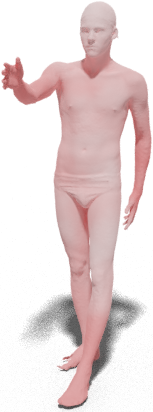}
  \put(20,105){\footnotesize{$G_Y$}}
   \end{overpic}
\end{minipage}
&
\hspace{1cm}

\begin{minipage}{0.15\linewidth}
 \begin{overpic}[trim=0cm 0cm 0cm 0cm,clip,width=0.93\linewidth]{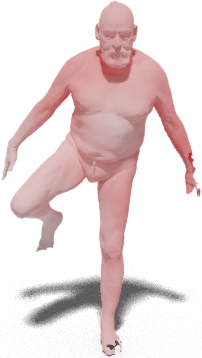}
  \put(20,105){\footnotesize{$G_X$}}
    \put(55,20){\footnotesize{2}}
   \end{overpic}
\end{minipage}
&
\begin{minipage}{0.1\linewidth}
 \begin{overpic}[trim=0cm 0cm 0cm 0cm,clip,width=0.93\linewidth]{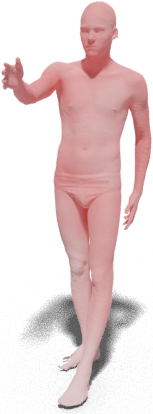}
  \put(20,105){\footnotesize{$G_Y$}}
   \end{overpic}
\end{minipage}
&

\hspace{1cm}

\begin{minipage}{0.15\linewidth}
 \begin{overpic}[trim=0cm 0cm 0cm 0cm,clip,width=0.93\linewidth]{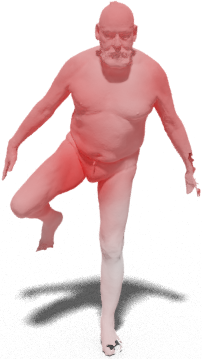}
  \put(20,105){\footnotesize{$G_X$}}
    \put(55,20){\footnotesize{3}}
   \end{overpic}
\end{minipage}
&
\begin{minipage}{0.1\linewidth}
 \begin{overpic}[trim=0cm 0cm 0cm 0cm,clip,width=0.93\linewidth]{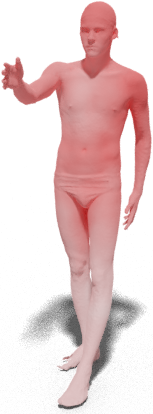}
  \put(20,105){\footnotesize{$G_Y$}}
   \end{overpic}
\end{minipage}
\\

\begin{minipage}{0.15\linewidth}
 \begin{overpic}[trim=0cm 0cm 0cm 0cm,clip,width=0.93\linewidth]{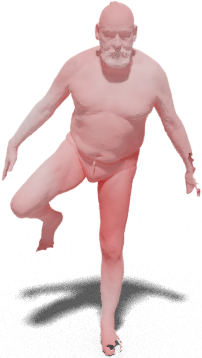}
  \put(55,20){\footnotesize{4}}
   \end{overpic}
\end{minipage}
&
\begin{minipage}{0.1\linewidth}
 \begin{overpic}[trim=0cm 0cm 0cm 0cm,clip,width=0.93\linewidth]{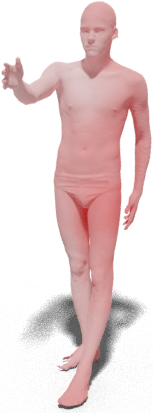}
   \end{overpic}
\end{minipage}
&
\hspace{1cm}

\begin{minipage}{0.15\linewidth}
 \begin{overpic}[trim=0cm 0cm 0cm 0cm,clip,width=0.93\linewidth]{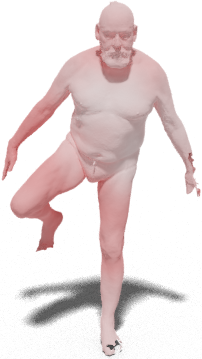}

    \put(55,20){\footnotesize{5}}
   \end{overpic}
\end{minipage}
&
\begin{minipage}{0.1\linewidth}
 \begin{overpic}[trim=0cm 0cm 0cm 0cm,clip,width=0.93\linewidth]{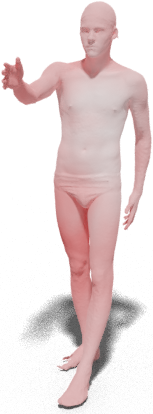}

   \end{overpic}
\end{minipage}
&

\hspace{1cm}

\begin{minipage}{0.15\linewidth}
 \begin{overpic}[trim=0cm 0cm 0cm 0cm,clip,width=0.93\linewidth]{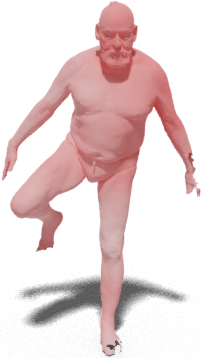}

    \put(55,20){\footnotesize{6}}
   \end{overpic}
\end{minipage}
&
\begin{minipage}{0.1\linewidth}
 \begin{overpic}[trim=0cm 0cm 0cm 0cm,clip,width=0.93\linewidth]{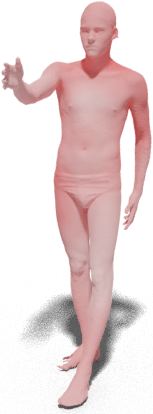}

   \end{overpic}
\end{minipage}
\\

\begin{minipage}{0.15\linewidth}
 \begin{overpic}[trim=0cm 0cm 0cm 0cm,clip,width=0.93\linewidth]{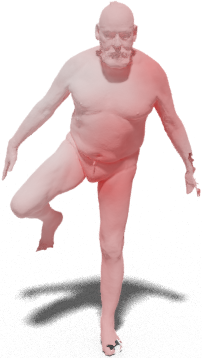}

  \put(55,20){\footnotesize{7}}
   \end{overpic}
\end{minipage}
&
\begin{minipage}{0.1\linewidth}
 \begin{overpic}[trim=0cm 0cm 0cm 0cm,clip,width=0.93\linewidth]{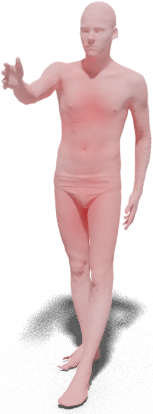}

   \end{overpic}
\end{minipage}
&
\hspace{1cm}

\begin{minipage}{0.15\linewidth}
 \begin{overpic}[trim=0cm 0cm 0cm 0cm,clip,width=0.93\linewidth]{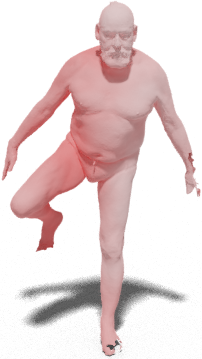}

    \put(55,20){\footnotesize{8}}
   \end{overpic}
\end{minipage}
&
\begin{minipage}{0.1\linewidth}
 \begin{overpic}[trim=0cm 0cm 0cm 0cm,clip,width=0.93\linewidth]{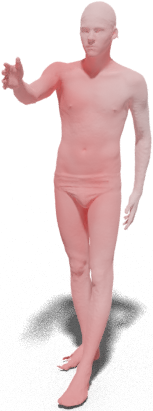}

   \end{overpic}
\end{minipage}
&

\hspace{1cm}

\begin{minipage}{0.15\linewidth}
 \begin{overpic}[trim=0cm 0cm 0cm 0cm,clip,width=0.93\linewidth]{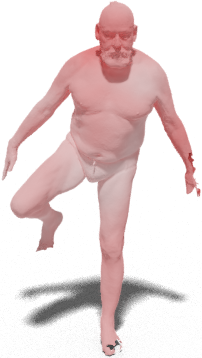}

    \put(55,20){\footnotesize{9}}
   \end{overpic}
\end{minipage}
&
\begin{minipage}{0.1\linewidth}
 \begin{overpic}[trim=0cm 0cm 0cm 0cm,clip,width=0.93\linewidth]{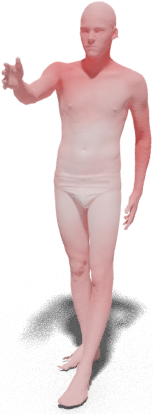}

   \end{overpic}
\end{minipage}
\\

\begin{minipage}{0.15\linewidth}
 \begin{overpic}[trim=0cm 0cm 0cm 0cm,clip,width=0.93\linewidth]{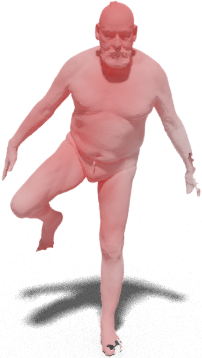}

  \put(55,20){\footnotesize{10}}
   \end{overpic}
\end{minipage}
&
\begin{minipage}{0.1\linewidth}
 \begin{overpic}[trim=0cm 0cm 0cm 0cm,clip,width=0.93\linewidth]{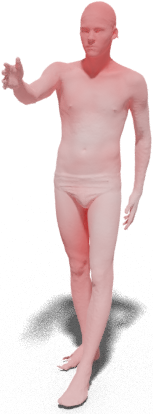}

   \end{overpic}
\end{minipage}
&
\hspace{1cm}

\begin{minipage}{0.15\linewidth}
 \begin{overpic}[trim=0cm 0cm 0cm 0cm,clip,width=0.93\linewidth]{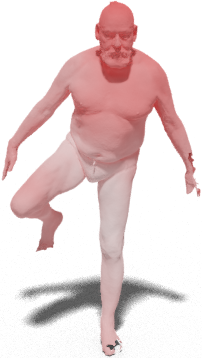}

    \put(55,20){\footnotesize{11}}
   \end{overpic}
\end{minipage}
&
\begin{minipage}{0.1\linewidth}
 \begin{overpic}[trim=0cm 0cm 0cm 0cm,clip,width=0.93\linewidth]{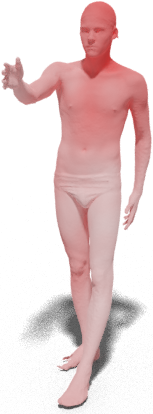}

   \end{overpic}
\end{minipage}
&

\hspace{1cm}

\begin{minipage}{0.15\linewidth}
 \begin{overpic}[trim=0cm 0cm 0cm 0cm,clip,width=0.93\linewidth]{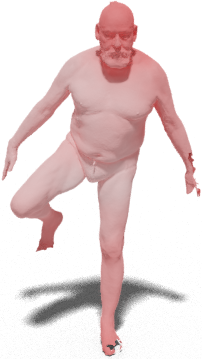}

    \put(55,20){\footnotesize{12}}
   \end{overpic}
\end{minipage}
&
\begin{minipage}{0.1\linewidth}
 \begin{overpic}[trim=0cm 0cm 0cm 0cm,clip,width=0.93\linewidth]{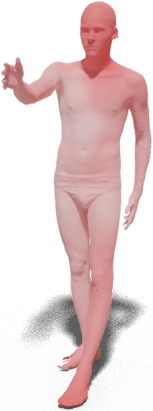}

   \end{overpic}
\end{minipage}
\\

\begin{minipage}{0.15\linewidth}
 \begin{overpic}[trim=0cm 0cm 0cm 0cm,clip,width=0.93\linewidth]{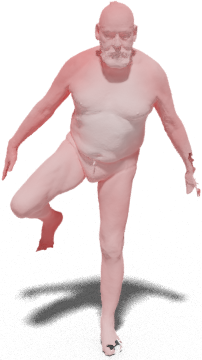}

  \put(55,20){\footnotesize{13}}
   \end{overpic}
\end{minipage}
&
\begin{minipage}{0.1\linewidth}
 \begin{overpic}[trim=0cm 0cm 0cm 0cm,clip,width=0.93\linewidth]{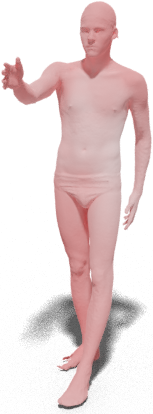}

   \end{overpic}
\end{minipage}
&
\hspace{1cm}

\begin{minipage}{0.15\linewidth}
 \begin{overpic}[trim=0cm 0cm 0cm 0cm,clip,width=0.93\linewidth]{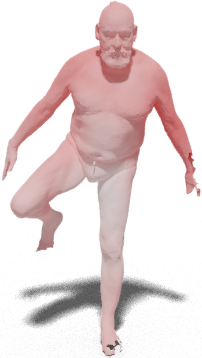}

    \put(55,20){\footnotesize{14}}
   \end{overpic}
\end{minipage}
&
\begin{minipage}{0.1\linewidth}
 \begin{overpic}[trim=0cm 0cm 0cm 0cm,clip,width=0.93\linewidth]{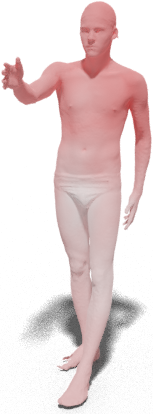}

   \end{overpic}
\end{minipage}
&

\hspace{1cm}

\begin{minipage}{0.15\linewidth}
 \begin{overpic}[trim=0cm 0cm 0cm 0cm,clip,width=0.93\linewidth]{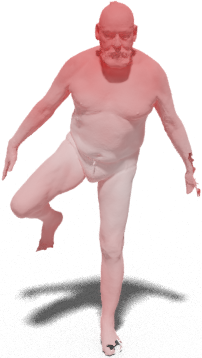}

    \put(55,20){\footnotesize{15}}
   \end{overpic}
\end{minipage}
&
\begin{minipage}{0.1\linewidth}
 \begin{overpic}[trim=0cm 0cm 0cm 0cm,clip,width=0.93\linewidth]{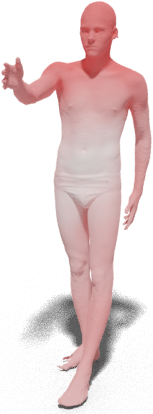}

   \end{overpic}
\end{minipage}

\\

\begin{minipage}{0.15\linewidth}
 \begin{overpic}[trim=0cm 0cm 0cm 0cm,clip,width=0.93\linewidth]{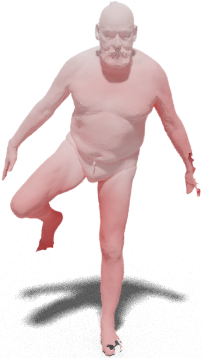}

  \put(55,20){\footnotesize{16}}
   \end{overpic}
\end{minipage}
&
\begin{minipage}{0.1\linewidth}
 \begin{overpic}[trim=0cm 0cm 0cm 0cm,clip,width=0.93\linewidth]{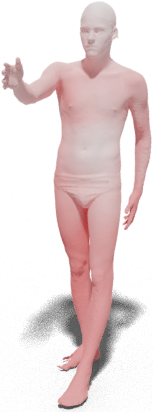}

   \end{overpic}
\end{minipage}
&
\hspace{1cm}

\begin{minipage}{0.15\linewidth}
 \begin{overpic}[trim=0cm 0cm 0cm 0cm,clip,width=0.93\linewidth]{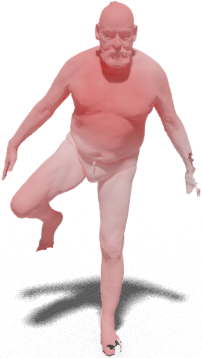}

    \put(55,20){\footnotesize{17}}
   \end{overpic}
\end{minipage}
&
\begin{minipage}{0.1\linewidth}
 \begin{overpic}[trim=0cm 0cm 0cm 0cm,clip,width=0.93\linewidth]{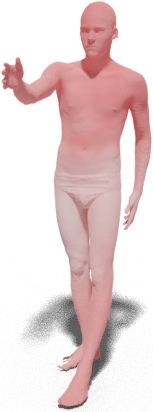}

   \end{overpic}
\end{minipage}
&

\hspace{1cm}

\begin{minipage}{0.15\linewidth}
 \begin{overpic}[trim=0cm 0cm 0cm 0cm,clip,width=0.93\linewidth]{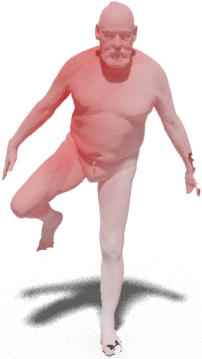}

    \put(55,20){\footnotesize{18}}
   \end{overpic}
\end{minipage}
&
\begin{minipage}{0.1\linewidth}
 \begin{overpic}[trim=0cm 0cm 0cm 0cm,clip,width=0.93\linewidth]{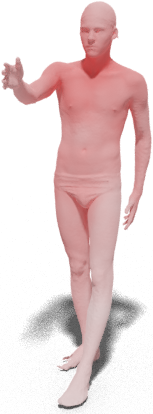}

   \end{overpic}
\end{minipage}

\hspace{0.5cm} 
\end{tabular}
\caption{  \label{fig:desc1} Descriptors from $1$ to $18$}
\end{figure}

\begin{figure}[!t]
\centering

  \setlength{\tabcolsep}{0pt}
  \begin{tabular}{c c c c c c}
 
\begin{minipage}{0.15\linewidth}
 \begin{overpic}[trim=0cm 0cm 0cm 0cm,clip,width=0.93\linewidth]{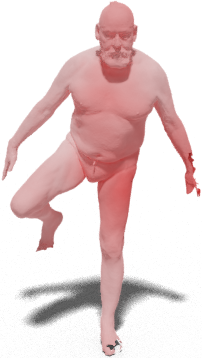}
  \put(20,105){\footnotesize{$G_X$}}
  \put(55,20){\footnotesize{19}}
   \end{overpic}
\end{minipage}
&
\begin{minipage}{0.1\linewidth}
 \begin{overpic}[trim=0cm 0cm 0cm 0cm,clip,width=0.93\linewidth]{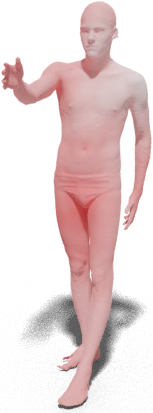}
  \put(20,105){\footnotesize{$G_Y$}}
   \end{overpic}
\end{minipage}
&
\hspace{1cm}

\begin{minipage}{0.15\linewidth}
 \begin{overpic}[trim=0cm 0cm 0cm 0cm,clip,width=0.93\linewidth]{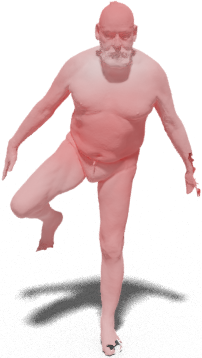}
  \put(20,105){\footnotesize{$G_X$}}
    \put(55,20){\footnotesize{20}}
   \end{overpic}
\end{minipage}
&
\begin{minipage}{0.1\linewidth}
 \begin{overpic}[trim=0cm 0cm 0cm 0cm,clip,width=0.93\linewidth]{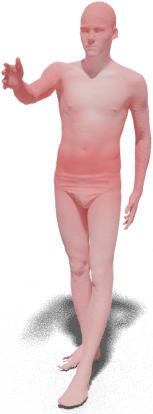}
  \put(20,105){\footnotesize{$G_Y$}}
   \end{overpic}
\end{minipage}
&

\hspace{1cm}

\begin{minipage}{0.15\linewidth}
 \begin{overpic}[trim=0cm 0cm 0cm 0cm,clip,width=0.93\linewidth]{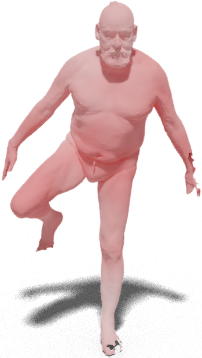}
  \put(20,105){\footnotesize{$G_X$}}
    \put(55,20){\footnotesize{21}}
   \end{overpic}
\end{minipage}
&
\begin{minipage}{0.1\linewidth}
 \begin{overpic}[trim=0cm 0cm 0cm 0cm,clip,width=0.93\linewidth]{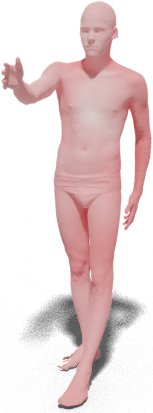}
  \put(20,105){\footnotesize{$G_Y$}}
   \end{overpic}
\end{minipage}
\\

\begin{minipage}{0.15\linewidth}
 \begin{overpic}[trim=0cm 0cm 0cm 0cm,clip,width=0.93\linewidth]{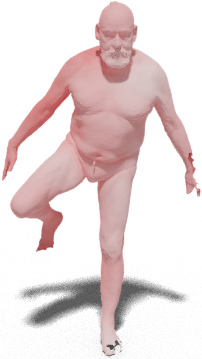}
  \put(55,20){\footnotesize{22}}
   \end{overpic}
\end{minipage}
&
\begin{minipage}{0.1\linewidth}
 \begin{overpic}[trim=0cm 0cm 0cm 0cm,clip,width=0.93\linewidth]{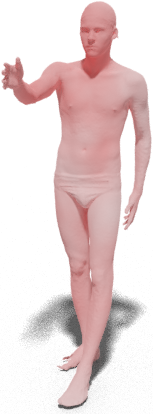}
   \end{overpic}
\end{minipage}
&
\hspace{1cm}

\begin{minipage}{0.15\linewidth}
 \begin{overpic}[trim=0cm 0cm 0cm 0cm,clip,width=0.93\linewidth]{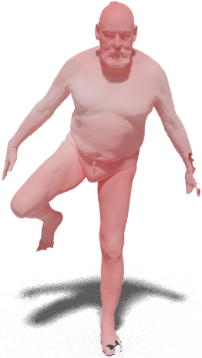}

    \put(55,20){\footnotesize{23}}
   \end{overpic}
\end{minipage}
&
\begin{minipage}{0.1\linewidth}
 \begin{overpic}[trim=0cm 0cm 0cm 0cm,clip,width=0.93\linewidth]{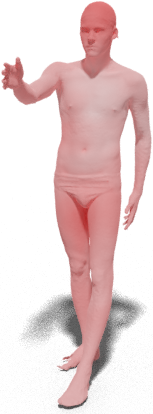}

   \end{overpic}
\end{minipage}
&

\hspace{1cm}

\begin{minipage}{0.15\linewidth}
 \begin{overpic}[trim=0cm 0cm 0cm 0cm,clip,width=0.93\linewidth]{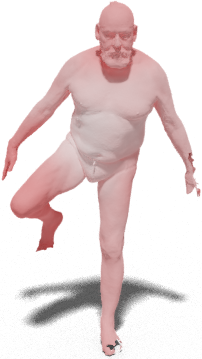}

    \put(55,20){\footnotesize{24}}
   \end{overpic}
\end{minipage}
&
\begin{minipage}{0.1\linewidth}
 \begin{overpic}[trim=0cm 0cm 0cm 0cm,clip,width=0.93\linewidth]{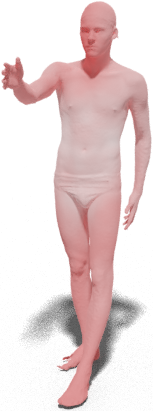}

   \end{overpic}
\end{minipage}
\\

\begin{minipage}{0.15\linewidth}
 \begin{overpic}[trim=0cm 0cm 0cm 0cm,clip,width=0.93\linewidth]{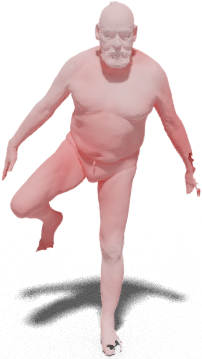}

  \put(55,20){\footnotesize{25}}
   \end{overpic}
\end{minipage}
&
\begin{minipage}{0.1\linewidth}
 \begin{overpic}[trim=0cm 0cm 0cm 0cm,clip,width=0.93\linewidth]{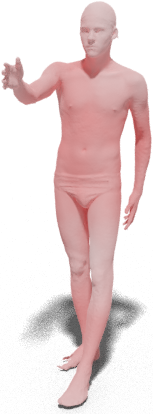}

   \end{overpic}
\end{minipage}
&
\hspace{1cm}

\begin{minipage}{0.15\linewidth}
 \begin{overpic}[trim=0cm 0cm 0cm 0cm,clip,width=0.93\linewidth]{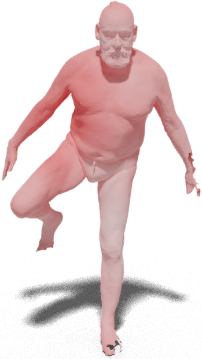}

    \put(55,20){\footnotesize{26}}
   \end{overpic}
\end{minipage}
&
\begin{minipage}{0.1\linewidth}
 \begin{overpic}[trim=0cm 0cm 0cm 0cm,clip,width=0.93\linewidth]{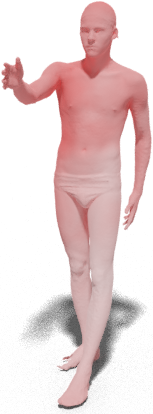}

   \end{overpic}
\end{minipage}
&

\hspace{1cm}

\begin{minipage}{0.15\linewidth}
 \begin{overpic}[trim=0cm 0cm 0cm 0cm,clip,width=0.93\linewidth]{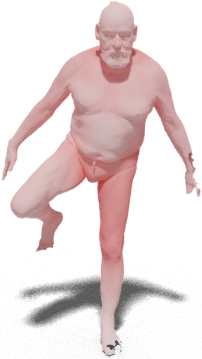}

    \put(55,20){\footnotesize{27}}
   \end{overpic}
\end{minipage}
&
\begin{minipage}{0.1\linewidth}
 \begin{overpic}[trim=0cm 0cm 0cm 0cm,clip,width=0.93\linewidth]{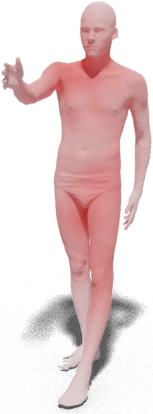}

   \end{overpic}
\end{minipage}
\\

\begin{minipage}{0.15\linewidth}
 \begin{overpic}[trim=0cm 0cm 0cm 0cm,clip,width=0.93\linewidth]{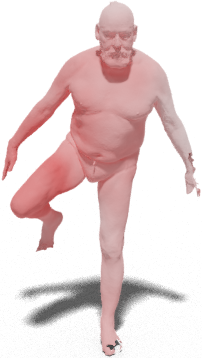}

  \put(55,20){\footnotesize{28}}
   \end{overpic}
\end{minipage}
&
\begin{minipage}{0.1\linewidth}
 \begin{overpic}[trim=0cm 0cm 0cm 0cm,clip,width=0.93\linewidth]{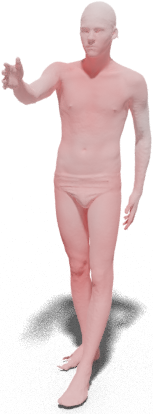}

   \end{overpic}
\end{minipage}
&
\hspace{1cm}

\begin{minipage}{0.15\linewidth}
 \begin{overpic}[trim=0cm 0cm 0cm 0cm,clip,width=0.93\linewidth]{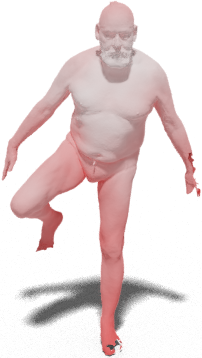}

    \put(55,20){\footnotesize{29}}
   \end{overpic}
\end{minipage}
&
\begin{minipage}{0.1\linewidth}
 \begin{overpic}[trim=0cm 0cm 0cm 0cm,clip,width=0.93\linewidth]{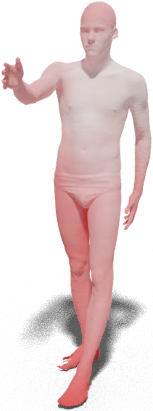}

   \end{overpic}
\end{minipage}
&

\hspace{1cm}

\begin{minipage}{0.15\linewidth}
 \begin{overpic}[trim=0cm 0cm 0cm 0cm,clip,width=0.93\linewidth]{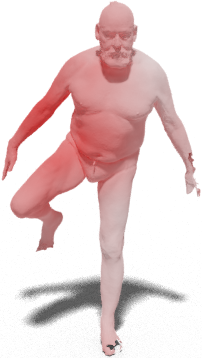}

    \put(55,20){\footnotesize{30}}
   \end{overpic}
\end{minipage}
&
\begin{minipage}{0.1\linewidth}
 \begin{overpic}[trim=0cm 0cm 0cm 0cm,clip,width=0.93\linewidth]{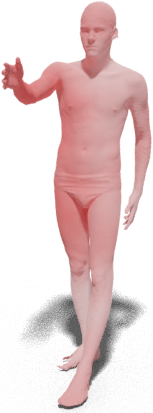}

   \end{overpic}
\end{minipage}
\\

\begin{minipage}{0.15\linewidth}
 \begin{overpic}[trim=0cm 0cm 0cm 0cm,clip,width=0.93\linewidth]{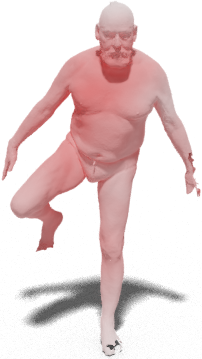}

  \put(55,20){\footnotesize{31}}
   \end{overpic}
\end{minipage}
&
\begin{minipage}{0.1\linewidth}
 \begin{overpic}[trim=0cm 0cm 0cm 0cm,clip,width=0.93\linewidth]{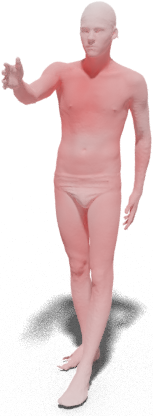}

   \end{overpic}
\end{minipage}
&
\hspace{1cm}

\begin{minipage}{0.15\linewidth}
 \begin{overpic}[trim=0cm 0cm 0cm 0cm,clip,width=0.93\linewidth]{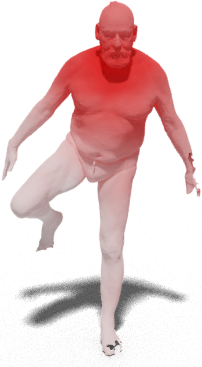}

    \put(55,20){\footnotesize{32}}
   \end{overpic}
\end{minipage}
&
\begin{minipage}{0.1\linewidth}
 \begin{overpic}[trim=0cm 0cm 0cm 0cm,clip,width=0.93\linewidth]{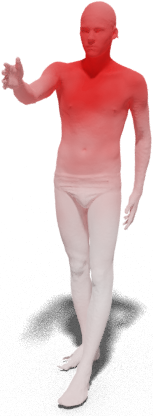}

   \end{overpic}
\end{minipage}
&

\hspace{1cm}

\begin{minipage}{0.15\linewidth}
 \begin{overpic}[trim=0cm 0cm 0cm 0cm,clip,width=0.93\linewidth]{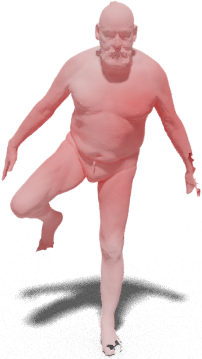}

    \put(55,20){\footnotesize{33}}
   \end{overpic}
\end{minipage}
&
\begin{minipage}{0.1\linewidth}
 \begin{overpic}[trim=0cm 0cm 0cm 0cm,clip,width=0.93\linewidth]{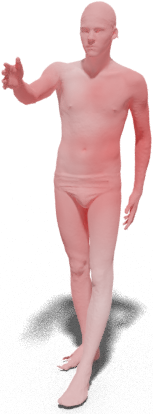}

   \end{overpic}
\end{minipage}

\\

\begin{minipage}{0.15\linewidth}
 \begin{overpic}[trim=0cm 0cm 0cm 0cm,clip,width=0.93\linewidth]{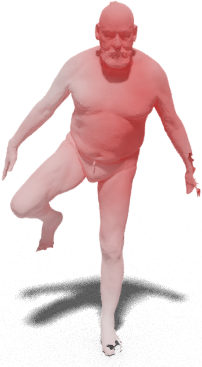}

  \put(55,20){\footnotesize{34}}
   \end{overpic}
\end{minipage}
&
\begin{minipage}{0.1\linewidth}
 \begin{overpic}[trim=0cm 0cm 0cm 0cm,clip,width=0.93\linewidth]{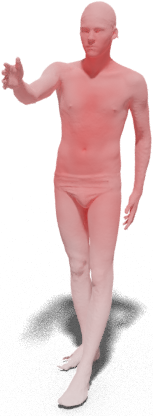}

   \end{overpic}
\end{minipage}
&
\hspace{1cm}

\begin{minipage}{0.15\linewidth}
 \begin{overpic}[trim=0cm 0cm 0cm 0cm,clip,width=0.93\linewidth]{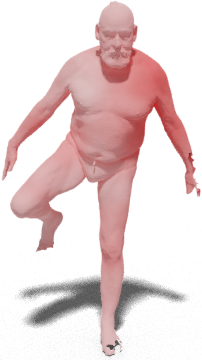}

    \put(55,20){\footnotesize{35}}
   \end{overpic}
\end{minipage}
&
\begin{minipage}{0.1\linewidth}
 \begin{overpic}[trim=0cm 0cm 0cm 0cm,clip,width=0.93\linewidth]{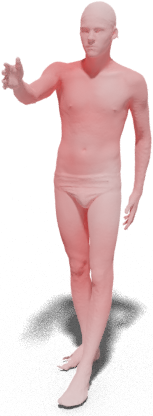}

   \end{overpic}
\end{minipage}
&

\hspace{1cm}

\begin{minipage}{0.15\linewidth}
 \begin{overpic}[trim=0cm 0cm 0cm 0cm,clip,width=0.93\linewidth]{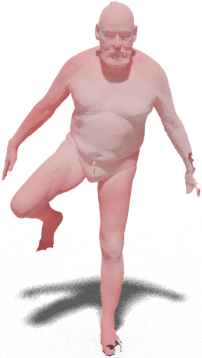}

    \put(55,20){\footnotesize{36}}
   \end{overpic}
\end{minipage}
&
\begin{minipage}{0.1\linewidth}
 \begin{overpic}[trim=0cm 0cm 0cm 0cm,clip,width=0.93\linewidth]{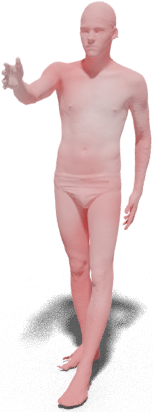}

   \end{overpic}
\end{minipage}

\hspace{0.5cm} 
\end{tabular}
\caption{  \label{fig:desc2} Descriptors from $19$ to $36$.}
\end{figure}

\begin{figure}[!t]
\centering

  \setlength{\tabcolsep}{0pt}
  \begin{tabular}{c c c c c c}
 
\begin{minipage}{0.15\linewidth}
 \begin{overpic}[trim=0cm 0cm 0cm 0cm,clip,width=0.93\linewidth]{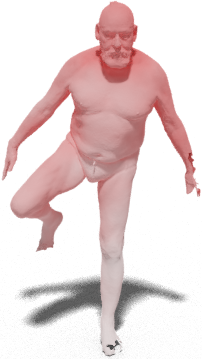}
  \put(20,105){\footnotesize{$G_X$}}
  \put(55,20){\footnotesize{37}}
   \end{overpic}
\end{minipage}
&
\begin{minipage}{0.1\linewidth}
 \begin{overpic}[trim=0cm 0cm 0cm 0cm,clip,width=0.93\linewidth]{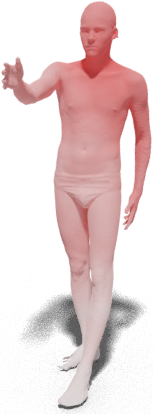}
  \put(20,105){\footnotesize{$G_Y$}}
   \end{overpic}
\end{minipage}
&
\hspace{1cm}

\begin{minipage}{0.15\linewidth}
 \begin{overpic}[trim=0cm 0cm 0cm 0cm,clip,width=0.93\linewidth]{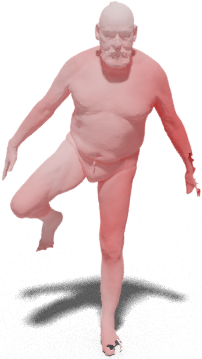}
  \put(20,105){\footnotesize{$G_X$}}
    \put(55,20){\footnotesize{38}}
   \end{overpic}
\end{minipage}
&
\begin{minipage}{0.1\linewidth}
 \begin{overpic}[trim=0cm 0cm 0cm 0cm,clip,width=0.93\linewidth]{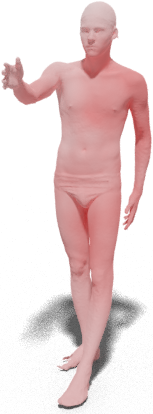}
  \put(20,105){\footnotesize{$G_Y$}}
   \end{overpic}
\end{minipage}
&

\hspace{1cm}

\begin{minipage}{0.15\linewidth}
 \begin{overpic}[trim=0cm 0cm 0cm 0cm,clip,width=0.93\linewidth]{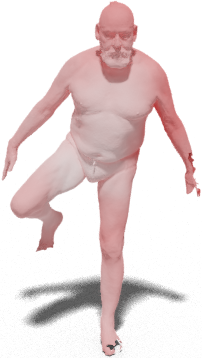}
  \put(20,105){\footnotesize{$G_X$}}
    \put(55,20){\footnotesize{39}}
   \end{overpic}
\end{minipage}
&
\begin{minipage}{0.1\linewidth}
 \begin{overpic}[trim=0cm 0cm 0cm 0cm,clip,width=0.93\linewidth]{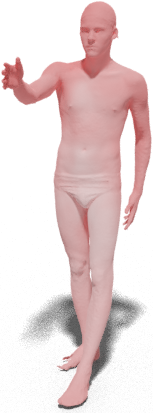}
  \put(20,105){\footnotesize{$G_Y$}}
   \end{overpic}
\end{minipage}
\\

\begin{minipage}{0.15\linewidth}
 \begin{overpic}[trim=0cm 0cm 0cm 0cm,clip,width=0.93\linewidth]{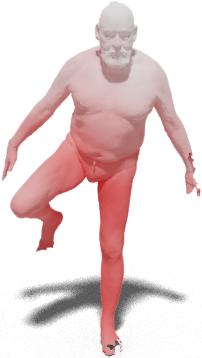}
  \put(55,20){\footnotesize{40}}
   \end{overpic}
\end{minipage}
&
\begin{minipage}{0.1\linewidth}
 \begin{overpic}[trim=0cm 0cm 0cm 0cm,clip,width=0.93\linewidth]{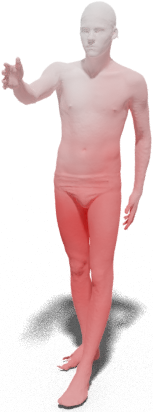}
   \end{overpic}
\end{minipage}
&
\hspace{1cm}

\end{tabular}
\caption{  \label{fig:desc3} Descriptors from $37$ to $40$.}
\end{figure}

\begin{figure}[!t]
\begin{center}
\begin{tabular}{cc|cc|cc}
\centering
 \begin{overpic}[trim=0cm 0cm 0cm 0cm,clip,width=0.11\linewidth]{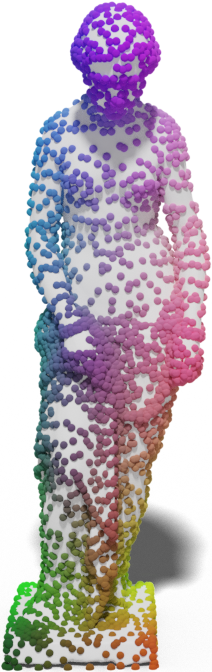}
   \end{overpic}
   &
    \begin{overpic}[trim=0cm 0cm 0cm 0cm,clip,width=0.11\linewidth]{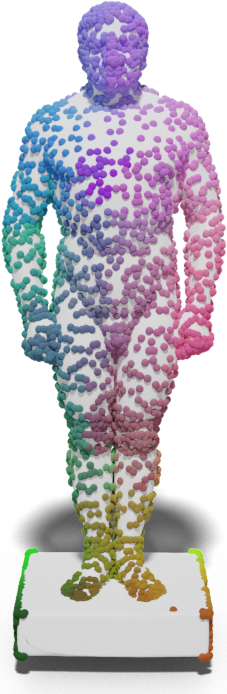}
   \end{overpic}
   &
  \begin{overpic}[trim=0cm 0cm 0cm 0cm,clip,width=0.08\linewidth]{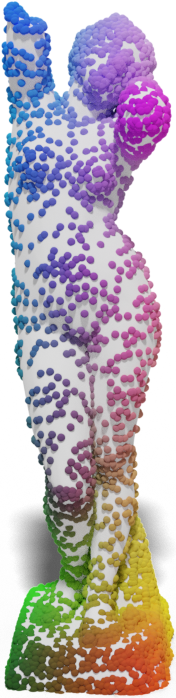}
   \end{overpic}
   &
    \begin{overpic}[trim=0cm 0cm 0cm 0cm,clip,width=0.145\linewidth]{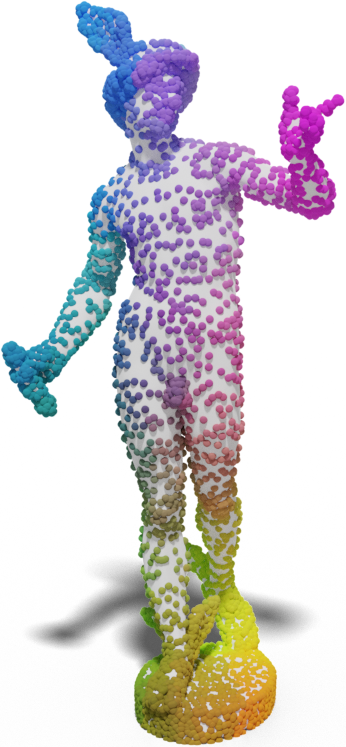}
   \end{overpic}  
    &
  \begin{overpic}[trim=0cm 0cm 0cm 0cm,clip,width=0.13\linewidth]{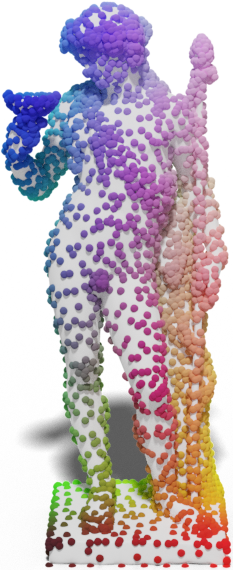}
   \end{overpic}
   &
    \begin{overpic}[trim=0cm 0cm 0cm 0cm,clip,width=0.13\linewidth]{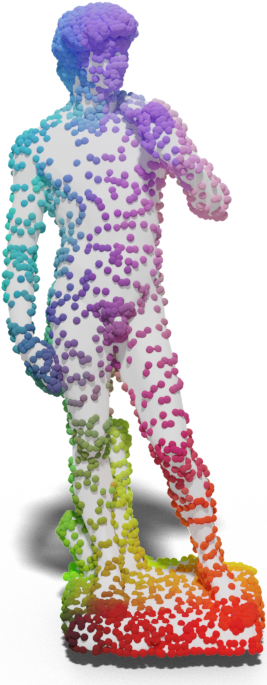}
   \end{overpic}    
   \end{tabular}
\end{center}
   \caption{\label{fig:smatstatues} More qualitative results between statues couples.}

\end{figure}

\end{document}